\documentclass[11pt]{article}
\usepackage[top=30truemm,bottom=30truemm,left=25truemm,right=25truemm]{geometry}

\usepackage{graphicx}
\usepackage{xcolor}

\usepackage{my_macro}
\usepackage{caption}
\usepackage{subcaption}
\usepackage{float}
\usepackage{adjustbox}
\newfloat{problem}{t}{lop}
\floatname{problem}{Problem}

\newcommand{\calf}{\mathcal{F}}

\renewcommand{\S}{\mathcal{S}}
\renewcommand{\P}{\mathbb{P}}

\title{Combinatorial Allocation Bandits with Nonlinear Arm Utility}

\author{
  Yuki Shibukawa\footnote{
    The University of Tokyo and RIKEN; 
    \texttt{shibu-yu762@g.ecc.u-tokyo.ac.jp}.
  }
  \and
  Koichi Tanaka\footnote{
    Keio University; 
    \texttt{kouichi\_1207@keio.jp}.
  }
  \and
  Yuta Saito \footnote{
    Hanuku-kaso, Co., Ltd.; \texttt{saito@hanjuku-kaso.com}.
  }
  \and 
  Shinji Ito\footnote{
    The University of Tokyo and RIKEN; \texttt{shinji@mist.i.u-tokyo.ac.jp}.
  }
}

\begin{document}
\maketitle

% abstract here
\begin{abstract}
A matching platform is a system that matches participants of different types, such as companies and job-seekers. In such a platform, maximizing matches may concentrate assignments on popular participants, increasing dissatisfaction among others, and eventually causing churn, which reduces the platform\char39s profit opportunities. To address this issue, we propose a novel online learning problem, Combinatorial Allocation Bandits (CAB), which incorporates the notion of \textit{arm satisfaction}. In CAB, at each round, the learner observes feature vectors for $K$ arms and $N$ users, assigns users to arms, and observes feedback following a generalized linear model (GLM). Unlike prior work, the objective is to maximize arm satisfaction rather than the number of positive feedback. For CAB, we develop an upper confidence bound algorithm that uses an approximate optimization oracle and achieves an approximate regret upper bound, whose dependence on $d$, $T$, and $N$ matches the known lower bound for contextual combinatorial linear bandits up to logarithmic factors. We also analyze a Thompson sampling algorithm with a standard regret bound under an exact optimization oracle, and propose a cheaper one-pass variant retaining sublinear approximate regret under a self-concordance assumption. Experiments on synthetic data support the objective and show that CAB-UCB achieves higher cumulative satisfaction than baselines.
\end{abstract}

\section{Introduction}
\label{sec:introduction}
% \vspace{-3pt}

Online learning is a framework in which decisions are made sequentially based on observed information.
It has a wide range of potential applications, such as recommender systems, and has been studied extensively from a theoretical perspective \citep{auer2002finite,cesa2006prediction,li2010contextual}.

Although these studies make important theoretical contributions, they mainly focus on maximizing the number of positive feedback, such as matches or clicks, which sometimes may not reflect real-world business objectives.
For example, under unconstrained settings, a match-maximizing algorithm often yields an imbalanced selection of arms, leading to dissatisfaction among infrequently selected arms.
In job-matching platforms that recommend companies to visiting users, the revenue model typically relies on fees paid by the companies participating in the platform to hire qualified applicants.
Thus, the economic cost of company churn can outweigh raw match counts.
% \begin{comment}
Similar structures arise in dating apps and paper review processes.
Dating apps match users with one another.
When matches concentrate among a few popular users, many others receive few or no matches, which in turn reduces their incentives to remain active.
This decline in active participation is undesirable for the platform.
Similarly, paper review processes can be regarded as a match between authors and reviewers. 
If authors are not sufficiently satisfied with the quality of the reviews, they may submit to other journals, resulting in a loss of future submissions.
% \end{comment}
% \looseness=-1

% The key point in the above discussions is that companies whose satisfaction falls below a certain level are expected to have a higher probability of leaving the platform.
% Here, incorporating principles of economics and the cost of interviewing many applicants\footnote{Similar costs arise in dating apps through going on dates and in paper review processes through responding to reviews.}, we model each arm\char39s satisfaction as a concave function of the number of matches it secures.
% Under this model, it is business-aligned to minimize the number of arms that receive few or no matches and are at risk of churn, rather than to allow a small subset of arms to monopolize many matches.
% Here, satisfaction represents an arm-side evaluation of the platform determined by the quality of its allocated users.
% We model it as a concave function to capture diminishing marginal utility and to mitigate concentration without imposing explicit fairness constraints.

The key point in the above discussions is that companies whose satisfaction falls below a certain level are expected to have a higher probability of leaving the platform.
Accordingly, the platform objective should not be to maximize raw match counts alone, but to avoid allocations in which matches are concentrated among a small subset of companies.
Motivated by economic principles and by the cost of interviewing many applicants\footnote{Similar costs arise in dating apps through going on dates and in paper review processes through responding to reviews.}, we model each arm's satisfaction as a concave function of its expected number of matches.
Here, satisfaction represents an arm-side evaluation of the platform determined by the quality of its allocated users.
The concavity captures diminishing marginal utility and mitigates concentration without imposing explicit fairness constraints.

\begin{figure}
    \captionsetup{skip=2.4pt}
    \centering
    % \vspace{-12pt} 
    \includegraphics[width=0.8\linewidth]{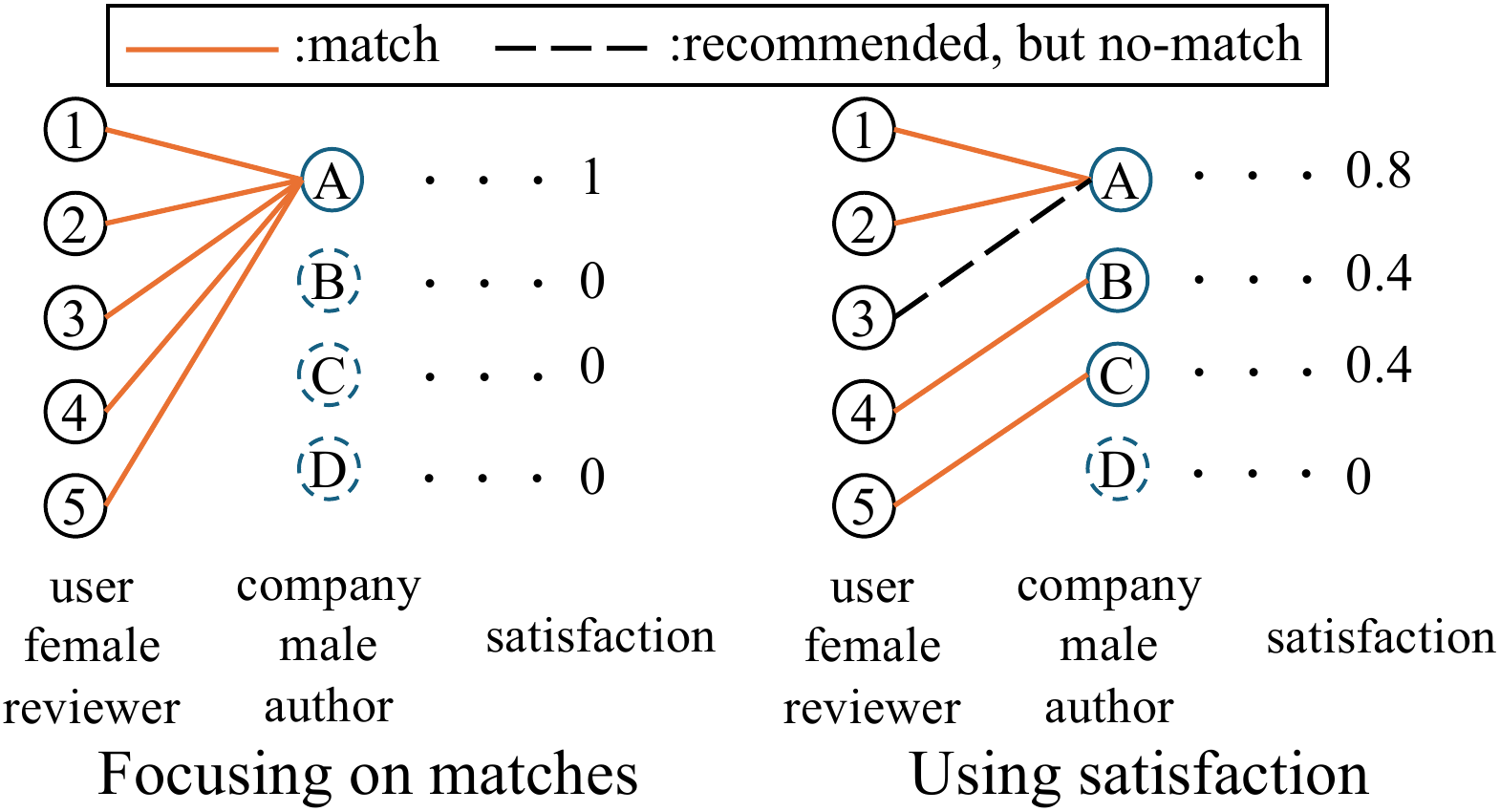}
    \caption{This figure schematically compares satisfaction-based outcomes with match-maximization outcomes. Firm popularity is assumed to decrease from Firm A to Firm D.}
    \label{fig:match_image}
    % \vspace{-10pt} .
\end{figure}

As illustrated in \cref{fig:match_image}, when one arm is substantially easier to match than the others, a match-maximizing policy concentrates all recommendations on that arm.
Such concentration is undesirable because arm satisfaction typically exhibits diminishing returns and therefore does not scale linearly with the number of matches.
Moreover, even for arm A, which receives many assignments, real-world constraints such as budget limits and capacity restrictions, as well as the economic principle of diminishing marginal utility, imply that the satisfaction obtained does not necessarily grow proportionally with the reward.

One way to model such limitations is to impose explicit resource constraints on arm usage.
This view is closely related to bandits with knapsacks (BwK), which incorporate resource constraints into online learning \citep{badanidiyuru2013bwk,agrawal2016linear_bwk,sankararaman2018combinatorial_bwk}. 
BwK is natural when the goal is to limit arm usage via explicit budgets or capacities, but BwK does not directly capture aspects of our interest, such as arm satisfaction or the diminishing property of the utility function.

Another approach imposes fairness constraints on each arm\char39s selection frequency \citep{joseph2016_fairness,li2019combinatorial_fairness,xu2020combinatorial_fairness}. 
These approaches focus on fairness in exposure or selection, rather than directly modeling arm-side utility. 
Consequently, they do not necessarily capture the satisfaction objective, which depends not only on how often an arm is selected but also on the quality of the assigned users and the arms' popularity.
\looseness=-1

% Sequential decision making requires balancing exploration and exploitation. Standard approaches include upper confidence bound (UCB) methods \citep{annals_of_ststs_Lai_1987_MAB,Auer_2002_bandit,siam_qin2014contextual,icml2017glmcontextbandits,lattimore2020bandit} and Thompson sampling (TS) methods \citep{thompson1933likelihood,agrawal2013_ts_cb,ICDM2019_Takemura_ArmWise_Randomization,aistats2017_abeille_ts}.

\vspace{-1.5pt}
\subsection{Our contributions}
\vspace{-1.5pt}
Our main contributions are primarily theoretical and are twofold.
First, we formulate Combinatorial Allocation Bandits (CAB), a contextual combinatorial semi-bandit problem with GLM feedback and nonlinear arm-side utility.
Beyond the linear case, the available combinatorial GLM result is limited to a logistic model with binary feedback \citep{liu_2025_combinatorial_log_bandits}.
In contrast, CAB is formulated for GLM feedback with a non-negative mean and can handle non-binary feedback distributions.
On the objective side, CAB uses a concave arm-side utility that aggregates the expected feedback of the users assigned to the same arm, rather than a linear or user-wise separable objective.
This structure captures diminishing returns and makes the per-round optimization problem NP-hard (\cref{thm:nphard}).

Second, we develop two learning algorithms for CAB: CAB-UCB (\cref{subsec:ucb}), based on the upper confidence bound (UCB) principle, and CAB-TS (\cref{subsec:ts}), based on Thompson sampling (TS), under different offline optimization assumptions.
For CAB-UCB, we combine GLM confidence bounds with an approximate optimization oracle and prove an approximate regret bound of $\tilde O(d\sqrt{NT}+dN)$ (\cref{thm:main_regret_ucb}), whose dependence on $d$, $T$, and $N$ matches the known lower bound for contextual combinatorial linear bandits up to logarithmic factors.
We also show how to implement the CAB-UCB oracle efficiently via a reduction to the submodular welfare problem (\cref{subsec:implementation}).
For CAB-TS, we assume access to an exact optimization oracle, introduce a user-wise i.i.d.\ sampling scheme in which the sampled perturbations enter the objective function linearly as user-wise additive perturbations, and prove a standard regret bound of $\tilde O(dN\sqrt{T}+dN^{3/2})$ (\cref{thm:bound_TS}).
Notably, the regret analyses of CAB-UCB and CAB-TS do not require a self-concordance assumption on the link function, which is used in recent analyses \citep{liu_2025_combinatorial_log_bandits,zhang2025generalizedlinearbanditsoptimal}.

As a computational extension, we also present a one-pass variant (\cref{subsec:one_pass}).
At each round $t$, CAB-UCB and CAB-TS solve a regularized MLE, incurring $O(t)$ cost.
To reduce this computational cost, we replace the MLE update with a one-pass parameter update based on online mirror descent (OMD), following \cite{zhang2025generalizedlinearbanditsoptimal}.
Under an additional self-concordance assumption on the link function, the resulting algorithm avoids solving a regularized MLE while retaining sublinear approximate regret.

Finally, we conduct experiments on synthetic data that support the proposed objective and show that CAB-UCB achieves higher cumulative satisfaction than match-oriented and fairness-oriented baselines (\cref{sec:experiments}).
% \looseness=-1

For space constraints, we provide a discussion of the related work in \cref{app:related_work}.

\vspace{-2pt}
\subsection{Technical challenges}
\vspace{-2pt}
A technical difficulty specific to CAB is that the per-round objective is no longer separable across users. 
Unlike standard contextual combinatorial semi-bandits with linear or arm-wise additive rewards, the benefit of assigning a user to an arm depends on whether other users are assigned to the same arm, because the arm-side utility exhibits diminishing returns.
As a result, analyses for separable reward objectives do not directly apply, and even the offline allocation problem becomes NP-hard.
This structural difference makes the offline allocation step nontrivial and requires oracle assumptions tailored to each algorithm.

This non-separability further complicates uncertainty assessment in the GLM setting.
For CAB-UCB, our key observation is that, although the objective itself is non-separable, its estimation error can still be upper-bounded by a bonus that decomposes as a sum of user-wise terms defined in \eqref{eq:definition_g_t}.
This form is essential for retaining the submodular welfare structure, which allows CAB-UCB to be analyzed using an implementable approximate optimization oracle.
Despite its simple form, the bonus term yields a regret bound with the desired dependence on the number of rounds and users.

CAB-TS is more delicate for several reasons. 
First, a common Gaussian perturbation is insufficient in CAB.
% Since the objective aggregates user-level rewards through a concave arm utility, the effect of such a perturbation can cancel after aggregation, and we cannot derive a suitable bound on the variance of the aggregate perturbation (see \cref{rem:reason_for_iid} for details).
With a common perturbation, user-level feature vectors are summed before their size is measured, so their directions can cancel, and we cannot derive a suitable bound on the variance of the aggregate perturbation (see \cref{rem:reason_for_iid}).
CAB-TS addresses this issue by using user-wise independent Gaussian perturbations, which preserve enough variance in the aggregate perturbation of the optimal allocation to support the regret analysis.
Second, directly perturbing the nonlinear utility is difficult because the perturbation enters a coupled concave objective.
We therefore add exploration through a separate user-wise linear perturbation.
We choose the covariance surrogate by placing the regularization term inside the Hessian-weighted sum, keeping it comparable to the GLM information matrix.
This design avoids perturbing the nonlinear utility directly (see \cref{rem:worse_point} for details).
% \looseness=-1

\vspace{-2.5pt}
\section{Preliminaries}
\vspace{-2.5pt}
In this section, we describe the setting of the GLM and the submodular welfare problem used in the implementation of our algorithms.

\vspace{-2pt}
\paragraph{Notations}
For $n \in \N$, let $\brk{n} = \set{1, \dots, n}$.
For $\bm{x} \in \R^d$, denote the transpose by $\bm{x}^\top$.
For a positive definite matrix $\bm{A}$, define $\nrm{\bm{x}}_{\bm{A}} = \sqrt{\bm{x}^\top \bm{A} \bm{x}}$, and let $\lambda_{\min}(\bm{A})$ and $\lambda_{\max}(\bm{A})$ denote its minimum and maximum eigenvalues, respectively.
Let $\bm{0}$  be the all-zero vector.
A set function $f\colon 2^X\to\mathbb{R}$ is called \textit{monotone} if $f(S)\leq f(T)$ whenever $S\subseteq T\subseteq X$.
It is called \textit{submodular} if $f(S\cup T)+f(S\cap T)\leq f(S)+f(T)$ for any $S,T\subseteq X$.
We denote the Lipschitz constant of the function $g$ by $L_g$.
Let $\Pi=\set{\pi\colon[N]\to[K]}$.

\vspace{-1.5pt}
\subsection{Generalized linear models}
\vspace{-1.5pt}
\label{subsec:glm}
% We introduce the likelihood theory of GLM.
Within the framework of GLM \citep{McCullagh_GLM}, the conditional distribution of the response variable $Y$ given the explanatory variable $X$ belongs to the exponential family.
Formally, the probability density function parameterized by $\bm{\theta}$ is given by
\vspace{-1pt}
\begin{align}
    \label{eq:probability_feedback}
    \mathbb{P}(Y|\bm{\theta}; \bm{X}) \propto \exp\crl{Y\bm{X}^\top\bm{\theta}-m(\bm{X}^\top\bm{\theta})},
\end{align}
% \vspace{-2pt}
where we use the unit-dispersion canonical form, since any fixed known dispersion parameter can be absorbed into the sub-Gaussian parameter and tuning constants without changing the regret rates.
Under this normalization, $m\colon\R\to\R$ is a known twice differentiable function and satisfies $\dot{m}(\bm{X}^\top\bm{\theta})=\E\brk{Y|\bm{X}}$ and $\ddot{m}(\bm{X}^\top\bm{\theta})=\mathrm{Var}(Y|\bm{X})$.
In what follows, we set $\mu(\bm{X}^\top\bm{\theta})=\dot{m}(\bm{X}^\top\bm{\theta})$.
The exponential family comprises distributions such as the Gaussian and Bernoulli.

In this setting, given independent samples $Y_1,\dots,Y_n$ conditional on $\bm{X}_1,\dots,\bm{X}_n$, we denote the dataset by $\mathcal{D}=\set{\prn{\bm{X}_i,Y_i}}_{i=1}^n$.
Then, the negative log-likelihood function is given by
$\mathcal{L}(\mathcal{D};\bm{\theta})  = \sum_{i=1}^n \brk{m\prn{\bm{X}_i^\top\bm{\theta}} - Y_i\bm{X}_i^\top\bm{\theta}}+ \text{constant}$.
Then, since $m$ is differentiable, the minimum likelihood estimator (MLE) is given by the minimizer of $\sum_{i=1}^n\brk{m\prn{\bm{X}_i^\top\bm{\theta}} - Y_i\bm{X}_i^\top\bm{\theta}}$ (see \cite{NIPS_2010_Filippi_GLM,icml2017glmcontextbandits} for details).

However, in our problem, using the MLE requires an initial exploration.
To avoid this issue, we employ a regularized MLE with ridge regularization.
Here, the regularized negative log-likelihood corresponding to the regularized MLE takes the form
$\tilde{\mathcal{L}}(\mathcal{D};\bm{\theta},\lambda) = \mathcal{L}(\mathcal{D};\bm{\theta}) + \frac{\lambda}{2}\nrm{\bm{\theta}}_2^2$.
% \looseness=-1

\vspace{-2pt}
\subsection{Submodular welfare problem}
\vspace{-2pt}
\label{subsec:submod}
The submodular welfare problem was first studied by \citet{ACM2001_Lehmann_Combinatorial_auctions} and is defined as follows:
\textit{the submodular welfare problem, given $m$ items and $n$ players with submodular utility functions $w_i\colon 2^{[m]}\to\R_{\geq0}$, is the problem of maximizing $\sum_{i=1}^n w_i(S_i)$, where $S_1,\dots,S_n$ are disjoint subsets of the item set.}
A limitation for the submodular welfare problem is that no approximation better than $1-1/e$ can be achieved unless $P=NP$ \citep{Algorithmica2008_submod}.
There are two commonly considered oracle models, the value oracle model and the demand oracle model.
The former model returns the value of utility $w_i(S)$, and the latter model returns the set $S$ which maximizes $w_i(S)-\sum_{j\in S}p_j$ given an assignment of prices to items $p\colon[m]\to\R$.
We call an algorithm $\epsilon$-approximate algorithm if the value obtained by the algorithm (we denote it $\text{ALG}$ ) satisfies the inequality $\epsilon \text{OPT}_{sub} \leq \text{ALG}$, where $\text{OPT}_{sub}$ denotes the optimal value.
Under the value oracle model, there exists an approximate algorithm that achieves the following approximation: 
\vspace{-1pt}
\begin{lemma}[{\citealp[Section 5]{STOC2008submodular_welfare}}]
    \label{lem:submod_e}
    The submodular welfare problem admits a $1 - 1/e$-approximation in the value-oracle model when the utility functions are monotone submodular.
\end{lemma}
\vspace{-1pt}
% In addition, even without the monotonicity assumption, \citet{lee_2009_submodular} implies the existence of a $1/4$-approximate solution, since the submodular welfare problem is maximizing a submodular function under matroid constraints.

\vspace{-3pt}
\section{Combinatorial allocation bandits}
\vspace{-3pt}
\label{sec:problem}
% \vspace{-2pt}
This section introduces a setting of Combinatorial Allocation Bandits (CAB) and our intention for constructing the problem.
% \vspace{-2pt}
\vspace{-2pt}
\subsection{Problem setting}
\vspace{-2pt}
{
\begin{problem}[t]
\caption{Combinatorial Allocation Bandits (CAB)}
\vspace{-2.8pt}
\label{alg:setting}
\begin{algorithmic}[1]
    \For {$t=1,\dots,T$}
        \State 
        \textbf{Context:}
        Observe $\set{{\bm{\phi}}_t(i,a)}_{a\in\brk{K}}$ for each $i\in[N]$.
        \State 
        \textbf{Action:}
        Choose allocation $\pi_t\in\Pi$.
        \State 
        \textbf{Feedback:}
        Observe $y_t(i)\sim \mathbb{P}(\cdot | {\bm{\theta}}^*; {\bm{\phi}}_t (i, \pi_t(i)))$ for each $i\in\brk{N}$.
        \State 
        \textbf{Reward (unobserved):}
        Receive
        $f_t(\pi_t; {\bm{\theta}}^\ast) = \sum_{a\in[K]} r\prn{\sum_{i\in\pi_t^{-1}(a)}\mu({\bm{\phi}}_t(i,a)^\top{\bm{\theta}}^\ast)}$.
    \EndFor
\end{algorithmic}
\vspace{-2.8pt}
\end{problem}}
We introduce CAB (Problem~\ref{alg:setting}), a novel online learning problem.
At round $t$, for each user $i \in [N]$, the learner obtains a context set $\set{{\bm{\phi}}_t(i,a)}_{a \in [K]}$ with $\nrm{{\bm{\phi}}_t(i,a)} \leq 1$, where the contexts are chosen by an oblivious adversary before the learning process begins.
Note that $i$ is not associated with a particular user, but instead denotes the index reflecting the order of observation.
Given the observations, the learner determines the allocation $\pi_t\in\Pi$ at round $t$.
Subsequently, based on $\pi_t$, the learner observes the feedback $y_t(i)=y_t(i,\pi_t(i))$ for each $i\in[N]$.
Let $\mathcal{F}_t=\sigma\prn{y_1(1),\dots,y_1(N),\dots,y_t(1),\dots,y_t(N),\pi_1, \dots,\pi_t }$, where $\sigma(\mathcal{A})$ is the smallest $\sigma$-algebra containing $\mathcal{A}$.
Then, $\prn{\mathcal{F}_t}_t$ is a filtration.
Denote the conditional probability and expectation given the history of observations by 
$\P_t(\cdot)=\P(\cdot\mid\mathcal{F}_{t-1})$ and $\E_t\brk{\cdot}=\E\brk{\cdot\mid\mathcal{F}_{t-1}}.$

We consider that the feedback $y_t(i)$ observed by the learner follows a GLM with an unknown parameter ${\bm{\theta}}^\ast \in \R^d$ (\cref{subsec:glm}).
We assume that $\set{y_t(i)}_{i \in [N]}$ are conditionally independent given $\set*{{\bm{\phi}}_t(i,\pi_t(i))}_{i \in [N]}$, and that $\nrm{{\bm{\theta}}^\ast}_2\leq D$ for some constant $D>0$.
Let $\Theta=\set*{\bm{\theta}\in\R^d:\nrm{\bm{\theta}}_2\leq D}$.
Accordingly, the probability density function (or the probability mass function) of $y_t(i)$ can be expressed using \eqref{eq:probability_feedback} as $\mathbb{P}(y_t(i)|{\bm{\theta}}^\ast;{\bm{\phi}}_t(i,\pi_t(i)))$, whose mean is given by $\mu({\bm{\phi}}_t(i,\pi_t(i))^\top {\bm{\theta}}^\ast)$.
In the target applications, the observed feedback $y_t(i)$ is used to assess outcomes such as matching success or other positive feedback.
Accordingly, its conditional mean $\mu({\bm{\phi}}_t(i,a)^\top\bm{\theta})$ represents a non-negative latent quantity.
In this problem, we assume that the deviation between the observation and its mean, $y_t(i) - \mu({\bm{\phi}}_t(i,\pi_t(i))^\top {\bm{\theta}}^\ast)$, is sub-Gaussian with parameter $\sigma > 0$.
That is, for any $t \in [T]$, $i \in [N]$, and $\xi\in\R$ it holds that $\E_t\brk{e^{\xi\prn{y_t(i) - \mu({\bm{\phi}}_t(i,\pi_t(i))^\top{\bm{\theta}}^\ast)}}}\leq e^{\xi^2\sigma^2/2}$.
Furthermore, motivated by prior GLM bandit analyses \citep{icml2017glmcontextbandits,ICML_2021_Jun_linear_bandit,liu_2025_combinatorial_log_bandits}, we impose the following assumption on $\mu$.
\vspace{-0.4pt}
\begin{assumption}\label{asp:kappa}
    Let $\mathcal{B}=\set{(\bm{x},\bm{\theta}):\nrm{\bm{x}}_2\leq1,\ \nrm{\bm{\theta}-\bm{\theta}^\ast}_2\leq\max\set{D+1,2D}}$.
    We assume that $\mu$ is first-order differentiable and Lipschitz continuous, that $\mu(\bm{x}^\top\bm{\theta})\geq0$ for all $(\bm{x},\bm{\theta})\in\mathcal{B}$, and that there exists a known constant $\kappa_\mu>0$ such that $\kappa_\mu\leq\inf_{(\bm{x},\bm{\theta})\in\mathcal{B}}\dot{\mu}(\bm{x}^\top\bm{\theta})$.
\end{assumption}
\vspace{-1.5pt}
% At the end of each round, an unknown satisfaction is generated for each arm according to the allocation $\pi_t$.
At the end of each round, the allocation $\pi_t$ induces arm-side satisfaction for each arm.
Specifically, we consider the satisfaction $r\prn{\sum_{i\in\pi_t^{-1}(a)}\mu({\bm{\phi}}_t(i,a)^\top{\bm{\theta}}^\ast)}$, where $r\colon\R_{\geq0}\to\R_{\geq0}$ is a known concave and monotone increasing function bounded by $M$.
$r$ is a platform-specified model of arm-side satisfaction, rather than an additional unknown function to be learned online.
The uncertainty lies in the unknown parameter $\bm{\theta}^\ast$.
We assume that $r$ is Lipschitz continuous.
For convenience, we define
\vspace{-0.6pt}
\begin{equation}
\label{eq:definition_f_t}
f_t(\pi;{\bm{\theta}}) = \sum_{a\in\brk{K}} r\prn[\big]{\sum_{i\in\pi^{-1}(a)}\mu({\bm{\phi}}_t(i,a)^\top{\bm{\theta}})}.
\end{equation}
% In CAB, the goal of the learner is to maximize the cumulative satisfaction $\sum_{t=1}^{T} f_t(\pi_t;{\bm{\theta}}^\ast)$.
% To evaluate the learner\char39s performance, it is natural to use regret, which is defined as the cumulative difference between the achieved value and the optimal value at each round.
% We define the standard regret as
%     $\mathcal{R}_T
%     =
%     \sum_{t=1}^{T}
%     \prn*{
%         f_t(\pi_t^\ast;\bm{\theta}^\ast)
%         -
%         f_t(\pi_t;\bm{\theta}^\ast)
%     },$
% where $\pi_t^\ast=\argmax_{\pi\in\Pi} f_t(\pi;\bm{\theta}^\ast)$ is the optimal allocation at round $t$.
In CAB, the goal of the learner is to maximize the cumulative satisfaction $\sum_{t=1}^{T} f_t(\pi_t;{\bm{\theta}}^\ast)$.
We measure its performance by the standard regret, the cumulative gap to the per-round optimum,
$\mathcal{R}_T
=
\sum_{t=1}^{T}
\prn*{
    f_t(\pi_t^\ast;\bm{\theta}^\ast)
    -
    f_t(\pi_t;\bm{\theta}^\ast)
},$
where $\pi_t^\ast=\argmax_{\pi\in\Pi} f_t(\pi;\bm{\theta}^\ast)$ is the optimal allocation at round $t$.
However, in general, maximizing $f_t$ is NP-hard even when ${\bm{\theta}}^\ast$ is known (\cref{thm:nphard}). 
Thus, some algorithms below use an $\alpha$-approximate optimization oracle ($0<\alpha\leq 1$).
For such algorithms, we also use the $\alpha$-approximate regret
% \begin{equation}
    $\mathcal{R}_T^\alpha
    =
    \sum_{t=1}^{T}
    \prn*{
        \alpha f_t(\pi_t^\ast;\bm{\theta}^\ast)
        -
        f_t(\pi_t;\bm{\theta}^\ast)
    }.$
% \end{equation}
% =========================

Next, we explain the modeling intuition behind the above formulation. 
In many applications, the observed feedback $y_t(i)$ is a realization of the underlying match quality between user $i$ and arm $a$. 
For example, in a job-matching platform, $y_t(i)\in\set{0,1}$ indicates whether a match occurs, and $\mu(\bm{\phi}_t(i,a)^\top\bm{\theta}^\ast)$ can be interpreted as the match probability.

The observed feedback $y_t(i)$ depends not only on the match quality between user $i$ and arm $a$, but also on random user-side factors after the assignment, such as the user\char39s acceptance decision or availability constraints.
For this reason, we do not model an arm\char39s satisfaction using the realized outcomes $y_t(i)$.
Instead, we use the latent match probability $\mu(\bm{\phi}_t(i,a)^\top\bm{\theta}^\ast)$ between user $i$ and arm $a$ as the basic quantity for evaluating the users assigned to each arm.
Thus, for each arm $a$, we use the total expected number of matches
$
\sum_{i\in\pi_t^{-1}(a)} \mu\prn{\bm{\phi}_t(i,a)^\top\bm{\theta}^\ast},
$
as the input to the satisfaction function $r$. 
The arm-side satisfaction in round $t$ is therefore modeled as
$
r\prn{\sum_{i\in\pi_t^{-1}(a)} \mu\prn{\bm{\phi}_t(i,a)^\top\bm{\theta}^\ast}}.
$

We consider a concave and nondecreasing satisfaction function $r$.
Monotonicity means that a greater expected number of successful matches should not reduce the arm\char39s utility.
Concavity captures diminishing marginal value under constraints such as interview capacity, screening costs, or budget constraints \citep{Pratt_economics,Mas-Colell_economics}.
Thus, the objective discourages excessive concentration on a small number of arms and can induce more balanced allocations without explicit fairness constraints.
% \looseness=-1

% \vspace{-2pt}
\vspace{-3.45pt}
\section{Algorithm and theoretical results}
\label{sec:alg_and_reget}
\vspace{-3.45pt}
% \vspace{-2pt}
This section presents UCB and TS algorithms for CAB and their regret analyses.

\vspace{-3pt}
\subsection{Upper confidence bound algorithm}
\label{subsec:ucb}
\vspace{-3pt}

% \begin{wrapfigure}[11]{r}{0.54\textwidth}
%     \vspace{-35pt}
%     \begin{minipage}{0.54\textwidth}
%         \begin{algorithm}[H]
%         \caption{CAB-UCB}
% \label{alg:ucb}
%             {\begin{algorithmic}[1]
%     \Require {The total rounds $T$, the number of users $N$, and tuning parameter $\lambda_0$ and $c_1$.}
%     \State {$\mathcal{D}_1\gets\varnothing$ and ${\bm{V}_1}\gets\lambda_0 \bm{I}$.}
%     \For {$t=1,\dots,T$}
%         % \State {$V_{t} \gets \lambda_0 {\bm{I}} + \sum_{s =1}^{t-1}\sum_{i=1}^N {\bm{x}_s} (i){\bm{x}_s} (i)^\top$.}
%         \State {$\bar{\bm{\theta}}_t\gets\argmin_{\bm{\theta}\in\R^d}\tilde{\mathcal{L}}(\mathcal{D}_t;\bm{\theta},\kappa_\mu\lambda_0)$.}
%         \State {Choose $\pi_t=\argmax_{\pi\in\Pi}\prn{f_t(\pi;\bar{\bm{\theta}}_t)+g_t(\pi)}.$
%         % , where $f_t$ and $g_t$ are defined in \eqref{eq:definition_f_t} and \eqref{eq:definition_g_t}
%         }
%         \State {Observe $y_t(i)$ for any $i\in\brk{N}$.}
%         \State {${\bm{x}_t}(i)\gets {\bm{\phi}}_t(i,\pi_t(i))$ for any $i\in[N]$ and $\mathcal{D}_{t+1}\gets \mathcal{D}_{t}\cup\set{\prn{\bm{x}_t(i),y_t(i)}}_{i\in[N]}$.}
%         \State {${\bm{V}_{t+1}} \gets \lambda_0 {\bm{I}} + \sum_{s =1}^{t}\sum_{i=1}^N \bm{x}_s (i)\bm{x}_s (i)^\top$.}
%     \EndFor
% \end{algorithmic}}
%         \end{algorithm}
%     \end{minipage}
%     % \vspace{-3pt}
% \end{wrapfigure}

% \begin{wrapfigure}[11]{r}{0.54\textwidth}
%     \vspace{-35pt}
%     \begin{minipage}{0.54\textwidth}
\begin{algorithm}[t]
    \caption{CAB-UCB}
    \vspace{-2.8pt}
        % \small
\label{alg:ucb}
            {\begin{algorithmic}[1]
    \Require {The total rounds $T$, the number of users $N$, tuning parameter $\lambda_0$ and $c_1$, and access to an $\alpha$-approximate optimization oracle.}
    \State {$\mathcal{D}_1\gets\varnothing$ and ${\bm{V}_1}\gets\lambda_0 \bm{I}$.}
    \For {$t=1,\dots,T$}
        \State {$\bar{\bm{\theta}}_t\gets\argmin_{\bm{\theta}\in\R^d}\tilde{\mathcal{L}}(\mathcal{D}_t;\bm{\theta},\kappa_\mu\lambda_0)$.}
        \State {Call an $\alpha$-approximate optimization oracle for $f_t(\pi;\bar{\bm{\theta}}_t)+g_t(\pi)$ and let $\pi_t$ denote its output.
        % , where $f_t$ and $g_t$ are defined in \eqref{eq:definition_f_t} and \eqref{eq:definition_g_t}
        }
        \State {Observe $y_t(i)$ for any $i\in\brk{N}$.}
        \State {${\bm{x}_t}(i)\gets {\bm{\phi}}_t(i,\pi_t(i))$ for any $i\in[N]$ and $\mathcal{D}_{t+1}\gets \mathcal{D}_{t}\cup\set{\prn{\bm{x}_t(i),y_t(i)}}_{i\in[N]}$.}
        \State {${\bm{V}_{t+1}} \gets \lambda_0 {\bm{I}} + \sum_{s =1}^{t}\sum_{i=1}^N \bm{x}_s (i)\bm{x}_s (i)^\top$.}
    \EndFor
\end{algorithmic}}
\vspace{-2.8pt}
        \end{algorithm}
%     \end{minipage}
%     % \vspace{-3pt}
% \end{wrapfigure}
Our proposed method, CAB-UCB, follows the UCB principle, a standard approach in bandit algorithm design
\citep{annals_of_ststs_Lai_1987_MAB,Auer_2002_bandit,siam_qin2014contextual,icml2017glmcontextbandits}.
% \looseness=-1
CAB-UCB has two parameters, $\lambda_0>0$ and $c_1>0$.
The parameter $\lambda_0$ determines the regularization strength for MLE.
In addition, $\lambda_0$ plays the role of ensuring that $\lambda_{\min}\prn{{\bm{V}_t}}$ is strictly positive.
The parameter $c_1$ controls exploration, and a larger value results in a greater degree of exploration.

In each round, we compute $\bar{\bm{\theta}}_t$, the regularized MLE of $\bm{\theta}^\ast$, using the set of observations $\mathcal{D}_t=\set{\prn{\bm{x}_s(i),y_s(i)}}_{i\in[N],s<t}$, where $\bm{x}_s(i)={\bm{\phi}}_s(i,\pi_s(i))$.
In calculating $\pi_t$, we balance exploitation and exploration by maximizing the estimated total satisfaction $f_t(\pi;\bar{\bm{\theta}}_t)$ and the bonus term
\vspace{-1pt}
\begin{equation}
    \label{eq:definition_g_t}
    g_t(\pi)=c_1\sum_{i=1}^N\nrm{{\bm{\phi}}_t(i,\pi(i))}_{{\bm{V}_{t}^{-1}}},
\end{equation}
% \vspace{-2pt}
where ${\bm{V}_t}=\lambda_0 \bm{I} + \sum_{s=1}^{t-1}\sum_{i=1}^N\bm{x}_s(i)\bm{x}_s(i)^\top$.
The bonus term is related to the width of the confidence interval.

\vspace{-1pt}
\subsubsection{Regret analysis}
\label{subsec:regret_ucb}
\vspace{-1pt}
% We present the regret upper bound achieved by \cref{alg:ucb}.
\cref{alg:ucb} achieves the following regret:
    \vspace{-0.3pt}
\begin{theorem}
    \label{thm:main_regret_ucb}
    Fix any $\delta\in(0,1)$.
    \cref{alg:ucb}, with tuning parameters chosen as in \cref{app:details_ucb}, achieves with probability at least $1-2\delta$ the regret bound
    $\mathcal{R}_T^\alpha = \tilde{O}\prn{\kappa_\mu^{-1} L_r L_\mu D \prn{d \sqrt{NT} + dN}}$.
\end{theorem}
% \vspace{-4pt}

% The dependence on $\lambda_0$ arises from the regularization used to guarantee the invertibility of ${\bm{V}_t}$.
% While this dependence can be avoided through initial exploration, doing so incurs an additional regret term and requires assumptions on ${\bm{\phi}}_t$.
This bound matches the known lower bound for contextual combinatorial linear bandits \citep[Theorem 7]{takemura2021near_combinatorial} in its dependence on $d$, $T$, and $N$, up to logarithmic factors.
If $f(\pi;\bm{\theta}_1)-f(\pi;\bm{\theta}_2)\leq C \sum_{i} \abs{\mu(\bm{x}_t(i)^\top\bm{\theta}_1)-\mu(\bm{x}_t(i)^\top\bm{\theta}_2)}$
holds for any $\pi\in\Pi$ and $\bm{\theta}_1,\bm{\theta}_2\in\R^d$, where $C>0$ is a constant, then we can derive a similar bound for a general CCGLS as well.
While we use the regularization in \cref{alg:ucb}, a similar bound can be obtained via an initial exploration.
Using the initial exploration, however, introduces an additional regret term and requires assumptions on ${\bm{\phi}}_t$
\footnote{\textit{E.g.,} \cite{icml2017glmcontextbandits} study generalized linear contextual bandits using initial exploration, where ${\bm{\phi}}_t$ is generated in a stochastic manner and additional regularity assumptions are imposed.}.
The full statement and proof of \cref{thm:main_regret_ucb} are given in \cref{app:details_ucb}.

\vspace{-1pt}
\subsubsection{Approximate optimization oracle construction}
\label{subsec:implementation}
\vspace{-1pt}
We next describe one concrete way to instantiate the $\alpha$-approximate optimization oracle used by CAB-UCB.
For a parameter $\bm{\theta}$, define $g_a(S;\bm{\theta})=r(\sum_{i\in S}\mu(\bm{\phi}_t(i,a)^\top\bm{\theta}))$.
By the concavity and monotonicity of $r$, each $g_a$ is monotone submodular. 
Hence, maximizing $f_t(\pi;\bm{\theta})$ is an instance of the submodular welfare problem. 
The UCB bonus $g_t(\pi)$ is additive, so $f_t(\pi;\bm{\theta})+g_t(\pi)$ remains within the same problem class. 
Therefore, \cref{lem:submod_e} provides a concrete $\alpha$-approximate optimization oracle for the CAB-UCB allocation step.

% We next describe one concrete way to instantiate the $\alpha$-approximate optimization oracle assumed in our framework.
% Our theoretical results only require the existence of such an oracle.
% Under the non-negative-mean assumption in \cref{sec:problem}, define $g_a\colon 2^{[N]} \to \R_{\geq0}$ by $g_a(S;\bm{\theta})=r\prn{\sum_{i\in S}\mu({\bm{\phi}}_t(i,a)^\top\bm{\theta})}$.
% Then, by the concavity and monotonicity of $r$, each $g_a$ is a monotone submodular function.
% We express $f_t$ as $f_t(\pi;\bm{\theta}) = \sum_{a\in[K]}g_a(\pi^{-1}(a))$.
% Hence, maximizing $f_t(\pi;\bm{\theta})$ corresponds to the submodular welfare problem, and $\max_{\pi\in\Pi}\prn{f_t(\pi;\bm{\theta})+g_t(\pi)}$ remains in the same class because $g_t$ is additive.
% Therefore, \cref{lem:submod_e} gives a concrete oracle construction.

% For the TS objective, $\max_{\pi\in\Pi}\prn{f_t(\pi;\bm{\theta})+h_t(\pi;{\mathcal{E}}_t)}$ need not be monotone because $h_t$ may take negative values, although it is still submodular as the sum of a submodular term and a modular term.
% Standard constant-factor approximation results for non-monotone submodular maximization under partition matroid constraints, such as \citet{lee_2009_submodular}, can be used here.
% To keep the presentation brief, we treat this step abstractly and simply assume the required $\alpha$-approximate oracle.

\vspace{-2pt}
\subsection{Thompson sampling algorithm}
\label{subsec:ts}
\vspace{-2pt}

Here, we introduce CAB-TS,  which is based on the TS method \citep{thompson1933likelihood}. 
The TS algorithm has been proposed for various bandits problems.
Theoretically, in these problems, the TS algorithm often has worse regret upper bounds than the UCB algorithms \citep{agrawal2013_ts_cb,ICDM2019_Takemura_ArmWise_Randomization,aistats2017_abeille_ts}.
Empirically, however, TS has often been found to perform comparably to, and sometimes better than, UCB algorithms \citep{NIPS_chapelle_2011_ts,JMLR_2012may_ts_context_bandit,wang_icml_2018_combinatorial}.

CAB-TS has parameters $\lambda_0>0$ and $a>0$, which control the regularization strength and the degree of exploration, respectively.
Up to the step of computing $\bar{\bm{\theta}}_t$ via the regularized MLE, the procedure is identical to \cref{alg:ucb}.
However, the subsequent method for computing $\pi_t$ differs.
In CAB-TS, after computing $\bar{\bm{\theta}}_t$, for each $i\in[N]$, we independently sample $\tilde{\bm{\epsilon}}_t(i)$ from $\mathcal{N}\prn{\bm{0},a^2\bm{H}_t^{-1}}$, where 
$\bm{H}_1 = L_\mu\lambda_0 \bm{I}$ 
and 
${\bm{H}_{t}}= \sum_{s=1}^{t-1}\sum_{i=1}^N \dot{\mu}({\bm{x}_{s}} (i)^\top\bar{\bm{\theta}}_t)\prn{{\bm{x}_{s}} (i){\bm{x}_{s}} (i)^\top+ \frac{\lambda_0}{N(t-1)}{\bm{I}}}$ for $t\geq2$.
The isotropic regularization term is intentionally placed inside the weighted sum. 
With the scaling $\lambda_0/N(t-1)$, the matrices inside this weighted sum add up to $\bm{V}_{t}$, because the $(t-1)$ copies of $\lambda_0\bm{I}/(t-1)$ add to $\lambda_0\bm{I}$. 
Since $\dot{\mu}(z)\in[\kappa_\mu,L_\mu]$, this gives $\kappa_\mu \bm{V}_{t}\preceq \bm{H}_{t}\preceq L_\mu \bm{V}_{t}$.
In what follows, we collectively denote these samples by $\tilde{\mathcal{E}}_t=\set{\tilde{\bm{\epsilon}}_t(1),\dots,\tilde{\bm{\epsilon}}_t(N)}$.
We choose the allocation $\pi_t$ to maximize the objective function $f_t(\pi;\bar{\bm{\theta}}_t)+h_t(\pi;\tilde{\mathcal{E}}_t)$, where  
\vspace{-2pt}
\begin{align}
    \label{eq:definition_h}
    h_t(\pi;{\tilde{\mathcal{E}}}) = \sum_{i=1}^{N}{\bm{\phi}}_t(i,\pi(i))^\top{\tilde{\bm{\epsilon}}}(i).
\end{align}
% \vspace{-4pt}
% \vspace{-2pt}
% \begin{align}
%     \label{eq:definition_f_tilde}
%     \tilde{f}\prn*{\pi;\bm{\bm{\theta}}}=\sum_{a\in[K]}r\prn*{\sum_{i\in\pi^{-1}}\mu\prn{{\bm{\phi}}_t(i,a)^\top\bm{\theta}(i)}}.
% \end{align}
We approximate the posterior of $\bm{\theta}^\ast$ by the Laplace approximation.
In this setting, we sample $\tilde{\bm{\epsilon}}_t(i)$ i.i.d.\ across users, because this independence is needed to preserve enough variance in the aggregate perturbation, although using a common $\tilde{\bm{\epsilon}}_t$ is also a natural idea (see \cref{rem:reason_for_iid} for details).

% \vspace{-2pt}
\vspace{-1pt}
\subsubsection{Regret analysis}
\label{subsec:regret_ts}
\vspace{-1pt}
% \vspace{-2pt}
Unlike CAB-UCB, our analysis of CAB-TS assumes access to an exact optimization oracle for $f_t(\pi;\bar{\bm{\theta}}_t)+h_t(\pi;\tilde{\mathcal{E}}_t)$.
Accordingly, the regret guarantee for CAB-TS is stated in terms of the standard regret $\mathcal{R}_T$.
CAB-TS achieves the following regret bound:
% \looseness=-1
% \vspace{-0.4pt}
\begin{theorem}
    \label{thm:bound_TS}
    Fix any $\delta\in(0,1/T)$.
    CAB-TS, with tuning parameters chosen as in \cref{app:details_ts}, achieves the regret upper bound
    $\E\brk{\mathcal{R}_T}=\tilde{O}\prn{\kappa_\mu^{-1} L_r L_\mu D \prn{dN \sqrt{T} + dN^{3/2}}}$.
\end{theorem}
% \vspace{-1pt}

This regret upper bound has an extra $\sqrt{N}$ factor compared with CAB-UCB.
The proof of \cref{thm:bound_TS} partitions $\Pi$ into a ``good'' subset and its complement and lower-bounds the probability that $\pi_t$ lies in the good subset.
This argument uses the exact optimality of $\pi_t$.
However, the approximate optimization oracle may return a near-optimal allocation outside this subset, so the same probability lower bound is unavailable (see \cref{rem:exact_oracle_ts} for details).
The full statement and proof of \cref{thm:bound_TS} are given in \cref{app:details_ts}.
Although the objective function $f_t(\pi;\bar{\bm{\theta}}_t)+h_t(\pi;\tilde{\mathcal{E}}_t)$ is designed so that $\tilde{\mathcal{E}}_t$ enters linearly, one can also consider a heuristic variant of CAB-TS that instead maximizes
$\sum_{a\in[K]}r\prn[\big]{ \sum_{i\in\pi^{-1}(a)}\mu\prn[\big]{{\bm{\phi}}_t(i,a)^\top(\bar{\bm{\theta}}_t+\tilde{\bm{\epsilon}}_t\prn{i})}}.$
We describe this heuristic in \cref{app:variant_ts} and include it only for empirical comparison.
Similar to CAB-UCB, the analysis for CAB-TS can also be extended to a general CCGLS under the appropriate assumption.
\looseness=-1

% \vspace{-1pt}
% \subsection{Approximate optimization oracle construction}
% \label{subsec:implementation}
% \vspace{-1pt}
% We next describe one concrete way to instantiate the $\alpha$-approximate optimization oracle assumed in our framework.
% Our theoretical results only require the existence of such an oracle.
% Under the non-negative-mean assumption in \cref{sec:problem}, define $g_a\colon 2^{[N]} \to \R_{\geq0}$ by $g_a(S;\bm{\theta})=r\prn{\sum_{i\in S}\mu({\bm{\phi}}_t(i,a)^\top\bm{\theta})}$.
% Then, by the concavity and monotonicity of $r$, each $g_a$ is a monotone submodular function.
% We express $f_t$ as $f_t(\pi;\bm{\theta}) = \sum_{a\in[K]}g_a(\pi^{-1}(a))$.
% Hence, maximizing $f_t(\pi;\bm{\theta})$ corresponds to the submodular welfare problem, and $\max_{\pi\in\Pi}\prn{f_t(\pi;\bm{\theta})+g_t(\pi)}$ remains in the same class because $g_t$ is additive.
% Therefore, \cref{lem:submod_e} gives a concrete oracle construction.

% For the TS objective, $\max_{\pi\in\Pi}\prn{f_t(\pi;\bm{\theta})+h_t(\pi;{\mathcal{E}}_t)}$ need not be monotone because $h_t$ may take negative values, although it is still submodular as the sum of a submodular term and a modular term.
% Standard constant-factor approximation results for non-monotone submodular maximization under partition matroid constraints, such as \citet{lee_2009_submodular}, can be used here.
% To keep the presentation brief, we treat this step abstractly and simply assume the required $\alpha$-approximate oracle.

\vspace{-2pt}
\subsection{One-pass update algorithm}
\label{subsec:one_pass}
\vspace{-2pt}
CAB-UCB and CAB-TS solve a regularized MLE at round $t$ using all prior observations.
For simplicity, throughout this runtime discussion, we suppress the dependence on the $d$ and $N$.
With full-history computation, if the solver uses $I_t$ iterations, the update cost grows linearly with the number of rounds, namely $O(t I_t)$.
To avoid this linear growth, we replace the MLE update with a one-pass parameter update based on OMD \citep{zhang2025generalizedlinearbanditsoptimal}.
However, this variant requires an additional self-concordance assumption on the link function.
This assumption is used in prior work \citep{liu_2025_combinatorial_log_bandits,zhang2025generalizedlinearbanditsoptimal}, and holds for commonly used link functions such as the logistic and Poisson link functions.
\vspace{-0.3pt}
\begin{assumption}
    \label{asp:self_concordance}
    There exists a constant $R>0$, such that for any $z\in\R$ $\abs{\ddot{\mu}(z)}\leq R \dot{\mu}(z)$.
\end{assumption}
We call the resulting procedure CAB-OFU with one-pass OMD update, where OFU denotes the principle of optimism in the face of uncertainty, which uses a confidence set constructed from past observations \cite{agrawal1995sample_mean_bandit}.
For each user, define the negative log-likelihood
    $\ell_{t,i}(\bm{\theta})
    =
    - y_t(i)\bm{x}_t(i)^\top \bm{\theta}
    + m\prn{\bm{x}_t(i)^\top \bm{\theta}},$
where $\dot{m}=\mu$. We form the quadratic surrogate
    $\tilde{\ell}_t(\bm{\theta})
    =
    \sum_{i=1}^N
    \prn{
        \inpr{\nabla_{\bm{\theta}} \ell_{t,i}(\bm{\theta}_t), \bm{\theta}-\bm{\theta}_t}
        +
        \frac{1}{2}\nrm{\bm{\theta}-\bm{\theta}_t}_{\nabla_{\bm{\theta}}^2 \ell_{t,i}(\bm{\theta}_t)}^2
    },$
and update the parameter by
\vspace{-2.3pt}
\begin{equation}
    \label{eq:onepass_theta_update}
    \bm{\theta}_{t+1}
    =
    \argmin_{\bm{\theta}\in\Theta}
    \prn[\big]{
        \tilde{\ell}_t(\bm{\theta})
        +
        \frac{1}{2\eta}\nrm{\bm{\theta}-\bm{\theta}_t}_{\bm{Q}_t}^2
    },
\end{equation}
and set
    $\bm{Q}_{t}
    =
    \lambda_{\mathrm{op}} \bm{I}
    +
    \sum_{s=1}^{t-1}\sum_{i=1}^N
    \nabla_{\bm{\theta}}^2 \ell_{s,i}(\bm{\theta}_{s+1})
    .
    $
Once $\bm{Q}_t$ is maintained incrementally, this subproblem uses only the current-round surrogate and $\bm{Q}_t$.
If \eqref{eq:onepass_theta_update} is solved in $\tilde{I}_t$ iterations, the update cost is $O(\tilde{I}_t)$.
We use the confidence set
    $C_t(\delta)
    =
    \set*{
        \bm{\theta}\in\Theta
        \,\middle|\,
        \nrm{\bm{\theta}-\bm{\theta}_t}_{\bm{Q}_t}
        \leq
        \beta_t(\delta)
    },$
with
    $\beta_t(\delta)
    =
    \tilde{O}\prn{D\sqrt{
        \lambda_{\mathrm{op}}}
        +
        \sqrt{d}
    }
    ,
    $
suppressing logarithmic factors.
Using this confidence set, we choose $\pi_t$ satisfying $\max_{\bm{\theta}\in C_t(\delta)} f_t(\pi_t;\bm{\theta})\geq \alpha \max_{\pi\in\Pi}
    \max_{\bm{\theta}\in C_t(\delta)}
    f_t(\pi;\bm{\theta}).$

By this update, we achieve the following regret bound:
% \vspace{-0.4pt}
\begin{theorem}
    \label{thm:onepass_regret}
    The above update with suitable tuning parameters achieves $\mathcal{R}_T^\alpha=\tilde{O}(\kappa_\mu^{-1}\max\set{d,N}\prn{\sqrt{dNT}+dN})$, where we display only the dependence on $\kappa_\mu$, $d$, $N$, and $T$.
\end{theorem}
% This proof introduces the $\kappa_\mu$ dependence because it compares the weighted matrix $\bm{Q}_t$ with the unweighted design matrix $\bm{V}_t$.
The full statement and proof are given in \cref{app:onepass}.

\vspace{-2pt}
\section{Experiments}
\vspace{-2pt}
\label{sec:experiments}
This section empirically evaluates CAB-UCB and CAB-TS using synthetic data. 
Our code to reproduce the experimental results is shared as Supplementary Material.
\vspace{-1pt}
\subsection{Setting}
\vspace{-1pt}
In synthetic experiments, we define the five-dimensional feature vector $\bm{\phi}(i,a)= \lambda \bm{\phi}_{pop}(i,a) + (1 - \lambda) \bm{\phi}_{base}(i,a),$
where $\bm{\phi}_{pop}$ and $\bm{\phi}_{base}$ are sampled from the standard normal distribution. 
We impose $\bm{\phi}_{pop}(i,a) < \bm{\phi}_{pop}(i,a+1)$ component-wise for all users $i$ and all $a \in [K-1]$.  
The parameter $\lambda$ controls arm-popularity strength.
A large $\lambda$ makes all users prefer arms in a similar order, making it difficult to jointly maximize matches and arm satisfaction.
We use $\mu(x) = 1 / (1 + \exp{(-x)})$ and $r(x) = \min\{ x, \beta \}$ as feedback mean and satisfaction functions, respectively.
Thus, matches beyond $\beta$ have no additional effect on satisfaction.
Smaller $\beta$ yields faster saturation.

We compare CAB-UCB (\cref{alg:ucb}), CAB-TS ($\epsilon$) (\cref{alg:ts}), the heuristic variant CAB-TS ($\theta$) (\cref{alg:ts2}), and One-pass OMD (\cref{alg:onepass_ofu}) against three baselines, ``Random'', ``Max match'', and ``FairX''.
Random selects arms uniformly at random. 
Max match is a UCB algorithm that aims to maximize cumulative expected matches. 
FairX, a UCB-based fairness algorithm proposed by \citet{pmlr-v139-wang21b}, ensures that each arm receives a share of exposure that is proportional to its expected match, aiming to reduce over-selection of specific arms.
Full baseline specifications are given in \cref{app:baseline_alg_detail}.
The optimization routines used in the experiments are practical proxies for the offline allocation steps appearing in the theory.
In particular, CAB-TS ($\theta$) is included only as an empirical heuristic variant and is not covered by our theory, while the one-pass variant is an empirical proxy for the theory-side oracle-based procedure.

\vspace{-1pt}
\subsection{Results}
\vspace{-1pt}
\begin{figure*}[t]
    \captionsetup[subfigure]{skip=1pt}
    \captionsetup{skip=2.2pt}
    \centering
    \includegraphics[width=1\linewidth]{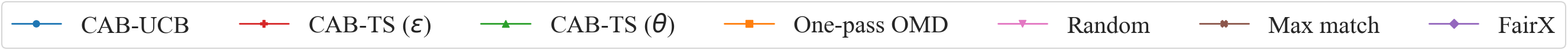} \\
    \vspace{-1pt}
    \begin{subfigure}[t]{0.32\linewidth}
        \centering
        \includegraphics[width=\linewidth,trim=0 8pt 0 0,clip]{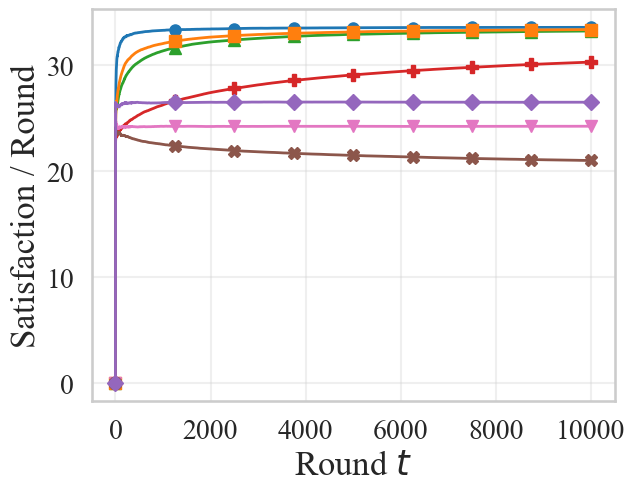}
        \caption{Cumulative satisfaction per round.}
        \label{fig:satisfaction_timestep}
    \end{subfigure}
    \hfill
    \begin{subfigure}[t]{0.32\linewidth}
        \centering
        \includegraphics[width=\linewidth,trim=0 8pt 0 0,clip]{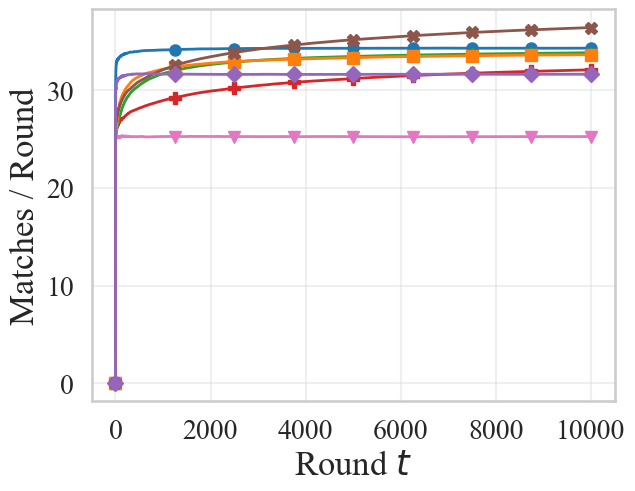}
        \caption{Cumulative matches per round.}
        \label{fig:matches_timestep}
    \end{subfigure}
    \hfill
    \begin{subfigure}[t]{0.32\linewidth}
        \centering
        \includegraphics[width=\linewidth,trim=0 8pt 0 0,clip]{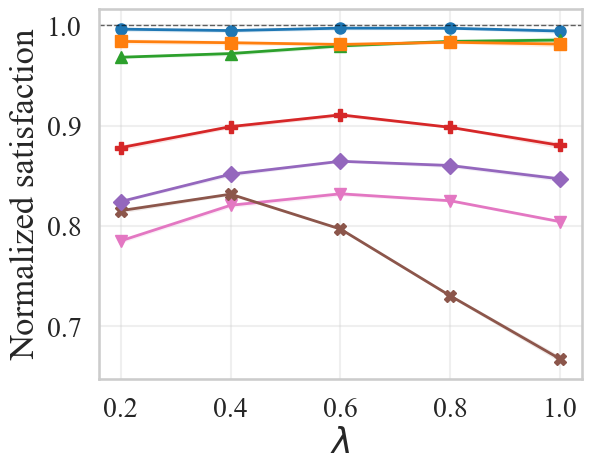}
        \caption{Normalized cumulative satisfaction under varying $\lambda$.}
        \label{fig:lambda_sweep}
    \end{subfigure}

    \vspace{-2pt}

    \begin{subfigure}[t]{0.32\linewidth}
        \centering
        \includegraphics[width=\linewidth,trim=0 8pt 0 0,clip]{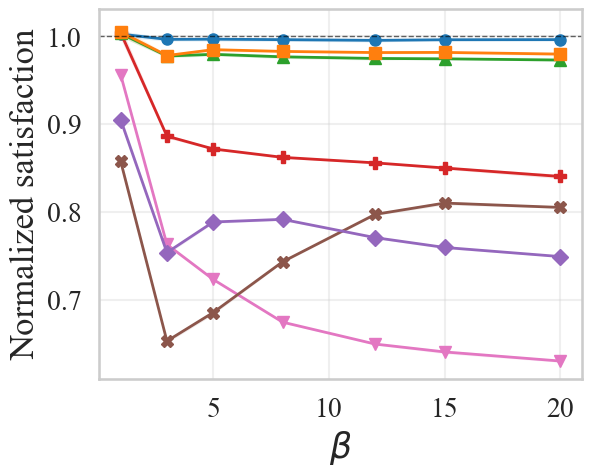}
        \caption{Normalized cumulative satisfaction under varying $\beta$.}
        \label{fig:beta_sweep}
    \end{subfigure}
    \hfill
    \begin{subfigure}[t]{0.32\linewidth}
        \centering
        \includegraphics[width=\linewidth,trim=0 8pt 0 0,clip]{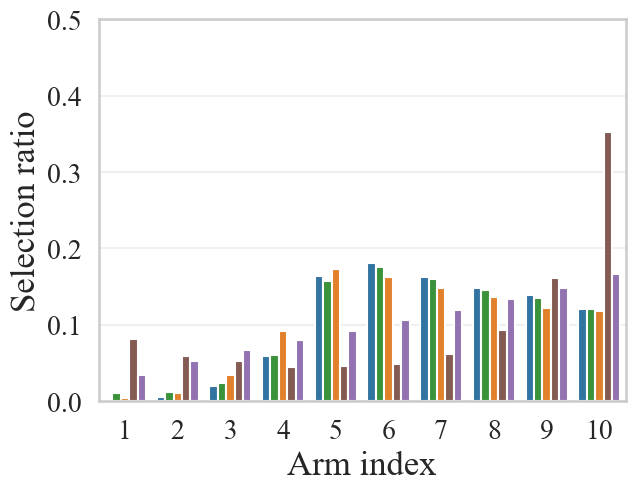}
        \caption{Selection share over the entire horizon.}
        \label{fig:selection_hist}
    \end{subfigure}
    \hfill
    \begin{subfigure}[t]{0.32\linewidth}
        \centering
        \includegraphics[width=\linewidth,trim=0 8pt 0 0,clip]{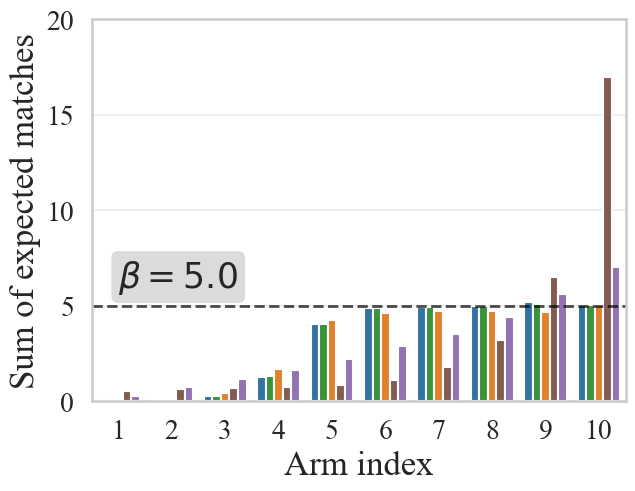}
        \caption{Average expected matches of each arm.}
        \label{fig:expected_matches_hist}
    \end{subfigure}
    \caption{Figures (a) and (b) use the default setting ($N=50$, $K=10$, $T=10000$, $\lambda=0.5$, $\beta=5.0$). Figures (c) and (d) vary $\lambda$ and $\beta$, respectively, with all other parameters fixed. Figures (e) and (f) show learned allocations for $\lambda=1$ and $\beta=5.0$: selection shares over the full horizon and arm-wise average expected matches over the last $100$ steps. Solid lines show run means, and shaded regions indicate $95$\% confidence intervals where applicable.}
    \label{fig:main_results}
    \vspace{-10pt}
\end{figure*}
We consider four settings.
Fix $N=50$ and $K=10$ for all settings.
The default comparisons use $T=10000$, $\lambda=0.5$, and $\beta=5.0$ over $10$ runs. 
The $\lambda$- and $\beta$-sweeps vary the respective parameter with $T=5000$ over $5$ runs. 
Histogram analyses use $T=5000$, $\lambda=1$, and $\beta=5.0$ with $5$ runs.
We compare overall performance under the default setting, examine the effect of popularity concentration and satisfaction saturation by varying $\lambda$ and $\beta$, and inspect the learned allocations.

We evaluate the learning performance of the proposed CAB algorithms in terms of cumulative satisfaction and matches.
\cref{fig:satisfaction_timestep} compares average cumulative satisfaction per round, and \cref{fig:matches_timestep} reports average cumulative matches per round.
As shown in \cref{fig:satisfaction_timestep}, all CAB variants substantially outperform the Max match and FairX, indicating that they successfully learn to optimize the satisfaction objective.
CAB-UCB achieves the highest average cumulative satisfaction throughout the horizon.
As intended, Max match yields the most matches in \cref{fig:matches_timestep}. 
However, this does not translate into maximizing satisfaction, since satisfaction is a concave function of the assigned users' expected matches.
Under this fixed configuration, CAB-UCB achieves high cumulative satisfaction earlier than CAB-TS and the one-pass variant. 
This pattern is at least consistent with its sharper dependence on $d$ and $N$ in the regret bound, although the theory does not directly predict finite-horizon transients and the guarantees are stated under different oracle assumptions.

We vary $\lambda$ and $\beta$ to test whether the proposed methods remain effective when arm popularity is concentrated and when satisfaction strongly saturates.
The $\lambda$-sweep tests robustness to preference concentration across users, while the $\beta$-sweep tests robustness to different concavity levels in the arm-side satisfaction function.
In \cref{fig:lambda_sweep,fig:beta_sweep}, each point reports the cumulative satisfaction under the corresponding parameter setting, normalized by that of a reference allocation computed using the true parameter and the same allocation routine as CAB-UCB.

The $\lambda$-sweep in \cref{fig:lambda_sweep} examines robustness to preference concentration across users.
As $\lambda$ increases, users rank arms more similarly, causing match-oriented algorithms to over-concentrate assignments on a few popular arms, even after their satisfaction is already close to the saturation level $\beta$.
This increases the gap between maximizing matches and maximizing satisfaction, since additional matches on well-served arms yield little marginal satisfaction.
Consistent with this mechanism, the normalized cumulative satisfaction of Max match decreases with $\lambda$, whereas CAB-UCB remains nearly flat and stays close to the reference value across the entire sweep.
The gap to FairX further suggests that reducing exposure imbalance alone is insufficient unless the allocation is explicitly aligned with the saturation structure of satisfaction.

The $\beta$-sweep in \cref{fig:beta_sweep} tests robustness to satisfaction saturation.
As $\beta$ increases, $r(x)=\min\set{x,\beta}$ becomes closer to linear over a wider range, and the satisfaction objective becomes less sensitive to over-allocation to already well-served arms.
Accordingly, Max match improves with $\beta$, while the proposed methods maintain strong performance across the entire sweep.
These results support the claim that explicitly modeling arm-side satisfaction is particularly important when utility saturates strongly, while CAB-UCB remains robust across different saturation regimes.

% =======================

% We next examine whether our methods avoid excessive concentration on a small number of arms.
% \cref{fig:selection_hist} shows the share of arm selections over the entire horizon.
% When $\lambda=1$, all users prefer arms in the same order. 
% Thus, Max match concentrates heavily on the most popular arms, whereas our methods select the arms at more balanced rates, except for the least popular ones.
% FairX, in contrast, also allocates users to unpopular arms that are rarely selected by tur methods, because it emphasizes fairness with respect to expected matches.
% \cref{fig:expected_matches_hist} reports the arm-wise total expected matches, averaged over the last $100$ steps.
% With $\beta=5.0$, Max match assigns high expected matches mainly to the most popular arms, whereas the proposed methods allocate not only to the popular arms but also to several moderately popular arms so that their expected matches are kept around the saturation threshold.
% These results suggest that our methods learn allocations that are well aligned with the saturation structure of the satisfaction objective through the estimations and learning strategies.

The allocation histograms in \cref{fig:selection_hist,fig:expected_matches_hist} examine whether the proposed methods avoid excessive arm concentration.
\cref{fig:selection_hist} shows selection shares over the full horizon. 
When $\lambda=1$, Max match heavily favors the most popular arms because all users share the same arm ranking.
In contrast, the proposed methods select arms at more balanced rates, except for the least popular ones.
FairX also reduces concentration, but allocates to unpopular arms rarely selected by the proposed methods, as it targets fairness in expected matches rather than the saturation structure of arm-side utility.
\cref{fig:expected_matches_hist} reports arm-wise total expected matches averaged over the last $100$ steps.
Max match assigns high expected matches mainly to the most popular arms, often exceeding the saturation threshold $\beta=5.0$.
The proposed methods instead allocate not only to the most popular arms but also to several moderately popular arms, keeping expected matches closer to the saturation threshold.
These results support the claim that CAB improves satisfaction by aligning allocations with diminishing returns, not merely by reducing imbalance.
\looseness=-1

Additional experimental results, including runtime comparisons, sweeps over other parameters, and further details of the experimental setup, are provided in \cref{app:experiment}.

\begin{comment}
Next, we evaluate whether the one-pass variant yields a computational advantage.
\cref{fig:runtime_onepass} compares the average execution time per round for different horizon lengths.
The one-pass variant is substantially faster than the MLE-based methods, and its per-round runtime remains nearly constant across different values of $T$.
In contrast, the per-round runtime of CAB-UCB and CAB-TS increases noticeably as $T$ grows.
This trend is consistent with the design of the one-pass variant, which avoids solving a regularized MLE at every round.
\end{comment}

\vspace{-2pt}
\section{Conclusions}
\vspace{-2pt}
We proposed CAB, developed its algorithms, established regret upper bounds, and conducted experimental evaluations of its performance.
We conclude with several future research directions.
One possible direction is to improve the dependence on $\kappa_\mu$.
In more specific settings, recent analyses have reduced such dependence under a self-concordance assumption \citep{liu_2025_combinatorial_log_bandits,zhang2025generalizedlinearbanditsoptimal}.
Second, the current CAB-TS analysis assumes an exact optimization oracle, whereas CAB-UCB only requires an approximate one.
Relaxing this assumption would require overcoming the proof obstacle that an approximate optimization oracle may return an unfavorable allocation among near-optimal allocations.
% \looseness=-1

% We proposed CAB, developed its algorithm, established regret upper bounds, and further conducted experimental evaluations of its performance.
% We conclude with several future research directions.
% One possible direction is to improve the dependence on $\kappa_\mu$.
% An approach toward this improvement is considered in \cite{liu_2025_combinatorial_log_bandits,zhang2025generalizedlinearbanditsoptimal}.
% However, instead of assuming Lipschitz continuity with respect to $\mu$, these approaches adopt a self-concordance assumption.
% Another natural extension is to consider dynamic environments, in which the set of available arms is not fixed but may evolve, with some arms leaving as a result of the allocation process and new arms emerging.
% Addressing such settings would require algorithms that can adapt to both statistical uncertainty and changes in the action set.

% \section*{Acknowledgments}
% aaa

\bibliography{reference}
\bibliographystyle{plainnat}
\newpage
% \clearpage
\appendix

% \onecolumn
\section{Notations}
\label{app:notation}
\cref{tab: notation} summarizes the symbols used in this paper.

\begin{table}[H]
\centering
\caption{Notation}
\renewcommand{\arraystretch}{1.2}
% \small
\begin{adjustbox}{max width=\textwidth}
\begin{tabular}{ll}
\toprule
Symbol & Meaning \\ 
\midrule
$T \in \mathbb{N}$ & time horizon \\
$N \in \mathbb{N}$ & the number of users\\
$K \in \mathbb{N}$ & the number of arms\\ 
$d \in \mathbb{N}$ & dimension of feature vector \\
$\sigma$ & sub-Gaussian parameter \\
$\kappa_\mu$ & parameter satisfying \cref{asp:kappa}\\
% \midrule
$L_\mu$ & Lipschitz constant of function $\mu$\\
$L_r$ & Lipschitz constant of function $r$\\
$D>0$ & Upper bound on $\nrm{\bm{\theta}^\ast}_2$ \\
$\Theta= \set{\bm{\theta} \in \R^d : \nrm{\bm{\theta}}_2 \leq D}$ & bounded parameter space \\
\midrule
$\mu\colon\R\to\R_{\geq0}$ & expectation of feedback \\
$r\colon\R_{\geq0}\to\R_{\geq0}$& satisfaction function \\
$\bm{\phi}_t(i,a)$ & feature vector according to user $i$ and arm $a$ at round $t$\\
$y_t(i)$ & feedback for $i$ at round $t$\\
$\Pi=\set{\pi:[N]\to[K]}$ & set of all functions from $[N]$ to $[K]$\\
$f_t(\pi;\bm{\theta}) = \sum_{a\in\brk{K}} r\prn{\sum_{i\in\pi^{-1}(a)}\mu(\bm{\phi}_t(i,a)^\top\bm{\theta})}$ & cumulative expected satisfaction at round $t$ \\
$\pi_t\in\Pi$ & chosen allocation at round $t$ \\
$\pi_t^\ast=\argmax_{\pi\in\Pi} f_t(\pi;\bm{\theta}^\ast)$ & optimal allocation at round $t$ \\
$\bm{x}_t(i)=\bm{\phi}_t(i,\pi_t(i))$ & chosen feature vector for user $i$ at round $t$ \\
$\bm{V}_{t}=\sum_{s=1}^{t-1}\sum_{i=1}^N\bm{x}_{s}(i)\bm{x}_{s}(i)^\top + \lambda_0 \bm{I} $ & information matrix at round $t$ \\
$\bm{\theta}^\ast$ & unknown parameter \\
$\bar{\bm{\theta}}_t$ & MLE of $\bm{\theta}^\ast$ at round $t$ \\
$\mathcal{R}_T^\alpha=\sum_{t=1}^{T}\prn*{\alpha f_t(\pi_t^\ast;\bm{\theta}^\ast) - f_t(\pi_t;\bm{\theta}^\ast)}$ & approximate regret \\
$\mathcal{R}_T=\sum_{t=1}^{T}\prn*{ f_t(\pi_t^\ast;\bm{\theta}^\ast) - f_t(\pi_t;\bm{\theta}^\ast)}$ & standard regret \\
\bottomrule
\end{tabular}
\label{tab: notation}
\end{adjustbox}
\end{table}

\section{Related work}
\label{app:related_work}
We introduce several related works.
\paragraph{Contextual Combinatorial Semi-bandits and Generalized Linear Models}
From a technical perspective, a closely related problem setting is that of contextual combinatorial semi-bandits (CCS).
CCS was first studied by \citet{siam_qin2014contextual}.
CCS are problems in which, at each round, one first observes the arms together with their associated contexts, then selects a combination of arms based on these observations and past outcomes, and finally observes the reward, which depends on the chosen contexts and an unknown parameter.
They consider a general framework that includes nonlinear reward functions and propose a UCB algorithm, while assuming a linear model for the feedback.
In the linear-feedback setting, their algorithm achieves a regret upper bound of $\tilde{O}\prn{{\max\set{\sqrt{d}, \sqrt{N}}}\sqrt{dNT}}$.
To the best of our knowledge, the best known bound is $\tilde{O}\prn{d\sqrt{NT}+dN}$, which is achieved by the UCB algorithm of \citet{takemura2021near_combinatorial}.
In \citet{ICDM2019_Takemura_ArmWise_Randomization}, in addition to UCB algorithms, a TS algorithm was studied under the setting where the reward is linear in the feedback, and it was shown that a regret upper bound of $\mathcal{R}_T = \tilde{O}\prn{\max\set{d, \sqrt{dN}}\sqrt{dNT}}$ can be achieved.
For CCS with linear reward functions, the existing works have investigated lower bounds in addition to upper bounds.\footnote{Lower bounds for general reward functions are not meaningful. Indeed, if the reward is constant, the regret can always be reduced to $0$.}
\citet{aistats_kveton2015tight} established a lower bound for non-contextual combinatorial semi-bandits in terms of the number of base arms and the action size. Applying their bound to an instance with $d$ base arms and action size $N$ yields
$\Omega\prn{\min\set{\sqrt{dNT},NT}}$.
For CCS, an improved lower bound of $\Omega\prn{\min\set{d\sqrt{NT}+dN,NT}}$ was derived by \citet{takemura2021near_combinatorial}.

As a related research direction, generalized linear (contextual) bandits have also been studied.
In this framework, a generalized linear model is adopted as the feedback model, rather than a linear model, and the contextual bandits setting was first investigated by \citet{NIPS_2010_Filippi_GLM}.
In the non-combinatorial case, existing work on GLM has primarily focused on UCB algorithms \citep{icml2017glmcontextbandits,nips_jun_2017_glm,neurips_lee_2024_glm,zhang2025generalizedlinearbanditsoptimal}, while TS algorithms have also been developed \citep{nips_jun_2017_glm,aistats2021_qin_glm_ts,AAAI_Kim_Lee_Paik_2023_TS_GLM}.
On the other hand, to the best of our knowledge, there are only a few studies that consider combinatorial settings.
One notable example is \citet{liu_2025_combinatorial_log_bandits}, which uses the UCB algorithm and considers a contextual setting in which the feedback is sampled from a Bernoulli distribution with its mean specified by a logistic model.

\paragraph{Fair Allocation}
One line of research with a closely related idea is fair resource allocation, although the technical relevance is limited.
Among them, some studies use an evaluation metric called $\alpha$-fairness, namely $fair_\alpha(x)=\frac{x^{1-\alpha}-1}{1-\alpha}$ ($0\leq\alpha<1$). 
The function $fair_\alpha(x)$ is concave, and its use is close in spirit to the idea of this work \citep{NEURIPS2020_semih_fair_allocation,acm_salem_fair_allocation_2022,NEURIPS2023_sinha_fair_resource_allocation}.
However, they consider an objective function of the form $\sum_{i}fair(\sum_{t=1}^T w_t(i))$, where $w_t(i)$ denotes the utility for $i$ at round $t$, which evaluates the overall fairness across the entire horizon. 
This differs from our setting.
If one emphasizes the final fairness of the allocation over the whole period, their objective function is more appropriate. 
In contrast, we consider that dissatisfaction over a short period may also lead to churn.
Therefore, as an objective function in CAB, we argue that $\sum_{t=1}^T f_t$ is appropriate.

Moreover, there are technical differences. 
Representative aspects are the feedback model and the action space.
Since the true utility value or the reward vector is observed at the end of each round in their works, their feedback assumption is stronger than the setting we consider (\eg the platform observes only whether a match occurred).
Regarding the action space, while \citet[Section 4]{NEURIPS2023_sinha_fair_resource_allocation} discusses integrality constraints, their formulation essentially considers fractional decisions rather than combinatorial ones.
Thus, although the high-level idea is similar to our work, the technical setting is entirely different, and a direct comparison of theoretical results is impossible.

\paragraph{Bandits With Fairness Constraints}
One line of research addressing fairness in the context of the bandit problem is \citet{joseph2016_fairness}.
In their work, fairness is defined as allocating arms without favoring any particular arm, based on their expected rewards.
On the other hand, a problem that deals with a concept of fairness similar to that considered in this study is the combinatorial sleeping bandits with fairness constraints proposed by \citet{li2019combinatorial_fairness}.
In this problem, for each arm $a$, a minimum selection count $n_a$ is specified, and the objective is to maximize the cumulative expected reward under this constraint.
\citet{li2019combinatorial_fairness} provided a UCB algorithm, while \citet{huang2020thompson_fairness} later proposed a TS algorithm.
\citet{xu2020combinatorial_fairness} focused on the objective function, considering a setting where the total reward $R$, obtained as a linear combination of the rewards from the arms, is transformed by a strictly concave function $f$, resulting in the objective function $f(R)$.
More recently, studies have considered constraints that combine knapsack constraints with fairness constraints \citep{liu2022combinatorial_knap_fairness}.

\paragraph{Bandits With Knapsacks}
The Bandits with Knapsacks (BwK) problem extends the standard bandit setting by introducing knapsack-type resource constraints.
Given budget limitations on multiple resources, the learner can no longer obtain rewards once any resource budget is depleted.
The goal is to maximize the cumulative expected reward.
As a method of introducing constraints into the multi-armed bandits, \citet{badanidiyuru2013bwk} proposed the BwK framework, which incorporates budget constraints.
Building on this framework, subsequent studies have extended the knapsack constraints to linear contextual bandits \citep{badanidiyuru2014resourceful_bwk,agrawal2016linear_bwk} and combinatorial semi-bandits \citep{sankararaman2018combinatorial_bwk}.

\section{NP-hardness}
\label{app:detail_setting}
We discuss the computational complexity of this problem.
As the following theorem shows, computing $\pi^\ast_t$ and $\pi_t$ is NP-hard.
\begin{theorem}
    \label{thm:nphard}
    Consider a set of $N$ items and $K$ players, and let $r\colon \mathbb{R}_{\geq0} \to \mathbb{R}_{\geq 0}$ be a monotone concave function.
    We consider the problem of maximizing $\sum_{i\in[K]} v_i(S_i)$ subject to $\bigcup_{i\in[K]} S_i = [N]$ and $S_i \cap S_j = \emptyset$ for $i \neq j$, where $v_a \colon 2^{[N]} \to \mathbb{R}_{\geq 0}$ is defined by $v_a(S) = r\left(\sum_{i \in S} w_{i,a}\right)$, and $w_{i,a} > 0$ denotes the value of item $i$ for agent $a$.
    Then, there exists a concave and monotone increasing function $r$ for which the above problem is NP-hard.
\end{theorem}
Our proof uses an approach similar to that of \citet[Theorem 10]{ACM2001_Lehmann_Combinatorial_auctions}.
\begin{proof}
    We will perform a reduction from the well-known NP-complete problem ``Subset Sum'' (as detailed in \citet{Garey1979_nphard}). 
    The problem is as follows: given a sequence of integers $a_1,\dots,a_m$ and a target total $t$, the objective is to determine if a subset $S$ of these integers exists such that the sum of the elements in $S$ equals $t$ (i.e., $\sum_{i\in S}a_i=t$). 
    Based on this input, we will construct two valuations for the $m$ items.

    We consider the decision problem of determining whether $\max f(\calf)=Vt+V-t$ holds for $\calf=\set{U_1,U_2}$ with $U_1\cap U_2=\emptyset$, $U_1\cup U_2=U$, $f(\calf)=v_1(U_1)+v_2(U_2)$, $r(x)=\min\set{Vt, x}$, and $V=\sum_{i\in U}a_i$.
    Let $w_{i,1} = Va_i, w_{i,2} = a_i$. We allocate $S$ to valuation $1$ and $S^c$ to valuation $2$. We examine three cases: $\sum_{i\in S}a_i=t$, $\sum_{i\in S}a_i<t$, and $\sum_{i\in S}a_i>t$.
    \begin{itemize}
        \item If $\sum_{i\in S}a_i=t$, then $v_1(S)+v_2(S^c)=Vt+V-t$.
        \item If $\sum_{i\in S}a_i<t$, then $v_1(S)+v_2(S^c)=V\sum_{i\in S}a_i+V-\sum_{i\in S}a_i< Vt+V-t$.
        \item If $\sum_{i\in S}a_i>t$, then $v_1(S)+v_2(S^c)=Vt+V-\sum_{i\in S}a_i< Vt+V-t$.
    \end{itemize}
    Therefore, the instance of ``Subset Sum'' is a Yes-instance precisely when the proposition is satisfied, and a No-instance otherwise.
\end{proof}
\section{Details of \cref{subsec:ucb,subsec:ts}}\label{app:details_main_result}
In this section, we provide the omitted proofs in \cref{subsec:ucb,subsec:ts}.

\subsection{Technical lemmas}
Here, in preparation for the proofs of the theorems, we introduce technical lemmas that are commonly used in the proofs of \cref{thm:main_regret_ucb,thm:bound_TS}.

The following lemma guarantees that the regularized MLE $\bar{{\bm{\theta}}}_t$ remains close to the true parameter with high probability.
While \citet[Lemma 9]{aistas_kveton20a_glm} relies on an initial exploration phase to establish this property, an analogous result can be derived by employing regularization instead.
\begin{lemma}
    \label{lem:nrm_theta_bar_t}
    Assume that $\lambda_0\geq \frac{\sigma^2}{\kappa_\mu^{2}} \prn*{d\log\prn*{1+\frac{NT}{d}}+2\log\frac{1}{\delta}}$.
    Then, for any $\delta\in(0,1)$, with probability at least $1-\delta$, for all $t\in[T]$,
    \[\nrm{\bar{{\bm{\theta}}}_t-{\bm{\theta}}^\ast}_2\leq D+1.\]
\end{lemma}

\begin{proof}
    Let
    $
        S_t = \sum_{s =1}^{t-1}\sum_{i=1}^N \big(y_s (i)-\mu({\bm{x}_{s}} (i)^\top {\bm{\theta}}^\ast)\big){\bm{x}_{s}} (i).
    $
    First, using \citet[Lemma A]{chen1999}, we show that the following proposition holds:
    \[
        \nrm{S_t-\kappa_\mu{\lambda_0}{\bm{\theta}}^\ast}_{{\bm{V}_{t}^{-1}}}\leq \kappa_\mu (D+1) \sqrt{{\lambda_0}}
        \Rightarrow
        \nrm{\bar{{\bm{\theta}}}_t-{\bm{\theta}}^\ast}_2\leq D+1.
    \]
Define the map
$
{\bm{G}_{t}}({\bm{\theta}})=\sum_{s=1}^{t-1}\sum_{i=1}^N \mu({\bm{x}_{s}}(i)^\top{\bm{\theta}}){\bm{x}_{s}}(i) + \kappa_\mu{\lambda_0}{\bm{\theta}} .
$
From the definition of $\bar{{\bm{\theta}}}_t$, we have 
% \[
$
{\bm{G}_{t}}(\bar{{\bm{\theta}}}_t)=\sum_{s=1}^{t-1}\sum_{i=1}^N y_s(i){\bm{x}_{s}}(i).
$
Thus, it holds that
\[
{\bm{G}_{t}}(\bar{{\bm{\theta}}}_t)-{\bm{G}_{t}}({\bm{\theta}}^\ast)=
\sum_{s=1}^{t-1}\sum_{i=1}^N (y_s(i)-\mu({\bm{x}_{s}}(i)^\top{\bm{\theta}}^\ast)){\bm{x}_{s}}(i) - \kappa_\mu{\lambda_0}{\bm{\theta}}^\ast
= S_t-\kappa_\mu{\lambda_0}{\bm{\theta}}^\ast.
\]
Moreover, for any ${\bm{\theta}}\in\set{{\bm{\theta}}\mid\nrm{{\bm{\theta}}-{\bm{\theta}}^\ast}_2\leq D+1}$, from \cref{asp:kappa}, we have
\begin{align*}
\bm{\nabla}_{\bm{\theta}} {\bm{G}_{t}}({\bm{\theta}})
&=\sum_{s=1}^{t-1}\sum_{i=1}^N \dot\mu({\bm{x}_{s}}(i)^\top{\bm{\theta}}){\bm{x}_{s}}(i){\bm{x}_{s}}(i)^\top + \kappa_\mu{\lambda_0} {\bm{I}}\\
&\succeq \kappa_\mu\sum_{s=1}^{t-1}\sum_{i=1}^N {\bm{x}_{s}}(i){\bm{x}_{s}}(i)^\top + \kappa_\mu{\lambda_0} {\bm{I}}
= \kappa_\mu {\bm{V}_{t}}.
\end{align*}
Therefore, applying \citet[Lemma A]{chen1999} 
to the map ${\bm{\theta}}\mapsto {\bm{V}_{t}}^{-1/2}\prn{{\bm{G}_{t}}({\bm{\theta}})-{\bm{G}_{t}}({\bm{\theta}}^\ast)}$ 
yields the proposition
\[
\nrm{S_t-\kappa_\mu{\lambda_0}{\bm{\theta}}^\ast}_{{\bm{V}_{t}^{-1}}}\leq \kappa_\mu (D+1)\sqrt{\lambda_{\min}({\bm{V}_{t}})}
\Rightarrow
\nrm{\bar{{\bm{\theta}}}_t-{\bm{\theta}}^\ast}_2\leq D+1.
\]
Since $\lambda_{\min}({\bm{V}_{t}})\geq {\lambda_0}$, we obtain
\begin{align*}
% \label{eq:prop_s_t}
\nrm{S_t-\kappa_\mu{\lambda_0}{\bm{\theta}}^\ast}_{{\bm{V}_{t}^{-1}}}\leq \kappa_\mu (D+1)\sqrt{{\lambda_0}}
\Rightarrow
\nrm{\bar{{\bm{\theta}}}_t-{\bm{\theta}}^\ast}_2\leq D+1.
\end{align*}

Next, we upper bound $\nrm{S_t-\kappa_\mu{\lambda_0}{\bm{\theta}}^\ast}_{{\bm{V}_{t}^{-1}}}$.
By the triangle inequality and ${\bm{V}_{t}}\succeq {\lambda_0} {\bm{I}}$, we have
\begin{align*}
\nrm{S_t-\kappa_\mu{\lambda_0}{\bm{\theta}}^\ast}_{{\bm{V}_{t}^{-1}}}
&\leq \nrm{S_t}_{{\bm{V}_{t}^{-1}}}+\kappa_\mu{\lambda_0}\nrm{{\bm{\theta}}^\ast}_{{\bm{V}_{t}^{-1}}}
\leq \nrm{S_t}_{{\bm{V}_{t}^{-1}}}+\kappa_\mu\sqrt{{\lambda_0}}\nrm{{\bm{\theta}}^\ast}_2\\
&\leq \nrm{S_t}_{{\bm{V}_{t}^{-1}}}+\kappa_\mu D \sqrt{{\lambda_0}},
\end{align*}
where the last inequality follows from $\nrm{{\bm{\theta}}^\ast}_2\leq D$.
Thus, we have 
\begin{align}
    \label{eq:prop_s_t}
    \nrm{S_t}_{{\bm{V}_{t}^{-1}}}\leq \kappa_\mu \sqrt{{\lambda_0}}
    \Rightarrow
    \nrm{\bar{{\bm{\theta}}}_t-{\bm{\theta}}^\ast}_2\leq D+1.
\end{align}
Here, from \citet[Theorem 1]{NIPS2011_Abbasi_bandit}, with probability at least $1-\delta$, for all $t\geq 1$, it holds that
% \[
% \nrm{S_t}_{{\bm{V}_{t}^{-1}}}^2
% \leq 2\sigma^2\log\!\Big(\frac{\det({\bm{V}_{t}})^{1/2}\det({\lambda_0} {\bm{I}})^{-1/2}}{\delta}\Big)
% \leq \sigma^2\Big(d\log\Big(1+\frac{}{d{\lambda_0}}\Big)+2\log\frac{1}{\delta}\Big).
% \]
    \begin{equation}
        \label{eq:bound_S_t}
        \nrm{S_t}_{{\bm{V}_{t}^{-1}}}^2\leq 2\sigma^2 \log\prn*{\frac{\det\prn{{\bm{V}_{t}}}^{1/2}\det\prn*{\lambda_0 {\bm{I}}}^{-1/2}}{\delta}}
        \leq \sigma^2 \prn*{d\log\prn*{1+\frac{NT}{d\lambda_0}}+2\log\frac{1}{\delta}},
    \end{equation}
    where the second inequality follows from $\det({\bm{V}_{t}})\leq \prn{\prn{\tr (\lambda_0 {\bm{I}}) + Nt}/d}^d=\prn{\lambda_0 + N t /d}^d$.
Thus, using \eqref{eq:bound_S_t} and the definition of $\lambda_0$, it holds that
$
\nrm{S_t}_{{\bm{V}_{t}^{-1}}} \leq \kappa_\mu \sqrt{{\lambda_0}}.
$
Therefore, by \eqref{eq:prop_s_t}, we have $\nrm{\bar{{\bm{\theta}}}_t - {\bm{\theta}}^\ast}_2\leq D+1$ for any $t$ with probability at least $1-\delta$.
\end{proof}

The following lemma shows that the regularized MLE also controls the error of the objective function.
\begin{lemma}
    \label{lem:bound_bar_ast}
    Assume that it holds that $\nrm{\bar{{\bm{\theta}}}_t-{\bm{\theta}}^\ast}_2\leq D+1$.
    Then, for any $\delta\in(0,1)$, with probability at least $1-\delta$, for all $t\in[T]$,
    \[\abs{f_t(\pi;{\bm{\theta}}^\ast)-f_t(\pi;\bar{{\bm{\theta}}}_t)}\leq c_1\sum_{i=1}^N \nrm{{\bm{\phi}}_t(i,\pi(i))}_{{\bm{V}_{t}^{-1}}},\]
    where $c_1= \kappa_\mu^{-1} L_r L_\mu \prn*{\sigma\sqrt{d\log\prn*{1+\frac{NT}{d\lambda_0}}+2\log\frac{1}{\delta}} + \kappa_\mu D\sqrt{\lambda_0}}$.
\end{lemma}

\begin{proof}
    Let $S_t = \sum_{s =1}^{t-1}\sum_{i=1}^N \prn{y_s (i)-\mu({\bm{x}_{s}} (i)^\top {\bm{\theta}}^\ast)}{\bm{x}_{s}} (i)$, $\mathcal{D}_1 = \set{{\bm{x}_{s}} (i), y_s (i)}_{i\in[N], s  < t}$, and $\mathcal{D}_2=\set{{\bm{x}_{s}} (i),\mu({\bm{x}_{s}} (i)^\top {\bm{\theta}}^\ast)}_{i\in[N], s  < t}$.
    
    First, we expand \citet[Lemma 1]{aistas_kveton20a_glm} to the regularized MLE.
    It holds that
    \begin{align*}
        S_t &= \bm{\nabla}_{\bm{\theta}}\mathcal{L}(\mathcal{D}_2;{\bm{\theta}}^\ast) - \bm{\nabla}_{\bm{\theta}}\mathcal{L}(\mathcal{D}_1;{\bm{\theta}}^\ast)
        = \bm{\nabla}_{\bm{\theta}}\tilde{\mathcal{L}}(\mathcal{D}_1;\bar{{\bm{\theta}}}_t,\kappa_\mu\lambda_0) - \bm{\nabla}_{\bm{\theta}}\mathcal{L}(\mathcal{D}_1;{\bm{\theta}}^\ast)\\
        &= \bm{\nabla}_{\bm{\theta}}\tilde{\mathcal{L}}(\mathcal{D}_1;\bar{{\bm{\theta}}}_t,\kappa_\mu\lambda_0) - \bm{\nabla}_{\bm{\theta}}\tilde{\mathcal{L}}(\mathcal{D}_1;{\bm{\theta}}^\ast,\kappa_\mu\lambda_0) + \kappa_\mu\lambda_0 {\bm{\theta}}^\ast\\
        &= \bm{\nabla}_{\bm{\theta}}^2 \tilde{\mathcal{L}}(\mathcal{D}_1;{\bm{\theta}}^\prime,\kappa_\mu\lambda_0) \prn{\bar{{\bm{\theta}}}_t-{\bm{\theta}}^\ast} + \kappa_\mu \lambda_0 {\bm{\theta}}^\ast \\
        &= \bm{V} \prn{\bar{{\bm{\theta}}}_t-{\bm{\theta}}^\ast} + \kappa_\mu\lambda_0 {\bm{\theta}}^\ast,
    \end{align*}
    where ${\bm{\theta}}^\prime$ is a convex combination of $\bar{{\bm{\theta}}}_t$ and ${\bm{\theta}}^\ast$, and $\bm{V} = \bm{\nabla}_{\bm{\theta}}^2 \tilde{\mathcal{L}}(\mathcal{D}_1;{\bm{\theta}}^\prime,\kappa_\mu\lambda_0)$.
    We used $\bm{\nabla}_{\bm{\theta}}\mathcal{L}(\mathcal{D}_2;{\bm{\theta}}^\ast) = \bm{\nabla}_{\bm{\theta}}\tilde{\mathcal{L}}(\mathcal{D}_1;\bar{{\bm{\theta}}}_t,\kappa_\mu\lambda_0) = 0$ in the second equality.
    % By \citet[Lemma 1]{aistas_kveton20a_glm}, for $\mathcal{D}_1 = \set{{\bm{x}_{s}} (i), y_s ({\bm{x}_{s}} (i))}_{i\in[k], s  < t}$ and $\mathcal{D}_2=\set{{\bm{x}_{s}} (i),\mu({\bm{x}_{s}} (i)^\top {\bm{\theta}}^\ast)}_{i\in[k], s  < t}$, it holds that 
    % \[S_t = \bm{\nabla}_{\bm{\theta}}^2 \tilde{\mathcal{L}}_0(\mathcal{D}_1;{\bm{\theta}}^\prime) = {\bm{\nabla}_{\bm{\theta}}^2 \mathcal{L}(\mathcal{D}_1;{\bm{\theta}}^\prime)}\prn{\bar{{\bm{\theta}}}_t-{\bm{\theta}}^\ast} + \lambda_0 {\bm{\theta}}^\ast,\]
    % where ${\bm{\theta}}^\prime$ is a convex combination of $\bar{{\bm{\theta}}}_t$ and ${\bm{\theta}}^\ast$.

    Next, we bound 
    $\abs{f_t(\pi;{\bm{\theta}}^\ast)-f_t(\pi;\bar{{\bm{\theta}}}_t)}$.
    \begin{align*}
        \abs{f_t(\pi;{\bm{\theta}}^\ast)-f_t(\pi;\bar{{\bm{\theta}}}_t)} 
        &\leq L_r L_\mu \sum_{i=1}^N \abs{{\bm{x}_t}(i)^\top({\bm{\theta}}^\ast - \bar{{\bm{\theta}}}_t)}\\
        &\leq L_r L_\mu \sum_{i=1}^N \nrm{{\bm{x}_t}(i)}_{{\bm{V}_{t}^{-1}}}\nrm{{\bm{\theta}}^\ast - \bar{{\bm{\theta}}}_t}_{{\bm{V}_{t}}} \\
        &= L_r L_\mu \sum_{i=1}^N \nrm{{\bm{x}_t}(i)}_{{\bm{V}_{t}^{-1}}}\nrm{{\bm{V}^{-1}}S_t - \kappa_\mu\lambda_0 {\bm{V}^{-1}}{\bm{\theta}}^\ast}_{{\bm{V}_{t}}}\\
        &\leq L_r L_\mu \sum_{i=1}^N \nrm{{\bm{x}_t}(i)}_{{\bm{V}_{t}^{-1}}} \prn{\nrm{{\bm{V}^{-1}}S_t}_{{\bm{V}_{t}}} + \kappa_\mu\lambda_0\nrm{{\bm{V}^{-1}}{\bm{\theta}}^\ast}_{{\bm{V}_{t}}}}\\
        &\leq L_r L_\mu \sum_{i=1}^N \nrm{{\bm{x}_t}(i)}_{{\bm{V}_{t}^{-1}}} \prn{\sqrt{S_t^\top {\bm{V}^{-1}} {\bm{V}_{t}} {\bm{V}^{-1}} S_t} + \kappa_\mu\lambda_0 \sqrt{{\bm{\theta}}^{\ast \top} {\bm{V}^{-1}} {\bm{V}_{t}} {\bm{V}^{-1}} {\bm{\theta}}^\ast}}\\
        &\leq \kappa_\mu^{-1}L_r L_\mu \sum_{i=1}^N \nrm{{\bm{x}_t}(i)}_{{\bm{V}_{t}^{-1}}} \prn{\nrm{S_t}_{{\bm{V}_{t}^{-1}}} + \kappa_\mu D\sqrt{\lambda_0}},
    \end{align*}
    where the second inequality follows from the Cauchy--Schwarz inequality,
    and the last inequality follows from the $\kappa_\mu {\bm{V}_{t}}\preceq {\bm{V}}$ on $\nrm{\bar{{\bm{\theta}}}_t-{\bm{\theta}}^\ast}_2\leq D+1$.

    Therefore, combining the above inequality with \eqref{eq:bound_S_t} in the proof of \cref{lem:nrm_theta_bar_t}, we have that, with probability at least $1-\delta$, for all $t\geq 1$,
    \[\abs{f_t(\pi;{\bm{\theta}}^\ast)-f_t(\pi;\bar{{\bm{\theta}}}_t)}\leq \kappa_\mu^{-1} L_r L_\mu \prn*{\sigma\sqrt{d\log\prn*{1+\frac{NT}{d\lambda_0}}+2\log\frac{1}{\delta}} + \kappa_\mu D \sqrt{\lambda_0}}\sum_{i=1}^{N}\nrm{\bm{x}_t(i)}_{{\bm{V}_{t}^{-1}}}.\]
\end{proof}

We can bound $\sum_{t=1}^{T}\sum_{i=1}^N\nrm{{\bm{x}_t}(i)}_{{\bm{V}_{t}^{-1}}}$ using the following lemma:
\begin{lemma}[{\citealt[Lemma 2]{takemura2021near_combinatorial}}]
    \label{lem:upper_bound_x_t}
    Let $\set{{\bm{x}_t}(i)}_{(i,t)\in[N]\times\N}$ be a sequence in $\R^d$ satisfying $\nrm{{\bm{x}_t}(i)}_2\leq 1$.
    For all $t\geq 1$, define ${\bm{V}_{t}} = \lambda_0 {\bm{I}} + \sum_{s =1}^{t-1}\sum_{i=1}^N {\bm{x}_{s}} (i){\bm{x}_{s}} (i)^\top$, where $\lambda_0\geq 0$.
    It holds that
    \begin{align*}
        \sum_{t=1}^{T}\sum_{i=1}^N\min\set*{\frac{1}{\sqrt{N}},\nrm{{\bm{x}_t}(i)}_{{\bm{V}_{t}^{-1}}}}&\leq \sqrt{2 d NT \log\prn*{1+\frac{N T}{d \lambda_0}}},\quad\text{and}\\
        \sum_{t=1}^{T}\sum_{i=1}^N\ind\brk*{\nrm{{\bm{x}_t}(i)}_{{\bm{V}_{t}^{-1}}}>\frac{1}{\sqrt{N}}}&<2dN\log\prn*{1+\frac{N T}{d \lambda_0}}.
    \end{align*}
\end{lemma}

This lemma implies that 
\begin{align}
    \label{eq:sum_nrm_xt_bound}
    \sum_{t=1}^{T}\sum_{i=1}^N\nrm{{\bm{x}_t}(i)}_{{\bm{V}_{t}^{-1}}}
    &=
    \sum_{t=1}^{T}\sum_{i=1}^N
    \ind\brk*{\nrm{{\bm{x}_t}(i)}_{{\bm{V}_{t}^{-1}}}>\frac{1}{\sqrt{N}}}
    \nrm{{\bm{x}_t}(i)}_{{\bm{V}_{t}^{-1}}}
    \nonumber\\
    &\quad+
    \sum_{t=1}^{T}\sum_{i=1}^N
    \ind\brk*{\nrm{{\bm{x}_t}(i)}_{{\bm{V}_{t}^{-1}}}\leq\frac{1}{\sqrt{N}}}
    \nrm{{\bm{x}_t}(i)}_{{\bm{V}_{t}^{-1}}}
    \nonumber\\
    &\leq
    \frac{1}{\sqrt{\lambda_0}}
    \sum_{t=1}^{T}\sum_{i=1}^N
    \ind\brk*{\nrm{{\bm{x}_t}(i)}_{{\bm{V}_{t}^{-1}}}>\frac{1}{\sqrt{N}}}
    +
    \sum_{t=1}^{T}\sum_{i=1}^N
    \min\set*{\frac{1}{\sqrt{N}},\nrm{{\bm{x}_t}(i)}_{{\bm{V}_{t}^{-1}}}}
    \nonumber\\
    &\leq \sqrt{2 d NT \log\prn*{1+\frac{N T}{d \lambda_0}}} + \frac{2dN}{\sqrt{\lambda_0}} \log\prn*{1+\frac{N T}{d \lambda_0}}
\end{align}

\subsection{Details of \cref{subsec:ucb}}\label{app:details_ucb}
We provide the full version of \cref{thm:main_regret_ucb}.
\begin{theorem}
    \label{thm:main_regret_ucb_complete}
    Fix any $\delta\in(0,1)$. If we run \cref{alg:ucb} 
    with $c_1=\frac{L_r L_\mu}{\kappa_\mu} \big(\sigma\sqrt{d\log\prn*{1+\frac{NT}{d\lambda_0}}+2\log\frac{1}{\delta}} + \kappa_\mu D \sqrt{\lambda_0}\big)$ and $\lambda_0\geq \frac{\sigma^2}{\kappa_\mu^{2}} \prn*{d\log\prn*{1+\frac{NT}{d}}+2\log\frac{1}{\delta}}$, 
    then, with probability at least $1-2\delta$, the regret of algorithm is upper bounded by
    \begin{align*}
        \mathcal{R}_T^\alpha
        &\leq 2c_1\prn*{\sqrt{2 d NT \log\prn*{1+\frac{N T}{d \lambda_0}}} + \frac{2dN}{\sqrt{\lambda_0}} \log\prn*{1+\frac{N T}{d \lambda_0}}}
        % \\
        % &=\tilde{O}\prn{\kappa_\mu^{-1} L_r L_\mu \prn{\sigma\sqrt{d} + \kappa_\mu \sqrt{\lambda_0}} \sqrt{dNT}}
        .  
    \end{align*}
\end{theorem}

If we set $\lambda_0 =  \kappa_\mu^{-2}\sigma^2 \prn*{d\log\prn*{1+\frac{NT}{d}}+2\log\frac{1}{\delta}}$, we have
\[
    \mathcal{R}_T^\alpha = \tilde{O}\prn{d\sqrt{NT} + dN}.
\]

\begin{proof}
    We define ${\bm{x}_t}^\ast(i)={\bm{\phi}}_t(i,\pi_t^\ast(i))$, $\Delta_t = \bar{{\bm{\theta}}}_t-{\bm{\theta}}^\ast$.
    First, we bound the one-step regret, $f_t\prn{\pi_t^\ast;{\bm{\theta}}^\ast}-f_t\prn{\pi_t;{\bm{\theta}}^\ast}$. 
    Here, we want to bound the following terms:
    \begin{align*}
    % $
        \alpha f_t\prn{\pi_t^\ast;{\bm{\theta}}^\ast}-f_t\prn{\pi_t;{\bm{\theta}}^\ast}&=
        (\alpha f_t\prn{\pi_t^\ast;\bar{{\bm{\theta}}}_t}-f_t\prn{\pi_t;\bar{{\bm{\theta}}}_t})\\
        &\qquad+\alpha (f_t\prn{\pi_t^\ast;{\bm{\theta}}^\ast}-f_t\prn{\pi_t^\ast;\bar{{\bm{\theta}}}_t})
        +(f_t\prn{\pi_t;\bar{{\bm{\theta}}}_t}-f_t\prn{\pi_t;{\bm{\theta}}^\ast}).
    % $
    \end{align*}
    From the definition of $\pi_t,$ the first term is bounded as 
    \begin{align}
        \label{eq:ucb_first_term}
        % $
        \alpha f_t\prn{\pi_t^\ast;\bar{{\bm{\theta}}}_t}-f_t\prn{\pi_t;\bar{{\bm{\theta}}}_t} 
        \leq 
        g_t(\pi_t)-\alpha g_t(\pi_t^\ast)
        =
        c_1 \sum_{i=1}^{N}\prn{\nrm{{\bm{x}_t}(i)}_{{\bm{V}_{t}^{-1}}} -\alpha \nrm{{\bm{x}_t}^\ast(i)}_{{\bm{V}_{t}^{-1}}}}.
        % $
    \end{align}
    The second and third terms are bounded as 
    \begin{align}
        \label{eq:ucb_secound_term}
        \alpha \prn{f_t\prn{\pi_t^\ast;{\bm{\theta}}^\ast}-f_t\prn{\pi_t^\ast;\bar{{\bm{\theta}}}_t}}&\leq \alpha c_1\sum_{i=1}^N \nrm{{\bm{x}_t}^\ast(i)}_{{\bm{V}_{t}^{-1}}},\\
        \label{eq:ucb_third_term}
        f_t\prn{\pi_t;\bar{{\bm{\theta}}}_t}-f_t\prn{\pi_t;{\bm{\theta}}^\ast}&\leq c_1\sum_{i=1}^N \nrm{{\bm{x}_t}(i)}_{{\bm{V}_{t}^{-1}}}
    \end{align}
    % where we used Lipschitz continuity of $p$ and $\mu$ in the second and third inequality, and used the Cauchy--Schwartz inequality in the last inequality.
    from \cref{lem:bound_bar_ast}.

    Combining \eqref{eq:ucb_first_term}, \eqref{eq:ucb_secound_term}, and \eqref{eq:ucb_third_term}, with probability at least $1-2\delta$, it holds that
    \begin{align*}
        \alpha f_t\prn{\pi_t^\ast;{\bm{\theta}}^\ast}-f_t\prn{\pi_t;{\bm{\theta}}^\ast}
        \leq 2c_1 \sum_{i=1}^{N}\nrm{{\bm{x}_t}(i)}_{{\bm{V}_{t}^{-1}}}
        .
    \end{align*}
    Since $\nrm{{\bm{x}_t}(i)}_2\leq 1$, we have
    \begin{align}
        \label{eq:bound_of_reg_tau}
        \sum_{t=1}^{T}\prn{\alpha f_t\prn{\pi_t^\ast;{\bm{\theta}}^\ast}-f_t\prn{\pi_t;{\bm{\theta}}^\ast}}
        &\leq 2c_1 \sum_{t=1}^T\sum_{i=1}^N \nrm{{\bm{x}_t}(i)}_{{\bm{V}_{t}^{-1}}} \nonumber\\
        &\leq 2c_1\prn*{\sqrt{2 d NT \log\prn*{1+\frac{N T}{d \lambda_0}}} + \frac{2dN}{\sqrt{\lambda_0}} \log\prn*{1+\frac{N T}{d \lambda_0}}},
    \end{align}
    where the last inequality follows from \eqref{eq:sum_nrm_xt_bound}.
    This is the desired upper bound.
\end{proof}

\subsection{Details of \cref{subsec:regret_ts}}
\label{app:details_ts}

% \begin{wrapfigure}[18]{r}{0.54\textwidth}
%     \vspace{-30pt}
    % \begin{minipage}{0.54\textwidth}
        \begin{algorithm}[H]
        \caption{CAB-TS}
        \label{alg:ts}
            {\begin{algorithmic}[1]
    \Require {The total rounds $T$, the number of users $N$, tuning parameter $\lambda_0$ and $a$, and access to an exact optimization oracle.}
    \State {$\mathcal{D}_1\gets\varnothing$, $\bm{V}_1\gets\lambda_0 {\bm{I}}$, and $\bm{H}_1\gets L_\mu\lambda_0 {\bm{I}}$.}
    \For {$t=1,\dots,T$}
        \State {${\bm{V}_{t}} \gets \lambda_0 {\bm{I}} + \sum_{s =1}^{t-1}\sum_{i=1}^N {\bm{x}_{s}} (i){\bm{x}_{s}} (i)^\top$.}
        \State {$\bar{\bm{\theta}}_t\gets\argmin_{\bm{\theta}\in\R^d}\tilde{\mathcal{L}} (\mathcal{D}_t;\bm{\theta},\kappa_\mu\lambda_0)$.}
        \State {If $t\geq2$, ${\bm{H}_{t}}\gets \sum_{s=1}^{t-1}\sum_{i=1}^N \dot{\mu}({\bm{x}_{s}} (i)^\top\bar{\bm{\theta}}_t)\prn{{\bm{x}_{s}} (i){\bm{x}_{s}} (i)^\top+ \frac{\lambda_0}{N(t-1)}{\bm{I}}}$.}
        \For {$i=1,\dots,N$}
            \State {$\tilde{\bm{\epsilon}}_t(i)\overset{\text{i.i.d.}}{\sim} \mathcal{N}(\bm{0},a^2 \bm{H}_t^{-1})$.}
        \EndFor
        \State {Call an exact optimization oracle for $f_t(\pi;\bar{\bm{\theta}}_t)+h_t(\pi;\tilde{\mathcal{E}}_t)$ and let $\pi_t$ denote its output
        % , where $f_t$ and $h_t$ are defined in \eqref{eq:definition_f_t} and \eqref{eq:definition_h}
        .
        }
        \State {Observe $y_t(i)$ for any $i\in[N]$.}
        \State {${\bm{x}_t}(i)\gets {\bm{\phi}}_t(i,\pi_t(i))$ for any $i\in[N]$ and $\mathcal{D}_{t+1}\gets \mathcal{D}_{t}\cup\set{\prn{\bm{x}_t(i),y_t(i)}}_{i\in[N]}$.}
    \EndFor
\end{algorithmic}}
        \end{algorithm}
%     \end{minipage}
%     % \vspace{-3pt}
% \end{wrapfigure}

For completeness, we provide the pseudo-code of CAB-TS (\cref{alg:ts}) and the full statement of \cref{thm:bound_TS}.
\begin{theorem}
    \label{thm:bound_TS_complete}
    \cref{alg:ts} with $a=c_1 \sqrt{L_\mu N}$ and $\lambda_0\geq \kappa_\mu^{-2}\sigma^2 \prn*{d\log\prn*{1+\frac{NT}{d}}+2\log\frac{1}{\delta}}$ achieves the following regret bound for any $\delta\in(0,1/T)$:
    \begin{align*}
        \E\brk{\mathcal{R}_T} &\leq \prn{c_1+c_2}\prn*{1+\frac{2}{0.15-2\delta}} \prn*{\sqrt{2 d NT \log\prn*{1+\frac{N T}{d \lambda_0}}} + \frac{2dN}{\sqrt{\lambda_0}} \log\prn*{1+\frac{N T}{d \lambda_0}}}\\
        &\qquad+ 4\delta KMT 
        % & = \tilde{O}\prn*{\kappa_\mu^{-3/2}L_rL_\mu^{3/2}N\prn{\sigma\sqrt{d}+\kappa_\mu\sqrt{\lambda_0}}\sqrt{dT}}=\tilde{O}\prn{dN\sqrt{T}}
        ,
    \end{align*}
    where $c_1 = \kappa_\mu^{-1} L_r L_\mu \prn*{\sigma\sqrt{d\log\prn*{1+\frac{NT}{d\lambda_0}}+2\log\frac{1}{\delta}} + \kappa_\mu D\sqrt{\lambda_0}}$ and $c_2= c_1\sqrt{2 \kappa_\mu^{-1} L_\mu N \log\frac{KN}{\delta}}$.
\end{theorem}

If we set $\lambda_0 =  \kappa_\mu^{-2}\sigma^2 \prn*{d\log\prn*{1+\frac{NT}{d}}+2\log\frac{1}{\delta}}$, we have
\[
    \mathcal{R}_T = \tilde{O}\prn{dN\sqrt{T} + dN^{3/2}}.
\]

To prove \cref{thm:bound_TS_complete}, we first bound the one-step regret.
For convenience, we define 
\begin{align}
    \label{eq:definition_E1t}
    E_{1,t}=\set{{}^\forall \pi\in\Pi \mid \abs{f_t(\pi;{\bm{\theta}}^\ast)-f_t(\pi;\bar{{\bm{\theta}}}_t)}\leq c_1\sum_{i=1}^N \nrm{{\bm{\phi}}_t(i,\pi(i))}_{{\bm{V}_{t}^{-1}}}}.
\end{align}

\begin{lemma}
    \label{lem:sub_bound_TS}
    Define events $E_{2,t}$ and $E_{3,t}$ as 
    \begin{align*}
        E_{2,t}&=\set*{{}^\forall \pi\in\Pi \mid \abs{h_t(\pi;\tilde{\mathcal{E}}_t)}\leq c_2 \sum_{i=1}^N\nrm{{\bm{\phi}}_t(i,\pi(i))}_{{\bm{V}_{t}^{-1}}}},\:\: \text{and}\\
        E_{3,t}&=\set*{h_t(\pi_t^\ast;\tilde{\mathcal{E}}_t)\geq c_1\sum_{i=1}^{N}\nrm{{\bm{\phi}}_t(i,\pi_t^\ast(i))}_{{\bm{V}_{t}^{-1}}}}.
    \end{align*}
    Let $\P_t\prn{E_{2,t}}\geq 1-p_2$ and $\P_t\prn{E_{3,t}}\geq p_3$ ($p_3>p_2$).
    If $E_{1,t}$ holds, then we have
    \[\E_t\brk{ f_t(\pi_t^\ast;{\bm{\theta}}^\ast)-f_t(\pi_t;{\bm{\theta}}^\ast)}\leq (c_1+c_2)\prn*{1+\frac{2}{p_3-p_2}}\E_t\brk*{\sum_{i=1}^N \nrm{{\bm{x}_t}(i)}_{{\bm{V}_{t}^{-1}}}} + p_2KM\]
\end{lemma}

\begin{proof}
    Let 
    $c=c_1+c_2$,
    $\S_t=\set{\pi\in\Pi \mid c\sum_{i=1}^N \nrm{{\bm{\phi}_t(i,\pi(i))}}_{{\bm{V}_{t}^{-1}}}< f_t(\pi_t^\ast;{\bm{\theta}}^\ast)-f_t(\pi;{\bm{\theta}}^\ast)}$,
    $\bar{\S}_t=\Pi/\S_t$,
    and $\pi_t^\prime=\argmin_{\pi\in\bar{\S}_t}\sum_{i=1}^N \nrm{{\bm{\phi}}_t(i,\pi(i))}_{{\bm{V}_{t}^{-1}}}$.

    First, we will bound $f_t(\pi_t^\ast;{\bm{\theta}}^\ast)-f_t(\pi_t;{\bm{\theta}}^\ast)$ on $E_{2,t}$.
    At round $t$ on $E_{2,t}$, we have 
    \begin{align*}
        f_t(\pi_t^\ast;{\bm{\theta}}^\ast) - f_t(\pi_t;{\bm{\theta}}^\ast)
        &\leq f_t(\pi_t^\ast;{\bm{\theta}}^\ast) - f_t(\pi_t^\prime;{\bm{\theta}}^\ast) + {f_t(\pi_t^\prime;\bar{{\bm{\theta}}}_t)+h_t(\pi_t^\prime;\tilde{\mathcal{E}}_t)} \\
        &\quad - (f_t(\pi_t;\bar{{\bm{\theta}}}_t)+ h_t(\pi_t;\tilde{\mathcal{E}}_t)) + h_t(\pi_t;\tilde{\mathcal{E}}_t) - h_t(\pi_t^\prime;\tilde{\mathcal{E}}_t) \\
        &\quad + c_1\sum_{i=1}^{N}\prn{\nrm{{\bm{\phi}}_t(i,\pi_t^\prime(i))}_{{\bm{V}_{t}^{-1}}} + \nrm{{\bm{\phi}}_t(i,\pi_t(i))}_{{\bm{V}_{t}^{-1}}}}\\
        &\leq (c_1+c_2)\sum_{i=1}^{N}\prn{2\nrm{{\bm{\phi}}_t(i,\pi_t^\prime(i))}_{{\bm{V}_{t}^{-1}}} + \nrm{{\bm{\phi}}_t(i,\pi_t(i))}_{{\bm{V}_{t}^{-1}}}},
    \end{align*}
    where the first inequality follows from the definition of $E_{1,t}$,
    and the second inequality follows from the definition of $E_{2,t}$, the optimality of $\pi_t$, and the definition of $\mathcal{S}_t$ and $\pi_t^\prime$.

    Second, we want to bound $\E_t\brk*{\sum_{i=1}^N\nrm{{\bm{\phi}}_t(i,\pi_t^\prime(i))}_{{\bm{V}_{t}^{-1}}}}$ by $\E_t\brk*{\sum_{i=1}^N\nrm{{\bm{\phi}}_t(i,\pi_t(i))}_{{\bm{V}_{t}^{-1}}}}$.
    Note that
    \begin{align*}
        \E_t\brk*{\sum_{i=1}^N\nrm{{\bm{\phi}}_t(i,\pi_t(i))}_{{\bm{V}_{t}^{-1}}}}
        &\geq \E_t\brk*{\sum_{i=1}^N\nrm{{\bm{\phi}}_t(i,\pi_t^\prime(i))}_{{\bm{V}_{t}^{-1}}}\mid\pi_t\in\bar{\S}_t}\P_t\prn{\pi_t\in\bar{\S}_t}\\
        &\geq \sum_{i=1}^N\nrm{{\bm{\phi}}_t(i,\pi_t^\prime(i))}_{{\bm{V}_{t}^{-1}}}\P_t\prn{\pi_t\in\bar{\S}_t}.
    \end{align*}
    Thus, $\E_t\brk{ f_t(\pi_t^\ast;{\bm{\theta}}^\ast)-f_t(\pi_t;{\bm{\theta}}^\ast)}\leq c \prn{1+2/\P_t\prn{\pi_t\in\bar{\S}_t}}\sum_{i=1}^N\nrm{{\bm{\phi}}_t(i,\pi_t(i))}_{{\bm{V}_{t}^{-1}}}$ on $E_{2,t}$.

    Third, we will lower-bound $\P_t\prn{\pi_t\in\bar{\S}_t}$.
    \begin{align*}
        \P_t(\pi_t\in\bar{\mathcal{S}}_t)&\geq \P_t\prn*{{}^\exists \pi \in\bar{\mathcal{S}}_t, f_t(\pi;\bar{{\bm{\theta}}}_t) + h_t(\pi;\tilde{\mathcal{E}}_t)> \max_{\pi^\prime\in\mathcal{S}_t}f_t(\pi^\prime;\bar{{\bm{\theta}}}_t) + h_t(\pi^\prime;\tilde{\mathcal{E}}_t)}\\
        &\geq \P_t\prn*{f_t(\pi_t^\ast;\bar{{\bm{\theta}}}_t) + h_t(\pi_t^\ast;\tilde{\mathcal{E}}_t)> \max_{\pi^\prime\in\mathcal{S}_t}f_t(\pi^\prime;\bar{{\bm{\theta}}}_t) + h_t(\pi^\prime;\tilde{\mathcal{E}}_t)}\\
        &\geq \P_t\prn*{f_t(\pi_t^\ast;\bar{{\bm{\theta}}}_t) + h_t(\pi_t^\ast;\tilde{\mathcal{E}}_t)> f_t(\pi_t^\ast;{\bm{\theta}}^\ast), E_{2,t}}\\
        &\geq \P_t\prn*{f_t(\pi_t^\ast;\bar{{\bm{\theta}}}_t) + h_t(\pi_t^\ast;\tilde{\mathcal{E}}_t)> f_t(\pi_t^\ast;\bar{{\bm{\theta}}}_t) + c_1\sum_{i=1}^{N}\nrm{{\bm{\phi}}_t(i,\pi_t^\ast(i))}_{{\bm{V}_{t}^{-1}}}} - \P\prn*{\bar{E}_{2,t}}\\
        &= \P_t\prn*{h_t(\pi_t^\ast;\tilde{\mathcal{E}}_t)\geq c_1\sum_{i=1}^{N}\nrm{{\bm{\phi}}_t(i,\pi_t^\ast(i))}_{{\bm{V}_{t}^{-1}}}} - \P\prn*{\bar{E}_{2,t}},
    \end{align*}
    where 
    % the first inequality follows from the definition of $\bar{\S}_t$,
    the second inequality follows from the fact that $\pi_t^\ast\in\bar{\S}_t$,
    the third inequality follows from 
    $f_t(\pi;\bar{{\bm{\theta}}}_t)+h_t(\pi;\tilde{\mathcal{E}}_t)\leq f_t(\pi;{\bm{\theta}}^\ast) + (c_1+c_2)\sum_{i=1}^N\nrm{{\bm{\phi}}_t(i,\pi(i))}_{{\bm{V}_{t}^{-1}}}\leq f_t(\pi_t^\ast;{\bm{\theta}}^\ast)$
    for any $\pi\in\S_t$ on $E_{1,t}$ and $E_{2,t}$,
    and the fourth inequality follows from the definition of $E_{1,t}$.

    Finally, we can achieve the desired bound, using the inequality, 
    \begin{align*}
    \E_t\brk{ f_t(\pi_t^\ast;{\bm{\theta}}^\ast)-f_t(\pi_t;{\bm{\theta}}^\ast)}&
    =\E_t[\prn{ f_t(\pi_t^\ast;{\bm{\theta}}^\ast)-f_t(\pi_t;{\bm{\theta}}^\ast)}\ind\brk{E_{2,t}}]\\
    &\quad+\E_t[\prn{ f_t(\pi_t^\ast;{\bm{\theta}}^\ast)-f_t(\pi_t;{\bm{\theta}}^\ast)}\ind\brk{\bar{E}_{2,t}}]\\
    &\leq\E_t[\prn{ f_t(\pi_t^\ast;{\bm{\theta}}^\ast)-f_t(\pi_t;{\bm{\theta}}^\ast)}\ind\brk{E_{2,t}}]+KM\P_t(\bar{E}_{2,t}). 
    \end{align*}
\end{proof}

\begin{remark}
    \label{rem:exact_oracle_ts}
    We explain why the analysis of CAB-TS assumes access to an exact optimization oracle.
    For brevity, write
    $F_t(\pi)=f_t(\pi;\bar{{\bm{\theta}}}_t)+h_t(\pi;\tilde{\mathcal{E}}_t)$.
    The proof of \cref{lem:sub_bound_TS} uses the bound
    \[
        \P_t(\pi_t\in\bar{\mathcal{S}}_t)
        \geq
        \P_t\prn*{{}^\exists \pi \in\bar{\mathcal{S}}_t
        ,
        F_t(\pi)> \max_{\pi^\prime\in\mathcal{S}_t}F_t(\pi^\prime)}.
    \]
    This step relies on exact optimality.
    Since $\pi_t^\ast\in\bar{\mathcal{S}}_t$, the set $\bar{\mathcal{S}}_t$ is nonempty.
    On the event on the right-hand side, an exact maximizer of $F_t$ cannot lie in $\mathcal{S}_t$; otherwise, the allocation in $\bar{\mathcal{S}}_t$ would have a strictly larger value, contradicting optimality.
    Hence the event implies $\pi_t\in\bar{\mathcal{S}}_t$.
    This implication is not preserved by an $\alpha$-approximate optimization oracle with $\alpha<1$.
    Exact maximization only requires the separation $\max_{\pi\in\bar{\mathcal{S}}_t}F_t(\pi)>\max_{\pi\in\mathcal{S}_t}F_t(\pi)$.
    However, an $\alpha$-approximate optimization oracle may still return an allocation in $\mathcal{S}_t$ whenever $\max_{\pi\in\mathcal{S}_t}F_t(\pi)\geq \alpha \max_{\pi\in\bar{\mathcal{S}}_t}F_t(\pi)$, because such an allocation satisfies the approximation guarantee.
    Thus, since the proof only ensures that a good allocation can attain the largest objective value, and does not provide the stronger margin $\max_{\pi\in\mathcal{S}_t}F_t(\pi)<\alpha \max_{\pi\in\bar{\mathcal{S}}_t}F_t(\pi)$, an $\alpha$-approximate optimization oracle with $\alpha<1$ may still return an allocation in $\mathcal{S}_t$, so the above probability lower bound is unavailable.
\end{remark}

$p_2$ and $p_3$ in \cref{lem:sub_bound_TS} can be bounded as follows, respectively.

\begin{lemma}
    \label{lem:bound_tilde_bar}
    For each $t\geq1$, $E_{2,t}$ holds with probability at least $1-2\delta$.
    % where
    % \[c_2= c_1\kappa_r^{-1}\kappa_\mu^{-3/2} L_r L_\mu^{3/2}\sqrt{2 N \log(KN/\delta)}.\]
\end{lemma}

\begin{lemma}
    \label{lem:bound_E3}
    $E_{3,t}$ holds with probability at least $0.15$.
    % $a = c_1 \kappa_r^{-1}\kappa_\mu^{-1} \sqrt{L_\mu N}$
\end{lemma}

We prove these lemmas.

\begin{proof}[Proof of \cref{lem:bound_tilde_bar}]
    From \citet[Lemma 4]{aistas_kveton20a_glm}, $E_{2,t}^\prime=\set{{}^\forall (i,a)\in [N]\times[K] \mid \abs{{\bm{\phi}}_t(i,a)^\top\tilde{\bm{\epsilon}}_t(i)}\leq c_2\nrm{{\bm{\phi}}_t(i,a)}_{{\bm{V}_{t}^{-1}}}}$ holds with probability at least $1-2\delta$.
    Here, we use the fact that $\bm{H}_t \succeq \kappa_\mu \bm{V}_t$ due to the definition of $\bm{H}_t$ and $\bm{V}_t$.
    Their proof uses the union bound, so we can apply that lemma to our algorithm, which samples $\tilde{\bm{\epsilon}}_t(i)$ independently for each $i\in[N]$.
    On $E_{2,t}^\prime$, we have $\sum_{i=1}^N \abs{{\bm{\phi}}_t(i,\pi(i))^\top\tilde{\bm{\epsilon}}_t(i)}\leq c_2\sum_{i=1}^N \nrm{{\bm{\phi}}_t(i,\pi(i))}_{{\bm{V}_{t}^{-1}}}$ for any $\pi\in\Pi$.
    In addition, we have $\abs{h_t(\pi;\tilde{\mathcal{E}}_t)}\leq \sum_{i=1}^N \abs{{\bm{\phi}}_t(i,\pi(i))^\top{\tilde{\bm{\epsilon}}_t(i)}}$ from the triangle inequality.
    Therefore, on $E_{2,t}^\prime$, it holds that $\abs{h_t(\pi;\tilde{\mathcal{E}}_t)}\leq c_2 \sum_{i=1}^N\nrm{{\bm{\phi}}_t(i,\pi(i))}_{{\bm{V}_{t}^{-1}}}$ for any $\pi\in\Pi$.
    In other words, $E_{2,t}$ holds with probability at least $1-2\delta$.
\end{proof}

\begin{proof}[Proof of \cref{lem:bound_E3}]
    Using ${\bm{H}_{t}} \preceq L_\mu {\bm{V}_{t}}$, and Cauchy--Schwarz inequality, we have
    \begin{align}
        \label{eq:probability_f_3t}
        &\P_t\prn*{h_t(\pi_t^\ast;\tilde{\mathcal{E}}_t) \geq c_1\sum_{i=1}^N\nrm{{\bm{\phi}}_t(i,\pi_t^\ast(i))}_{{\bm{V}_{t}^{-1}}}}\nonumber\\ 
        % &\geq \P_t\prn*{f_t(\pi_t^\ast;\tilde{{\bm{\theta}}}_t) - f_t(\pi_t^\ast;\bar{{\bm{\theta}}}_t) > c_1\sqrt{d\cond\prn{{\bm{V}_{t}}}}\nrm{\sum_{{\bm{\phi}}\in \pi_t^\ast}{\bm{\phi}}}_{{\bm{V}_{t}^{-1}}}}\\
        &\geq \P_t\prn*{h_t(\pi_t^\ast;\tilde{\mathcal{E}}_t) \geq c_1 \sqrt{L_\mu N \sum_{i=1}^N\nrm{{\bm{\phi}}_t(i,\pi_t^\ast(i))}_{{\bm{H}_{t}}^{-1}}^2 }} \nonumber\\
        & = \P_t\prn*{\sum_{i=1}^N {\bm{\phi}}_t(i,\pi_t^\ast(i))^\top\tilde{\bm{\epsilon}}_t(i)\geq a\sqrt{\sum_{i=1}^N\nrm{{\bm{\phi}}_t(i,\pi_t^\ast(i))}_{{\bm{H}_{t}}^{-1}}^2 }}
    \end{align}
    
    Next, we derive a lower bound.
    From $\tilde{\bm{\epsilon}}_t(i)\sim \mathcal{N}(\bm{0},a^2{\bm{H}_{t}}^{-1})$, we have ${\bm{\phi}}_t(i,\pi_t^\ast(i))^\top\tilde{\bm{\epsilon}}_t(i) \sim \mathcal{N}\prn{0, a^2\nrm{{\bm{\phi}}_t(i,\pi_t^\ast(i))}_{{\bm{H}_{t}}^{-1}}^2}$.
    Moreover, since $\set{\tilde{\bm{\epsilon}}_t(i)}_{i=1}^N$ are independent, it follows that $\sum_{i=1}^N {\bm{\phi}}_t(i,\pi_t^\ast(i))^\top\tilde{\bm{\epsilon}}_t(i)\sim \mathcal{N}\prn{0, a^2\sum_{i=1}^{N}\nrm{{\bm{\phi}}_t(i,\pi_t^\ast(i))}_{{\bm{H}_{t}}^{-1}}^2}$.
    Consequently, we obtain
    \begin{align}
        \label{eq:probability_linear_3t}
        \P_t\prn*{\sum_{i=1}^N {\bm{\phi}}_t(i,\pi_t^\ast(i))^\top\tilde{\bm{\epsilon}}_t(i)\geq a\sqrt{\sum_{i=1}^N\nrm{{\bm{\phi}}_t(i,\pi_t^\ast(i))}_{{\bm{H}_{t}}^{-1}}^2 }} \geq 0.15.
    \end{align}
    Combining \eqref{eq:probability_f_3t} with \eqref{eq:probability_linear_3t} yields the desired bound.   
\end{proof}
    
We are now ready to prove \cref{thm:bound_TS_complete}.
\begin{proof}[Proof of \cref{thm:bound_TS_complete}]
    Let $E_{4,t} = \set{\nrm{\bar{{\bm{\theta}}}_t-{\bm{\theta}}^\ast}_2\leq D+1}$, $\P\prn{E_{4,t}}\geq 1-p_4$, $p_1\geq \P(\bar{E}_{1,t}\land E_{4,t})$, $\P_t\prn{{E}_{2,t}}\geq 1-p_2$ on the event $E_{4,t}$, and $\P_t\prn{E_{3,t}}\geq p_3$.

    \begin{align*}
        \E\brk{\mathcal{R}_T}&= \sum_{t=1}^{T}\E\brk{ f_t(\pi_t^\ast;{\bm{\theta}}^\ast)-f_t(\pi_t;{\bm{\theta}}^\ast)}\\
        &\leq \sum_{t=1}^{T}\E\brk{\prn{ f_t(\pi_t^\ast;{\bm{\theta}}^\ast)-f_t(\pi_t;{\bm{\theta}}^\ast)}\ind\brk{E_{1,t},E_{4,t}}} + \prn{p_1+p_4} KMT\\
        &\leq \sum_{t=1}^{T}\E\brk{\E_t\brk{\prn{ f_t(\pi_t^\ast;{\bm{\theta}}^\ast)-f_t(\pi_t;{\bm{\theta}}^\ast)}\ind\brk{E_{1,t},E_{4,t}}}} + \prn{p_1+p_4} KMT
    \end{align*}

    From \cref{lem:sub_bound_TS}, it holds that
    \begin{align*}
        \E\brk{\mathcal{R}_T} &\leq \prn{c_1+c_2}\prn*{1+\frac{2}{p_3-p_2}} \E\brk*{\sum_{t=1}^{T}\sum_{i=1}^N \nrm{{\bm{x}_t}(i)}_{{\bm{V}_{t}^{-1}}}} + \prn{p_1+p_2+p_4}KMT\\
        &\leq \prn{c_1+c_2}\prn*{1+\frac{2}{p_3-p_2}} \E\brk*{\sqrt{NT\sum_{t=1}^{T}\sum_{i=1}^N \nrm{{\bm{x}_t}(i)}_{{\bm{V}_{t}^{-1}}}^2}} + \prn{p_1+p_2+p_4}KMT,
    \end{align*} 
    where we used the Cauchy--Schwarz inequality in the last inequality.

    Next, we bound $p_1$, $p_2$, $p_3$, and $p_4$.
    From \cref{lem:bound_bar_ast,lem:bound_tilde_bar,lem:bound_E3,lem:nrm_theta_bar_t}, we have $p_1\leq \delta$, $p_2\leq 2\delta$, $p_3\geq0.15$, and $p_4\leq\delta$.
    In addition, from the definition of $\lambda_0$, we have $p_4\leq \delta$ from \citet[Lemma 9]{aistas_kveton20a_glm}.
    Therefore, using these bounds and \eqref{eq:sum_nrm_xt_bound}, we can achieve the desired bound.
\end{proof}

\begin{remark}
    \label{rem:reason_for_iid}
    For each $i$, sampling i.i.d. from a Gaussian distribution is used to obtain a lower bound for \eqref{eq:probability_f_3t}.
    Indeed, $\sum_{i=1}^N {\bm{\phi}}_t(i,\pi_t^\ast(i))^\top\tilde{\bm{\epsilon}}_t(i)\sim\mathcal{N}\prn{0, a^2\sum_{i=1}^{N}\nrm{{\bm{\phi}}_t(i,\pi_t^\ast(i))}_{{\bm{H}_{t}}^{-1}}^2}$, which leads to the bound in \eqref{eq:probability_linear_3t}, is derived from the independence.
    On the other hand, if we sample a single $\bar{\bm{\epsilon}}_t$ from $\mathcal{N}\prn{0, a^2\sum_{i=1}^{N}\nrm{{\bm{\phi}}_t(i,\pi_t^\ast(i))}_{{\bm{H}_{t}}^{-1}}^2}$ and set $\tilde{\bm{\epsilon}}_t(i)=\bar{\bm{\epsilon}}_
    t$ for any $i\in[N]$, then $\sum_{i=1}^N {\bm{\phi}}_t(i,\pi_t^\ast(i))^\top\tilde{\bm{\epsilon}}_t(i)\sim\mathcal{N}\prn{0, a^2\nrm{\sum_{i=1}^{N}{\bm{\phi}}_t(i,\pi_t^\ast(i))}_{{\bm{H}_{t}}^{-1}}^2}$, and this prevents us from obtaining a desirable probability bound.
\end{remark}

\subsection{Variant of \cref{alg:ts}}
\label{app:variant_ts}
\begin{algorithm}[H]
\caption{CAB-TS (heuristic variant)}
\label{alg:ts2}
\begin{algorithmic}[1]
    \Require {The total rounds $T$, the number of users $N$, tuning parameter $\lambda_0$ and $a$, and access to a practical allocation routine.}
    \State {$\mathcal{D}_1=\varnothing$, $\bm{V}_1=\lambda_0 \bm{I}$, and $\bm{H}_1=L_\mu \lambda_0 \bm{I}$.}
    \For {$t=1,\dots,T$}
        \State {${\bm{V}_{t}} \gets \lambda_0 {\bm{I}} + \sum_{s =1}^{t-1}\sum_{i=1}^N {\bm{x}_{s}} (i){\bm{x}_{s}} (i)^\top$.}
        \State {$\bar{{\bm{\theta}}}_t\gets\argmin_{{\bm{\theta}}\in\R^d}\tilde{\mathcal{L}} (\mathcal{D}_t;{\bm{\theta}},\kappa_\mu\lambda_0)$.}
        \State {If $t\geq2$, ${\bm{H}_{t}}\gets \sum_{s=1}^{t-1}\sum_{i=1}^N \dot{\mu}({\bm{x}_{s}} (i)^\top\bar{\bm{\theta}}_t)\prn{{\bm{x}_{s}} (i){\bm{x}_{s}} (i)^\top+ \frac{\lambda_0}{N(t-1)}{\bm{I}}}$.}
        \For {$i=1,\dots,N$}
            \State {$\tilde{{\bm{\theta}}}_t(i)\overset{\text{i.i.d.}}{\sim} \mathcal{N}(\bar{{\bm{\theta}}}_t,a^2 {\bm{H}_{t}}^{-1})$.}
        \EndFor
        \State {Call the practical allocation routine for $\tilde{f}_t\prn{\pi;\tilde{\Theta}_t}$ and let $\pi_t$ denote its output.}
        \State {Observe $y_t(i)$ for any $i\in[N]$.}
        \State {Let ${\bm{x}_t}(i)\gets {\bm{\phi}}_t(i,\pi_t(i))$ for any $i\in[N]$ and $\mathcal{D}_{t+1}\gets \mathcal{D}_{t}\cup\set{\prn{{\bm{x}_t}(i),y_t(i)}}_{i\in[N]}$.}
    \EndFor
\end{algorithmic}
\end{algorithm}
In this section, we summarize a heuristic variant of CAB-TS related to \cref{alg:ts}. We retain it only for the experiments and do not include it in the theoretical contribution.

In \cref{alg:ts}, we sample $\tilde{\bm{\epsilon}}_t(i)$ from $\mathcal{N}(\bm{0},a^2 {\bm{H}_{t}}^{-1})$ for any $i\in[N]$, and maximize $f_t(\pi;\bar{{\bm{\theta}}}_t)+h_t(\pi;\tilde{\mathcal{E}}_t)$. 
However, it is also a natural idea to sample $\tilde{{\bm{\theta}}}_t(i)$ from $\mathcal{N}(\bar{{\bm{\theta}}}_t,a^2 {\bm{H}_{t}}^{-1})$ and instead maximize $\tilde{f}_t$, defined as follows:
\begin{align}
    \label{eq:definition_f_tilde}
    \tilde{f}_t(\pi;\vartheta)=\sum_{a\in[K]}r\prn*{\sum_{i\in\pi^{-1}(a)}\mu({\bm{\phi}}_t(i,a)^\top\vartheta(i))}.
\end{align}
The algorithm can be written as \cref{alg:ts2}.
% We use it in the experiments only as a heuristic variant.

\begin{remark}
    \label{rem:worse_point}
    The heuristic variant \cref{alg:ts2} is retained only for empirical comparison and is not supported by a theoretical guarantee in this paper.
    To extend the proof of \cref{thm:bound_TS_complete} by a similar argument, we would need a positive lower bound on the probability of the event
    \[
        E_{3,t}^\star
        =
        \set*{
            \tilde{f}_t(\pi_t^\ast;\tilde{\Theta}_t) - f_t(\pi_t^\ast;\bar{{\bm{\theta}}}_t)
            \geq
            c_1^\star \sum_{i\in[N]}\nrm{{\bm{\phi}}_t(i,\pi_t^\ast(i))}_{{\bm{V}_{t}^{-1}}}
        }.
    \]
    However, for the direct nonlinear objective, the concavity of $r$ alone does not yield a useful lower bound on
    $\tilde{f}_t(\pi_t^\ast;\tilde{\Theta}_t) - f_t(\pi_t^\ast;\bar{{\bm{\theta}}}_t)$,
    and thus obtaining such a probabilistic lower bound is difficult.
    Even if one imposes a uniform lower-bound assumption on $\dot{r}$, analogous to the assumption on $\dot{\mu}$, this issue remains unresolved.
\end{remark}

\section{Omitted details of \cref{subsec:one_pass}}
\label{app:onepass}
In this section, we provide the omitted details of \cref{subsec:one_pass}.
For completeness, we first restate CAB-OFU with one-pass OMD update and the quantities used in the analysis, since the main text only presents a compressed description.

\subsection{Complete description of CAB-OFU with one-pass OMD update}
\label{subsec:onepass_complete_description}
We consider the following discussions under \cref{asp:self_concordance}.
The one-pass variant is summarized in \cref{alg:onepass_ofu}.

\begin{algorithm}[H]
\caption{CAB-OFU with one-pass OMD update}
\label{alg:onepass_ofu}
{\begin{algorithmic}[1]
\Require {The total rounds $T$, the number of users $N$, and tuning parameters $\lambda_{\mathrm{op}}$, $\eta$, $\alpha$, and $\delta$.}
\State {Initialize $\bm{\theta}_1\in\Theta$ and $\bm{Q}_1\gets \lambda_{\mathrm{op}}\bm{I}$.}
\For {$t=1,\dots,T$}
    \State {Call an $\alpha$-approximate OFU oracle satisfying \eqref{eq:onepass_ofu_rule} and obtain $\pi_t$.}
    \State {Observe $y_t(i)$ for all $i\in[N]$.}
    \State {Set $\bm{x}_t(i)\gets\bm{\phi}_t(i,\pi_t(i))$ for all $i\in[N]$.}
    \State {Update $\bm{\theta}_{t+1}$ by \eqref{eq:re_onepass_theta_update}.}
\EndFor
\end{algorithmic}}
\end{algorithm}

For each round $t$ and user $i$, define the negative log-likelihood
\[
    \ell_{t,i}(\bm{\theta})
    =
    - y_t(i)\bm{x}_t(i)^\top \bm{\theta}
    + m\prn{\bm{x}_t(i)^\top \bm{\theta}},
\]
where $\dot{m}=\mu$.
The one-pass OMD update uses the quadratic surrogate
\[
    \tilde{\ell}_t(\bm{\theta})
    =
    \sum_{i=1}^N
    \prn{
        \inpr{\nabla_{\bm{\theta}} \ell_{t,i}(\bm{\theta}_t), \bm{\theta}-\bm{\theta}_t}
        +
        \frac{1}{2}\nrm{\bm{\theta}-\bm{\theta}_t}_{\nabla_{\bm{\theta}}^2 \ell_{t,i}(\bm{\theta}_t)}^2
    }.
\]
The parameter is then updated by
\begin{equation}
    \label{eq:re_onepass_theta_update}
    \bm{\theta}_{t+1}
    =
    \argmin_{\bm{\theta}\in\Theta}
    \prn*{
        \tilde{\ell}_t(\bm{\theta})
        +
        \frac{1}{2\eta}\nrm{\bm{\theta}-\bm{\theta}_t}_{\bm{Q}_t}^2
    },
\end{equation}
and the matrix sequence is
\[
    \bm{Q}_{t+1}
    =
    \lambda_{\mathrm{op}} \bm{I}
    +
    \sum_{s=1}^{t}\sum_{i=1}^N
    \nabla_{\bm{\theta}}^2 \ell_{s,i}(\bm{\theta}_{s+1})
    =
    \lambda_{\mathrm{op}} \bm{I}
    +
    \sum_{s=1}^{t}\sum_{i=1}^N
    \dot{\mu}\prn{\bm{x}_s(i)^\top\bm{\theta}_{s+1}}
    \bm{x}_s(i)\bm{x}_s(i)^\top.
\]

We use the confidence set
\[
    C_t(\delta)
    =
    \set*{
        \bm{\theta}\in\Theta
        \,\middle|\,
        \nrm{\bm{\theta}-\bm{\theta}_t}_{\bm{Q}_t}
        \leq
        \beta_t(\delta)
    },
\]
where
\[
    \beta_t(\delta)
    =
    \sqrt{
        4\lambda_{\mathrm{op}} D^2
        +
        2\eta \log\prn*{1/\delta}
        +
        d\prn*{6\eta^2+\eta}
        \log\prn{1+ L_\mu Nt/\lambda_{\mathrm{op}}}
    }.
\]
Given the optimistic value
\[
    \mathrm{OPT}^{\mathrm{op}}_t
    =
    \max_{\pi\in\Pi}\max_{\bm{\theta}\in C_t(\delta)} f_t(\pi;\bm{\theta}),
\]
the allocation is chosen by an $\alpha$-approximate OFU rule:
\begin{equation}
    \label{eq:onepass_ofu_rule}
    \max_{\bm{\theta}\in C_t(\delta)} f_t(\pi_t;\bm{\theta})
    \geq
    \alpha \mathrm{OPT}^{\mathrm{op}}_t.
\end{equation}

Next, we compare the cost of updating the parameter-update step.
If the regularized MLE at round $t$ is solved by an iterative method with $I_t$ optimization iterations, and one pass over the first $t-1$ rounds costs $O(t)$ when suppressing the dependence on $d$ and $N$, then the MLE-based update costs $O(t I_t)$.
In contrast, once $\bm{Q}_t$ is maintained incrementally, the one-pass OMD update \eqref{eq:re_onepass_theta_update} uses only the current-round surrogate and $\bm{Q}_t$.
Hence, if the quadratic subproblem is solved in $\tilde{I}_t$ optimization iterations, its update cost is $O(\tilde{I}_t)$.
The optimistic allocation step is separate and, in the practical implementation, may require multiple calls to the submodular welfare oracle.

\subsection{Proof of \cref{thm:onepass_regret}}
% The technical lemmas used in the proof are collected in \cref{subsec:onepass_lemmas}.
We provide the complete version of \cref{thm:onepass_regret} and its proof.
\begin{theorem}[complete version of \cref{thm:onepass_regret}]
    \label{thm:onepass_regret_complete}
    Under \cref{asp:self_concordance}, with probability at least $1-\delta$, the $\alpha$-approximate regret of \cref{alg:onepass_ofu} with $\eta=1+RD$ and $\lambda_{\mathrm{op}}\geq\max\set{14d\eta R^2, 6\eta RDL_\mu N }$ is upper bounded as
    \begin{align*}
        \mathcal{R}_T^\alpha&\leq 2\kappa_\mu^{-1/2} L_r L_\mu \beta_T(\delta) \prn*{\sqrt{2 d NT \log\prn*{1+\frac{N T}{d \lambda_{\mathrm{op}}}}} + \frac{2dN}{\sqrt{\lambda_{\mathrm{op}}}} \log\prn*{1+\frac{N T}{d \lambda_{\mathrm{op}}}}} \nonumber\\
        &\qquad + 16 \kappa_\mu^{-1} RL_\mu L_r d N \beta_T^2(\delta)\log\prn*{1+\frac{N T}{d \lambda_{\mathrm{op}}}}
    \end{align*}
\end{theorem}

\begin{proof}
    By decreasing $\kappa_\mu$ if necessary, we may assume $\kappa_\mu\leq 1$.
    Let $(\pi_t,\tilde{\bm{\theta}}^{\mathrm{op}}_t)$ be an output of the $\alpha$-approximate optimization oracle for $\max_{\pi\in\Pi}\max_{\bm{\theta}\in C_t(\delta)} f_t(\pi;\bm{\theta}).$
    We work on the high-probability event in \cref{lem:onepass_conf_set}, so that
    $\nrm{\tilde{\bm{\theta}}^{\mathrm{op}}_t-\bm{\theta}^\ast}_{\bm{Q}_t}\leq 2\beta_t(\delta)$
    for all $t\in[T]$.
    Let
    \[
        \bm{V}_t
        =
        \lambda_{\mathrm{op}} \bm{I}
        +
        \sum_{s=1}^{t-1}\sum_{i=1}^{N}
        \bm{x}_s(i)\bm{x}_s(i)^\top.
    \]
    By the lower bound encoded in $\kappa_\mu$, we have
    \[
        \bm{Q}_t
        =
        \lambda_{\mathrm{op}} \bm{I}
        +
        \sum_{s=1}^{t-1}\sum_{i=1}^{N}
        \dot{\mu}\prn*{\bm{x}_s(i)^\top\bm{\theta}_{s+1}}
        \bm{x}_s(i)\bm{x}_s(i)^\top
        \succeq
        \kappa_\mu \bm{V}_t.
    \]
    We first decompose the instantaneous regret and then bound the linear and quadratic terms from the Taylor expansion separately.
    We have
    \begin{align}\label{eq:onepass_regret_taylor}
        \mathcal{R}_T^\alpha&= \sum_{t=1}^T \prn*{\alpha f_t(\pi_t^\ast;\bm{\theta}^\ast) - f_t(\pi_t;\bm{\theta}^\ast)}\nonumber\\
        &\leq \sum_{t=1}^{T} \prn*{f_t(\pi_t;\tilde{\bm{\theta}}^{\mathrm{op}}_t) - f_t(\pi_t;\bm{\theta}^\ast)}
        \leq L_r\sum_{t=1}^{T} \sum_{i=1}^{N} \abs*{\mu(\bm{x}_t(i)^\top \tilde{\bm{\theta}}^{\mathrm{op}}_t) - \mu(\bm{x}_t(i)^\top \bm{\theta}^\ast)}\nonumber\\
        &\leq
        L_r \sum_{t=1}^{T} \sum_{i=1}^{N}
        \Biggl|
            \dot{\mu}(\bm{x}_t(i)^\top {\bm{\theta}}^\ast)\bm{x}_t(i)^\top\prn{\tilde{\bm{\theta}}^{\mathrm{op}}_t-\bm{\theta}^\ast}
            \nonumber\\
            &\qquad
            +
            \prn*{\bm{x}_t(i)^\top\prn*{\tilde{\bm{\theta}}^{\mathrm{op}}_t-\bm{\theta}^\ast}}^2
            \int_0^1 (1-v)
            \ddot{\mu}\prn[\Big]{
                \bm{x}_t(i)^\top \bm{\theta}^\ast
                + v \bm{x}_t(i)^\top \prn*{\tilde{\bm{\theta}}^{\mathrm{op}}_t-\bm{\theta}^\ast}
            } \, dv
        \Biggr| \nonumber\\
        &\leq
        L_r \sum_{t=1}^{T} \sum_{i=1}^{N}
        \Biggl(
            \dot{\mu}(\bm{x}_t(i)^\top {\bm{\theta}}^\ast)
            \abs*{\bm{x}_t(i)^\top\prn{\tilde{\bm{\theta}}^{\mathrm{op}}_t-\bm{\theta}^\ast}}
            % \nonumber\\
            % &\qquad\qquad\qquad\qquad
            +
            RL_\mu\abs*{
                \prn*{\bm{x}_t(i)^\top\prn*{\tilde{\bm{\theta}}^{\mathrm{op}}_t-\bm{\theta}^\ast}}^2
            }
        \Biggr)
        % &\leq 2L_r L_\mu \beta_T(\delta)\sum_{t=1}^{T} \sum_{i=1}^{N}\nrm{\bm{x}_t(i)}_{\bm{H}_t^{-1}} +  2 RL_\mu \beta_T^2(\delta) \sum_{t=1}^{T} \sum_{i=1}^{N} \nrm{\bm{x}_t(i)}_{\bm{H}_t^{-1}}^2 \nonumber\\
        % &\leq 2L_r L_\mu \beta_T(\delta)\sum_{t=1}^{T} \sum_{i=1}^{N}\nrm{\bm{x}_t(i)}_{\bm{V}_t^{-1}} +  2 RL_r L_\mu \beta_T^2(\delta) \sum_{t=1}^{T} \sum_{i=1}^{N} \nrm{\bm{x}_t(i)}_{\bm{V}_t^{-1}}^2
        .
    \end{align}
    Here the first inequality follows from the $\alpha$-approximate OFU rule, since $\bm{\theta}^\ast\in C_t(\delta)$ implies
    \[
        f_t(\pi_t;\tilde{\bm{\theta}}^{\mathrm{op}}_t)
        \geq
        \alpha \max_{\pi\in\Pi}\max_{\bm{\theta}\in C_t(\delta)} f_t(\pi;\bm{\theta})
        \geq
        \alpha f_t(\pi_t^\ast;\bm{\theta}^\ast),
    \]
    the second inequality follows from the Lipschitz continuity of $r$, the third inequality follows from the second-order Taylor expansion of $\mu$ around $\bm{\theta}^\ast$ with the integral remainder form, and the fourth inequality follows from \cref{asp:self_concordance}.

    We next control the first-order term in \eqref{eq:onepass_regret_taylor}. 
    By Cauchy--Schwarz and \cref{lem:onepass_conf_set},
% \begin{comment}
    \begin{align}
        \label{eq:onepass_linear_bound}
        &\sum_{t=1}^{T} \sum_{i=1}^{N} \dot{\mu}(\bm{x}_t(i)^\top {\bm{\theta}}^\ast)\abs*{\bm{x}_t(i)^\top\prn{\tilde{\bm{\theta}}^{\mathrm{op}}_t-\bm{\theta}^\ast}}\nonumber\\
        &\leq \sum_{t=1}^{T} \sum_{i=1}^{N} \dot{\mu}(\bm{x}_t(i)^\top {\bm{\theta}}^\ast) \nrm{\bm{x}_t(i)}_{\bm{Q}_t^{-1}} \nrm{\tilde{\bm{\theta}}^{\mathrm{op}}_t-\bm{\theta}^\ast}_{\bm{Q}_t}
        \leq 2 \beta_T(\delta) \sum_{t=1}^{T} \sum_{i=1}^{N}\dot{\mu}(\bm{x}_t(i)^\top {\bm{\theta}}^\ast)  \nrm{\bm{x}_t(i)}_{\bm{Q}_t^{-1}}
        % \nonumber\\
        % &\leq 2 \beta_T(\delta) \sum_{t\in\mathcal{T}_1} \sum_{i=1}^{N}\dot{\mu}(\bm{x}_t(i)^\top {\bm{\theta}}^\ast)  \nrm{\bm{x}_t(i)}_{\bm{H}_t^{-1}} + 2 \beta_T(\delta) \sum_{t\in \mathcal{T}_2} \sum_{i=1}^{N}\dot{\mu}(\bm{x}_t(i)^\top \tilde{\bm{\theta}}_t)  \nrm{\bm{x}_t(i)}_{\bm{H}_t^{-1}}
        .
    \end{align}
    % where we split the time indices according to whether the derivative term is larger at $\bm{\theta}^\ast$ or at the current estimate:
    % \[
    %     \mathcal{T}_1
    %     =
    %     \set*{
    %         t\in[T] \,|dle|\,
    %         \sum_{i=1}^{N} \dot{\mu}(\bm{x}_t(i)^\top {\bm{\theta}}^\ast)
    %         \geq
    %         \sum_{i=1}^{N} \dot{\mu}(\bm{x}_t(i)^\top {\bm{\theta}}_{t+1})
    %     },\text{ and}
    %     \quad
    %     \mathcal{T}_2 = [T] \setminus \mathcal{T}_1.
    % \]
    Using $\dot{\mu}(\bm{x}_t(i)^\top {\bm{\theta}}^\ast)\leq L_\mu$ and $\bm{Q}_t\succeq \kappa_\mu \bm{V}_t$, we have
    \begin{align}
        \label{eq:onepass_T1_bound}
        &\sum_{t=1}^{T} \sum_{i=1}^{N} \dot{\mu}(\bm{x}_t(i)^\top {\bm{\theta}}^\ast)\abs{\bm{x}_t(i)^\top\prn{\tilde{\bm{\theta}}^{\mathrm{op}}_t-\bm{\theta}^\ast}}\nonumber\\
        &\leq 2L_\mu \beta_T(\delta)\sum_{t=1}^{T} \sum_{i=1}^{N}\nrm{\bm{x}_t(i)}_{\bm{Q}_t^{-1}}\nonumber\\
        &\leq 2\kappa_\mu^{-1/2}L_\mu \beta_T(\delta)\sum_{t=1}^{T} \sum_{i=1}^{N}\nrm{\bm{x}_t(i)}_{\bm{V}_t^{-1}}\nonumber\\
        &\leq 2\kappa_\mu^{-1/2}L_\mu \beta_T(\delta)\prn*{\sqrt{2 d NT \log\prn*{1+\frac{N T}{d \lambda_{\mathrm{op}}}}} + \frac{2dN}{\sqrt{\lambda_{\mathrm{op}}}} \log\prn*{1+\frac{N T}{d \lambda_{\mathrm{op}}}}}
        ,
    \end{align}
    where the last inequality follows from \eqref{eq:sum_nrm_xt_bound}.
    By applying \citet[Lemma 5]{takemura2021near_combinatorial}, we obtain
    \begin{equation}
        \label{eq:sum_nrm_xt_bound_square}
        \sum_{t=1}^{T} \sum_{i=1}^{N} \nrm*{\bm{x}_t(i)}_{\bm{V}_t^{-1}}^2\leq 4 d N \log\prn*{1+\frac{N T}{d \lambda_{\mathrm{op}}}}.
    \end{equation}

    % The same argument applies to the rounds in $\mathcal{T}_2$ after swapping the roles of $\bm{\theta}^\ast$ and $\tilde{\bm{\theta}}_t$:
    % \begin{align}
    %     \label{eq:onepass_T2_bound}
    %     \sum_{t\in \mathcal{T}_2} \sum_{i=1}^{N}\dot{\mu}(\bm{x}_t(i)^\top \tilde{\bm{\theta}}_t) \nrm{\bm{x}_t(i)}_{\bm{H}_t^{-1}}
    %     \leq \sum_{t\in \mathcal{T}_2}\sum_{i=1}^{N}\sqrt{\dot{\mu}(\bm{x}_t(i)^\top {\bm{\theta}}^\ast)}\nrm*{\sqrt{\dot{\mu}(\bm{x}_t(i)^\top \tilde{\bm{\theta}}_t)}\bm{x}_t(i)}_{\bm{H}_t^{-1}},
    % \end{align}
    % where the inequality follows from the definition of $\mathcal{T}_2$.
    % This term is bounded in the same way as \eqref{eq:onepass_T1_bound}, so it contributes at the same order.
    % Thus, we can upper bound \eqref{eq:onepass_linear_bound} as
    % % \[
    %     \begin{align}
    %         \label{eq:onepass_first_term_linear_bound}
    %         &\sum_{t=1}^{T} \sum_{i=1}^{N} \dot{\mu}(\bm{x}_t(i)^\top {\bm{\theta}}^\ast)\abs*{\bm{x}_t(i)^\top\prn{\tilde{\bm{\theta}}_t-\bm{\theta}^\ast}}\\
    %         &\leq 4L_r \beta_T(\delta)\sqrt{\kappa_\mu^\ast NT + L_\mu D NT} \sqrt{\text{log term}} + 2RL_\mu L_r\beta_T^2(\delta)\text{log term}.
    %     \end{align}
    % % \]

    It remains to control the second-order remainder term in \eqref{eq:onepass_regret_taylor}. By the self-concordance-type bound and \cref{lem:onepass_conf_set},
    \begin{align}
        \label{eq:onepass_remainder_bound}
        RL_\mu\sum_{t=1}^{T} \sum_{i=1}^{N} \abs*{
                \prn*{\bm{x}_t(i)^\top\prn*{\tilde{\bm{\theta}}^{\mathrm{op}}_t-\bm{\theta}^\ast}}^2
            }
        &\leq 4 RL_\mu \beta_T^2(\delta) \sum_{t=1}^{T} \sum_{i=1}^{N} \nrm{\bm{x}_t(i)}_{\bm{Q}_t^{-1}}^2\nonumber\\
        &\leq 4\kappa_\mu^{-1} RL_\mu \beta_T^2(\delta) \sum_{t=1}^{T} \sum_{i=1}^{N} \nrm{\bm{x}_t(i)}_{\bm{V}_t^{-1}}^2\nonumber\\
        &\leq 16 \kappa_\mu^{-1} RL_\mu dN \beta_T^2(\delta)  \log\prn*{1+\frac{N T}{d \lambda_{\mathrm{op}}}},
    \end{align}
    where the last inequality follows from \eqref{eq:sum_nrm_xt_bound_square}.
% \end{comment}

    Combining \eqref{eq:onepass_T1_bound} and \eqref{eq:onepass_remainder_bound}, we obtain the following regret bound:
    \begin{align}
        \label{eq:onepass_final_regret_bound}
        \mathcal{R}_T^\alpha&\leq 2\kappa_\mu^{-1/2} L_r L_\mu \beta_T(\delta) \prn*{\sqrt{2 d NT \log\prn*{1+\frac{N T}{d \lambda_{\mathrm{op}}}}} + \frac{2dN}{\sqrt{\lambda_{\mathrm{op}}}} \log\prn*{1+\frac{N T}{d \lambda_{\mathrm{op}}}}} \nonumber\\*
        &\qquad + 16 \kappa_\mu^{-1} RL_\mu L_r d N \beta_T^2(\delta)\log\prn*{1+\frac{N T}{d \lambda_{\mathrm{op}}}}
    \end{align}
\end{proof}

\subsection{Useful lemmas for the proof of \cref{thm:onepass_regret_complete}}
\label{subsec:onepass_lemmas}
In this section, we present the lemmas to prove \cref{thm:onepass_regret_complete}.

First, we prove that the confidence set defined in \cref{lem:onepass_conf_set} contains the true parameter $\bm{\theta}^\ast$ with high probability.
\begin{lemma}
    \label{lem:onepass_conf_set}
    Let $\delta \in (0,1)$ and
    \[
        C_t(\delta)=\set*{\bm{\theta}\in\Theta \mid \nrm{\bm{\theta}-\bm{\theta}_t}_{\bm{Q}_t} \leq \beta_t({\delta})},
    \]
    where
    \begin{equation*}
        \beta_t({\delta}) = \sqrt{4\lambda_{\mathrm{op}} D^2 +2\eta\log\prn*{\frac{1}{\delta}}+d\prn{6\eta^2+\eta}\log\prn*{1+\frac{L_\mu Nt}{\lambda_{\mathrm{op}}}}}.
    \end{equation*}
    Set $\eta=1+RD$ and $\lambda_{\mathrm{op}}\geq\max\set{14d\eta R^2, 6\eta RDL_\mu N}$ for \cref{alg:onepass_ofu}.
    Then, with probability at least $1-\delta$, we have $\bm{\theta}^\ast\in C_t(\delta)$, equivalently $\nrm{\bm{\theta}_t-\bm{\theta}^\ast}_{\bm{Q}_t}\leq \beta_t(\delta)$, for any $t\in[T]$.
\end{lemma}
\begin{proof}
    By \cref{lem:onepass_linear_sum_bound}, we have
    \begin{align}
        \label{eq:theta_nrm_initial_bound}
        \nrm{\bm{\theta}_{t+1}-\bm{\theta}^\ast}_{\bm{Q}_{t+1}}^2
        &\leq
        2\eta
        \sum_{s=1}^t \sum_{i=1}^{N}
        \prn[\big]{
            \ell_{s,i}(\bm{\theta}^\ast) - \ell_{s,i}(\bm{\theta}_{s+1})
        }
        + 4\lambda_{\mathrm{op}} D^2
        \nonumber\\*
        &\quad
        + 2\eta RDL_\mu N \sum_{s=1}^t \nrm{\bm{\theta}_{s+1}-\bm{\theta}_s}_{2}^2
        - \sum_{s=1}^t \nrm{\bm{\theta}_{s}-\bm{\theta}_{s+1}}_{\bm{Q}_s}^2
        .
    \end{align}
    First, to upper bound the linear term, we decompose as
    \begin{equation*}
        \sum_{s=1}^t \sum_{i=1}^{N}
            \prn{
                \ell_{s,i}(\bm{\theta}^\ast) - \ell_{s,i}(\bm{\theta}_{s+1})
            }
            =\sum_{s=1}^t  \prn*{\sum_{i=1}^{N}
                \ell_{s,i}(\bm{\theta}^\ast) - m_s(P_s)
            }
            +
            \sum_{s=1}^t  \prn*{m_s(P_s) - \sum_{i=1}^{N}\ell_{s,i}(\bm{\theta}_{s+1})}
            ,
    \end{equation*}
    where $P_s=\mathcal{N}\prn{\bm{\theta}_s,\zeta \bm{Q}_s^{-1} }$ with $\zeta=3\eta/2$ is a $d$-dimensional multivariate normal distribution and the function $m_s\colon P\mapsto \mathbb{R}$ is defined as $m_s(P_s) = -\log\prn{\E_{\bm{\theta}\sim P_s}\brk{\exp(-\sum_{i=1}^N\ell_{s,i}(\bm{\theta}))}}$.

    The first term can be bounded by applying \cref{lem:onepass_high_prob_bound} to the aggregated loss $\sum_{i=1}^N \ell_{s,i}(\bm{\theta}^\ast)$.
    Thus, it holds that with probability at least $1-\delta$, 
    \[
        \sum_{s=1}^t  \prn*{
            \sum_{i=1}^{N}
                \ell_{s,i}(\bm{\theta}^\ast) - m_s(P_s)
        }\leq \log \prn*{\frac{1}{\delta}}.
    \]
    In addition, by \citet[Lemma 6]{zhang2025generalizedlinearbanditsoptimal}, we can bound the second term as
    \[
        \sum_{s=1}^t  \prn*{
            m_s(P_s) - \sum_{i=1}^{N}\ell_{s,i}(\bm{\theta}_{s+1})
        }
        \leq
        \frac{1}{3\eta}\sum_{s=1}^t \nrm{\bm{\theta}_{s+1}-\bm{\theta}_s}_{\bm{Q}_s}^2+d\prn*{3\eta+\frac{1}{2}}\log\prn*{1+\frac{L_\mu N t}{\lambda_{\mathrm{op}}}}.
    \]
    Although there is a difference between \citet[Lemma 6]{zhang2025generalizedlinearbanditsoptimal} and our setting in that the latter is combinatorial, we can derive the above upper bound by treating $\sum_{i=1}^N\ell_{s,i}(\bm{\theta}_{s+1})$ as a single unit.
    Substituting these bounds into the \eqref{eq:theta_nrm_initial_bound}, it holds that
    \begin{align*}
        \nrm{\bm{\theta}_{t+1}-\bm{\theta}^\ast}_{\bm{Q}_{t+1}}^2
        &\leq
        4\lambda_{\mathrm{op}} D^2 + 2\eta \log \prn*{\frac{1}{\delta}} + d\prn*{6\eta^2+\eta}\log\prn*{1+\frac{L_\mu N t}{\lambda_{\mathrm{op}}}}\nonumber\\
        &\qquad + 2\eta RDL_\mu N \sum_{s=1}^t \nrm{\bm{\theta}_{s+1}-\bm{\theta}_s}_{2}^2
        - \frac{1}{3\eta}\sum_{s=1}^t \nrm{\bm{\theta}_{s}-\bm{\theta}_{s+1}}_{\bm{Q}_s}^2\nonumber\\
        &\leq
        4\lambda_{\mathrm{op}} D^2 + 2\eta \log \prn*{\frac{1}{\delta}} + d\prn*{6\eta^2+\eta}\log\prn*{1+\frac{L_\mu N t}{\lambda_{\mathrm{op}}}}\nonumber\\
        &\qquad + 2\eta RDL_\mu N \sum_{s=1}^t \nrm{\bm{\theta}_{s+1}-\bm{\theta}_s}_{2}^2
        - \frac{1}{3}\sum_{s=1}^t \nrm{\bm{\theta}_{s}-\bm{\theta}_{s+1}}_{\bm{Q}_s}^2\nonumber\\
        &\leq
        4\lambda_{\mathrm{op}} D^2 + 2\eta \log \prn*{\frac{1}{\delta}} + d\prn*{6\eta^2+\eta}\log\prn*{1+\frac{L_\mu N t}{\lambda_{\mathrm{op}}}}
        ,
    \end{align*}
    where the last inequality follows from $\lambda_{\mathrm{op}}\geq 6\eta RDL_\mu N $.
\end{proof}
% In this subsection, we use the same notation as in the main text.
% In particular, for each round $t$ and user $i\in[N]$, let
% \[
%     \bm{x}_t(i)=\bm{\phi}_t(i,\pi_t(i))
% \]
% be the selected context.
% We write the per-user negative log-likelihood contribution as $\ell_{t,i}(\bm{\theta})$, so that
% \[
%     \nabla \ell_{t,i}(\bm{\theta})
%     = \prn[\big]{\mu\prn*{\bm{x}_t(i)^\top \bm{\theta}}-y_t(i)}\bm{x}_t(i),
%     \qquad
%     \nabla^2 \ell_{t,i}(\bm{\theta})
%     = \dot{\mu}\prn*{\bm{x}_t(i)^\top \bm{\theta}} \bm{x}_t(i)\bm{x}_t(i)^\top .
% \]
% For the one-pass update, define
% \[
%     \tilde{\ell}_t(\bm{\theta})
%     =
%     \sum_{i=1}^N
%     \prn*{
%         \inpr{\nabla \ell_{t,i}(\bm{\theta}_t), \bm{\theta}-\bm{\theta}_t}
%         + \frac{1}{2}\nrm{\bm{\theta}-\bm{\theta}_t}^2_{\nabla^2 \ell_{t,i}(\bm{\theta}_t)}
%     }
% \]
% and
% \[
%     \bm{\theta}_{t+1}
%     =
%     \argmin_{\bm{\theta}\in\Theta}
%     \prn*{
%         \tilde{\ell}_t(\bm{\theta})
%         + \frac{1}{2\eta}\nrm{\bm{\theta}-\bm{\theta}_t}_{\bm{H}_t}^2
%     },
%     \qquad
%     \bm{H}_t
%     =
%     \lambda \bm{I}
%     + \sum_{s=1}^{t-1}\sum_{i=1}^N \nabla^2 \ell_{s,i}(\bm{\theta}_{s+1}).
% \]

The following lemma is regarding the property of the online mirror descent update.
\begin{lemma}[{\citealt[Lemma 1]{zhang2025generalizedlinearbanditsoptimal}}]
    \label{lem:onepass_proximal}
    Let $f\colon \Theta \to \mathbb{R}$ be a convex function, let $\Theta$ be a convex set, and let $\bm{A}\in\mathbb{R}^{d\times d}$ be a positive definite matrix.
    % \[
        $\bm{\theta}_{+}
        =
        \argmin_{\bm{\theta}\in\Theta}
        \prn{
            f(\bm{\theta}) + \frac{1}{2\eta}\nrm{\bm{\theta}-\bm{\theta}_0}_{\bm{A}}^2
        }$
    % \]
    satisfies
    \[
        \nrm{\bm{\theta}_{+}-\bm{u}}_{\bm{A}}^2
        \leq
        2\eta \inpr{\nabla f(\bm{\theta}_{+}), \bm{u}-\bm{\theta}_{+}}
        + \nrm{\bm{\theta}_0-\bm{u}}_{\bm{A}}^2
        - \nrm{\bm{\theta}_0-\bm{\theta}_{+}}_{\bm{A}}^2
    \]
    for all $\bm{u}\in\Theta$.
\end{lemma}

This lemma is used to control the distance between the updated parameter and the true parameter.
\begin{lemma}
    \label{lem:onepass_linear_sum_bound}
    Assume that $\nrm{\bm{\theta}^\ast}_2\leq D$, $\Theta\subseteq\set{\bm{\theta}\in\mathbb{R}^d:\nrm{\bm{\theta}}_2\leq D}$, $\nrm{\bm{x}_t(i)}_2\leq 1$ for all $t$ and $i$, $\dot{\mu}(z)\leq L_\mu$ for all $z$, and $\ddot{\mu}(z)\leq R$ for all $z$.
    When we use \cref{alg:onepass_ofu} with $\eta=1+DR$, for any $\lambda_{\mathrm{op}}>0$,
    \begin{align*}
        \nrm{\bm{\theta}_{t+1}-\bm{\theta}^\ast}_{\bm{Q}_{t+1}}^2
        &\leq
        2\eta
        \sum_{s=1}^t \sum_{i=1}^{N}
        \prn[\big]{
            \ell_{s,i}(\bm{\theta}^\ast) - \ell_{s,i}(\bm{\theta}_{s+1})
        }
        + 4\lambda_{\mathrm{op}} D^2
        \\*
        &\quad
        + 2\eta RDL_\mu N \sum_{s=1}^t \nrm{\bm{\theta}_{s+1}-\bm{\theta}_s}_{2}^2
        - \sum_{s=1}^t \nrm{\bm{\theta}_{s}-\bm{\theta}_{s+1}}_{\bm{Q}_s}^2.
    \end{align*}
    % In addition, if $\lambda\geq 6RDL_\mu N(1+DR)$, then
    % \[
    %     \nrm{\bm{\theta}_{t+1}-\bm{\theta}^\ast}_{\bm{H}_{t+1}}^2
    %     \leq
    %     (2+2DR)
    %     \sum_{s=1}^t \sum_{i=1}^{N}
    %     \prn[\big]{
    %         \ell_{s,i}(\bm{\theta}_{s+1}) - \ell_{s,i}(\bm{\theta}^\ast)
    %     }
    %     - \frac{2}{3}\sum_{s=1}^t \nrm{\bm{\theta}_{s+1}-\bm{\theta}_s}_{\bm{H}_s}^2
    %     + 4\lambda D^2 .
    % \]
\end{lemma}

\begin{proof}
    Fix $s\in[t]$.
    By the second-order Taylor expansion of $\ell_{s,i}$ around $\bm{\theta}_{s+1}$,
    \[
        \ell_{s,i}(\bm{\theta}^\ast)
        =
        \ell_{s,i}(\bm{\theta}_{s+1})
        + \inpr{\nabla \ell_{s,i}(\bm{\theta}_{s+1}), \bm{\theta}^\ast-\bm{\theta}_{s+1}}
        + \nrm{\bm{\theta}_{s+1}-\bm{\theta}^\ast}_{\widetilde{\bm{H}}_{s,i}}^2,
    \]
    where
    \[
        \widetilde{\bm{H}}_{s,i}
        =
        \int_0^1 (1-v)
        \nabla^2 \ell_{s,i}\prn*{(1-v)\bm{\theta}_{s+1}+v\bm{\theta}^\ast}\, dv.
    \]
    Rearranging gives
    \[
        \ell_{s,i}(\bm{\theta}_{s+1})-\ell_{s,i}(\bm{\theta}^\ast)
        \leq
        \inpr{\nabla \ell_{s,i}(\bm{\theta}_{s+1}), \bm{\theta}_{s+1}-\bm{\theta}^\ast}
        - \nrm{\bm{\theta}_{s+1}-\bm{\theta}^\ast}_{\widetilde{\bm{H}}_{s,i}}^2.
    \]

    Next, let
    % \[
        $
        \widetilde{\bm{H}}_{s,i}
        =
        \widetilde{\alpha}_{s,i}\bm{x}_s(i)\bm{x}_s(i)^\top,
        $
        and
        $
        \widetilde{\alpha}_{s,i}
        =
        \int_0^1 (1-v)
        \dot{\mu}\prn*{\bm{x}_s(i)^\top\prn[\big]{(1-v)\bm{\theta}_{s+1}+v\bm{\theta}^\ast}}
        dv.$
    % \]
    By \citet[Lemma 8]{zhang2025generalizedlinearbanditsoptimal}, we have
    \[
        \widetilde{\alpha}_{s,i}
        \geq
        \frac{\dot{\mu}\prn*{\bm{x}_s(i)^\top\bm{\theta}_{s+1}}}{2+2DR}.
    \]
    Therefore,
    \[
        \tilde{\bm{H}}_{s,i}
        \succeq
        \frac{1}{2+2DR}\nabla^2 \ell_{s,i}(\bm{\theta}_{s+1}).
    \]
    Summing over $i$ yields
    \begin{align}
        \label{eq:onepass_loss_curvature}
        \sum_{i=1}^{N}\prn[\big]{\ell_{s,i}(\bm{\theta}_{s+1})-\ell_{s,i}(\bm{\theta}^\ast)}
        &\leq
        \sum_{i=1}^{N}\inpr{\nabla\ell_{s,i}(\bm{\theta}_{s+1}),\bm{\theta}_{s+1}-\bm{\theta}^\ast}
        \notag\\
        &\quad
        -
        \frac{1}{2+2DR}
        \nrm{\bm{\theta}_{s+1}-\bm{\theta}^\ast}_{\sum_{i=1}^N \nabla^2 \ell_{s,i}(\bm{\theta}_{s+1})}^2.
    \end{align}

    We now control the linear term.
    Applying \cref{lem:onepass_proximal} with $f=\tilde{\ell}_s$, $\bm{A}=\bm{Q}_s$, $\bm{\theta}_0=\bm{\theta}_s$, $\bm{\theta}_+=\bm{\theta}_{s+1}$, and $\bm{u}=\bm{\theta}^\ast$, we obtain
    \begin{equation}
        \label{eq:onepass_prox_step}
        \inpr{\nabla \tilde{\ell}_s(\bm{\theta}_{s+1}),\bm{\theta}_{s+1}-\bm{\theta}^\ast}
        \leq
        \frac{1}{2\eta}
        \prn*{
            \nrm{\bm{\theta}_s-\bm{\theta}^\ast}_{\bm{Q}_s}^2
            - \nrm{\bm{\theta}_{s+1}-\bm{\theta}^\ast}_{\bm{Q}_s}^2
            - \nrm{\bm{\theta}_s-\bm{\theta}_{s+1}}_{\bm{Q}_s}^2
        }.
    \end{equation}

    Let
    % \[
        $\Delta_s=\bm{\theta}_{s+1}-\bm{\theta}_s.$
    % \]    
    Since
    % \[
        $\nabla\tilde{\ell}_s(\bm{\theta}_{s+1})
        =
        \sum_{i=1}^N
        \prn*{
            \nabla \ell_{s,i}(\bm{\theta}_s) + \nabla^2 \ell_{s,i}(\bm{\theta}_s)\Delta_s
        },$
    % \]
    the difference between the true and surrogate gradients is
    \[
        \sum_{i=1}^N \nabla\ell_{s,i}(\bm{\theta}_{s+1})-\nabla\tilde{\ell}_s(\bm{\theta}_{s+1})
        =
        \sum_{i=1}^N
        \prn*{
            \nabla\ell_{s,i}(\bm{\theta}_{s+1})
            - \nabla\ell_{s,i}(\bm{\theta}_s)
            - \nabla^2\ell_{s,i}(\bm{\theta}_s)\Delta_s
        }.
    \]
    By Taylor's theorem, for each $i$ there exists $\xi_{s,i}$ on the line segment between $\bm{x}_s(i)^\top\bm{\theta}_s$ and $\bm{x}_s(i)^\top\bm{\theta}_{s+1}$ such that
    \[
        \nabla\ell_{s,i}(\bm{\theta}_{s+1})
        - \nabla\ell_{s,i}(\bm{\theta}_s)
        - \nabla^2\ell_{s,i}(\bm{\theta}_s)\Delta_s
        =
        \frac{\ddot{\mu}(\xi_{s,i})}{2}\prn[\big]{\bm{x}_s(i)^\top\Delta_s}^2\bm{x}_s(i).
    \]
    Hence,
    \begin{align*}
        &\inpr*{
            \sum_{i=1}^N \nabla\ell_{s,i}(\bm{\theta}_{s+1})-\nabla\tilde{\ell}_s(\bm{\theta}_{s+1}),
            \bm{\theta}_{s+1}-\bm{\theta}^\ast
        }
        \\
        &=
        \sum_{i=1}^N
        \frac{\ddot{\mu}(\xi_{s,i})}{2}
        \prn[\big]{\bm{x}_s(i)^\top\Delta_s}^2
        \bm{x}_s(i)^\top(\bm{\theta}_{s+1}-\bm{\theta}^\ast)
        \\
        &\leq
        \sum_{i=1}^N
        \frac{R\dot{\mu}(\xi_{s,i})}{2}
        \abs{\bm{x}_s(i)^\top(\bm{\theta}_{s+1}-\bm{\theta}^\ast)}
        \prn[\big]{\bm{x}_s(i)^\top\Delta_s}^2
        \\
        &\leq
        \sum_{i=1}^N
        \frac{R}{2}\cdot 2D \cdot L_\mu \nrm{\Delta_s}_2^2
        =
        RDL_\mu N \nrm{\Delta_s}_2^2,
    \end{align*}
    % where $\xi_{s,i}$ lies between $\bm{x}_s(i)^\top\bm{\theta}_s$ and $\bm{x}_s(i)^\top\bm{\theta}_{s+1}$.
    where we used $\nrm{\bm{x}_s(i)}_2\leq 1$, $\nrm{\bm{\theta}_{s+1}}_2\leq D$, $\nrm{\bm{\theta}^\ast}_2\leq D$, and $\dot{\mu}\leq L_\mu$.
    Combining this bound with \eqref{eq:onepass_prox_step}, we obtain
    \begin{align}
        \label{eq:onepass_linear_term}
        \sum_{i=1}^{N}\inpr{\nabla\ell_{s,i}(\bm{\theta}_{s+1}),\bm{\theta}_{s+1}-\bm{\theta}^\ast}
        &\leq
        \frac{1}{2\eta}
        \prn*{
            \nrm{\bm{\theta}_s-\bm{\theta}^\ast}_{\bm{Q}_s}^2
            - \nrm{\bm{\theta}_{s+1}-\bm{\theta}^\ast}_{\bm{Q}_s}^2
            - \nrm{\bm{\theta}_s-\bm{\theta}_{s+1}}_{\bm{Q}_s}^2
        }
        \notag\\
        &\quad
        +
        RDL_\mu N \nrm{\bm{\theta}_{s+1}-\bm{\theta}_s}_2^2.
    \end{align}

    Since $\eta=1+DR$, we have $1/(2\eta)=1/(2+2DR)$.
    Substituting \eqref{eq:onepass_linear_term} into \eqref{eq:onepass_loss_curvature} and using the definition of $\bm{Q}_{s}$,
    % \[
    %     \bm{H}_{s+1}
    %     =
    %     \bm{H}_s + \sum_{i=1}^N \nabla^2\ell_{s,i}(\bm{\theta}_{s+1}),
    % \]
    we obtain
    \begin{align*}
        \sum_{i=1}^{N}\prn[\big]{\ell_{s,i}(\bm{\theta}_{s+1})-\ell_{s,i}(\bm{\theta}^\ast)}
        &\leq
        \frac{1}{2+2DR}
        \prn*{
            \nrm{\bm{\theta}_s-\bm{\theta}^\ast}_{\bm{Q}_s}^2
            - \nrm{\bm{\theta}_{s+1}-\bm{\theta}^\ast}_{\bm{Q}_{s+1}}^2
            - \nrm{\bm{\theta}_s-\bm{\theta}_{s+1}}_{\bm{Q}_s}^2
        }
        \\*
        &\quad
        +
        RDL_\mu N \nrm{\bm{\theta}_{s+1}-\bm{\theta}_s}_2^2.
    \end{align*}
    Summing over $s=1,\dots,t$ gives
    \begin{align*}
        &(2+2DR)\sum_{s=1}^{t}\sum_{i=1}^{N}
        \prn[\big]{\ell_{s,i}(\bm{\theta}_{s+1})-\ell_{s,i}(\bm{\theta}^\ast)}
        \\
        &\leq
        \nrm{\bm{\theta}_1-\bm{\theta}^\ast}_{\bm{Q}_1}^2
        - \nrm{\bm{\theta}_{t+1}-\bm{\theta}^\ast}_{\bm{Q}_{t+1}}^2
        - \sum_{s=1}^{t}\nrm{\bm{\theta}_s-\bm{\theta}_{s+1}}_{\bm{Q}_s}^2
        + (2+2DR)RDL_\mu N \sum_{s=1}^{t}\nrm{\bm{\theta}_{s+1}-\bm{\theta}_s}_2^2.
    \end{align*}
    Since $\bm{Q}_1=\lambda_{\mathrm{op}} \bm{I}$ and $\Theta\subseteq \set{\bm{\theta}:\nrm{\bm{\theta}}_2\leq D}$, we have
    \[
        \nrm{\bm{\theta}_1-\bm{\theta}^\ast}_{\bm{Q}_1}^2
        =
        \lambda_{\mathrm{op}} \nrm{\bm{\theta}_1-\bm{\theta}^\ast}_2^2
        \leq 4\lambda_{\mathrm{op}} D^2.
    \]
    Rearranging proves the first claim.

    For the second claim, note that $\bm{Q}_s\succeq \lambda_{\mathrm{op}} \bm{I}$, and thus
    \[
        \nrm{\bm{\theta}_{s+1}-\bm{\theta}_s}_{\bm{Q}_s}^2
        \geq
        \lambda_{\mathrm{op}} \nrm{\bm{\theta}_{s+1}-\bm{\theta}_s}_2^2.
    \]
    If $\lambda_{\mathrm{op}}\geq 6RDL_\mu N(1+DR)$, then
    \[
        (2+2DR)RDL_\mu N \nrm{\bm{\theta}_{s+1}-\bm{\theta}_s}_2^2
        \leq
        \frac{1}{3}\nrm{\bm{\theta}_{s+1}-\bm{\theta}_s}_{\bm{Q}_s}^2.
    \]
    Substituting this into the first inequality yields
    \[
        \nrm{\bm{\theta}_{t+1}-\bm{\theta}^\ast}_{\bm{Q}_{t+1}}^2
        \leq
        (2+2DR)
        \sum_{s=1}^t \sum_{i=1}^{N}
        \prn[\big]{\ell_{s,i}(\bm{\theta}^\ast)-\ell_{s,i}(\bm{\theta}_{s+1})}
        - \frac{2}{3}\sum_{s=1}^t \nrm{\bm{\theta}_{s+1}-\bm{\theta}_s}_{\bm{Q}_s}^2
        + 4\lambda_{\mathrm{op}} D^2,
    \]
    as claimed.
\end{proof}

The following lemma extends \citet[Lemma 5]{zhang2025generalizedlinearbanditsoptimal} to the CAB setting, where multiple feedback observations can be obtained at the same time.
\begin{lemma}
    \label{lem:onepass_high_prob_bound}
    Let $\mathcal{G}_t$ be the filtration defined by $\mathcal{G}_t = \sigma\prn*{\set{\set{\bm{x}_s(i),y_s(i)}_{i=1}^{N}}_{s=1}^{t-1}}$ and $\set{P_t}_{t=1}^\infty$ be a stochastic sequence of distributions over $\R^d$.
    $P_t$ is  $\mathcal{G}_t$-measurable for each $t\geq1$.
    Define $\ell_{s,i}(\bm{\theta})=- y_s(i)\bm{x}_s(i)^\top \bm{\theta} + m\prn*{\bm{x}_s(i)^\top \bm{\theta}}$, $L_t(\bm{\theta}^\ast)=\sum_{s=1}^t \sum_{i=1}^N \ell_{s,i}(\bm{\theta}^\ast)$, and $F_t=-\sum_{s=1}^t \log\prn*{\E_{\bm{\theta}\sim P_s}\brk{\exp\prn*{-\sum_{i=1}^N\prn*{\ell_{s,i}(\bm{\theta})}}}}$ for any $t\geq 1$.
    
    Then, for any $\delta\in(0,1)$, with probability at least $1-\delta$, it holds that
    \[
        \P\brk*{\forall t\in[T], L_t(\bm{\theta}^\ast)-F_t \leq \log\prn*{\frac{1}{\delta}}}\geq 1-\delta.
    \]
\end{lemma}

\begin{proof}
    Since it is necessary to properly handle the joint distribution, we extend the proof of \citet[Lemma 5]{zhang2025generalizedlinearbanditsoptimal} under the assumption of conditional independence.
    First, define $M_t=\exp\prn*{L_t(\bm{\theta}^\ast)-F_t}$.
    For each natural parameter $z$, let $p(y; z)$ denote the normalized density or mass function of the GLM.
    We use only the likelihood-ratio identity
    \[
        \frac{p(y; z)}{p(y; z^\prime)}
        =
        \exp\prn*{y(z-z^\prime)-m(z)+m(z^\prime)}.
    \]
    For notational convenience, write
    $
        p(y(1),\dots,y(N);\set{z_i}_{i=1}^N)
        =
        \prod_{i=1}^N p(y(i);z_i).
    $
    By the definition of $M_t$, we have
    \[
        M_t=\frac{\prod_{s=1}^t\E_{\bm{\theta}\sim P_s}\brk*{\prod_{i=1}^{N}\exp\prn{-\ell_{s,i}(\bm{\theta})}} }{\prod_{s=1}^t\prod_{i=1}^N\exp\prn{-\ell_{s,i}(\bm{\theta}^\ast)}}.
    \]
    Since $\set{y_t(i)}_{i=1}^N$ are conditionally independent given
    $\set{\bm{x}_t(i)^\top\bm{\theta}^\ast}_{i=1}^N$, we have
    \[
        p(y_t(1),\dots,y_t(N); \set{\bm{x}_t(i)^\top\bm{\theta}^\ast}_{i=1}^N)
        = \prod_{i=1}^N p(y_t(i); \bm{x}_t(i)^\top\bm{\theta}^\ast).
    \]
    By the likelihood-ratio identity, for any $\bm{\theta}$,
    \[
        \frac{
            \prod_{i=1}^{N}\exp\prn{-\ell_{t,i}(\bm{\theta})}
        }{
            \prod_{i=1}^{N}\exp\prn{-\ell_{t,i}(\bm{\theta}^\ast)}
        }
        =
        \frac{
            p(y_t(1),\dots,y_t(N);\set{\bm{x}_t(i)^\top\bm{\theta}}_{i=1}^N)
        }{
            p(y_t(1),\dots,y_t(N);\set{\bm{x}_t(i)^\top\bm{\theta}^\ast}_{i=1}^N)
        }.
    \]
    Together with the definition of $M_t$, this implies that
    \begin{equation}
        \label{eq:onepass_martingale_update}
        M_t=M_{t-1}\frac{\E_{\bm{\theta}\sim P_t}\brk*{p(y_t(1),\dots,y_t(N);\set{\bm{x}_t(i)^\top\bm{\theta}}_{i=1}^N)}}{p(y_t(1),\dots,y_t(N); \set{\bm{x}_t(i)^\top\bm{\theta}^\ast}_{i=1}^N)}.
    \end{equation}

    By using the above discussion, we can show that $\set{M_t}_{t=1}^\infty$ is a martingale with respect to $\set{\mathcal{G}_t}_{t=1}^\infty$.
    Here, $\E[\cdot]$ denotes expectation with respect to $y_t(1),\dots,y_t(N)$.
    \begin{align*}
        \E\brk*{M_t|\mathcal{G}_t}
        &=M_{t-1}\E\brk*{\frac{\E_{\bm{\theta}\sim P_t}\brk*{p(y_t(1),\dots,y_t(N);\set{\bm{x}_t(i)^\top\bm{\theta}}_{i=1}^N)}}{{p(y_t(1),\dots,y_t(N); \set{\bm{x}_t(i)^\top\bm{\theta}^\ast}_{i=1}^N)}} \middle| \mathcal{G}_t}\\
        &=M_{t-1}\int
        \frac{
            \E_{\bm{\theta}\sim P_t}\brk*{
                p(y(1),\dots,y(N); \set{\bm{x}_t(i)^\top\bm{\theta}}_{i=1}^N)
            }
        }{
            p(y(1),\dots,y(N); \set{\bm{x}_t(i)^\top\bm{\theta}^\ast}_{i=1}^N)
        }
        \nonumber\\
        &\qquad\cdot p(y(1),\dots,y(N); \set{\bm{x}_t(i)^\top\bm{\theta}^\ast}_{i=1}^N)\,dy(1)\dots dy(N)\\
        &=M_{t-1}\int \E_{\bm{\theta}\sim P_t}\brk*{p(y(1),\dots,y(N); \set{\bm{x}_t(i)^\top\bm{\theta}}_{i=1}^N)} dy(1)\dots dy(N)\\
        &=M_{t-1}\E_{\bm{\theta}\sim P_t}\brk*{\int p(y(1),\dots,y(N); \set{\bm{x}_t(i)^\top\bm{\theta}}_{i=1}^N) dy(1)\dots dy(N)}\\
        &=M_{t-1},
    \end{align*}
    where we used the fact that $M_{t-1}$ is $\mathcal{G}_t$-measurable and \eqref{eq:onepass_martingale_update} in the first equality.
    Therefore, since $\set{M_t}_{t=1}^\infty$ is a martingale, $M_0=1$, and $M_t\geq 0$, we can use \citet[Theorem 3.9]{lattimore2020bandit} to obtain 
    \[
        \P\brk*{\sup_{t}M_t\geq \epsilon}\leq \frac{1}{\epsilon},
    \] 
    for any $\epsilon>0$.
    By setting $\epsilon=1/\delta$, we can prove the argument.
\end{proof}

\section{Details of \cref{sec:experiments}}
\label{app:experiment}
This section describes the detailed experiment settings and reports additional results.

\subsection{Baseline algorithms}
\label{app:baseline_alg_detail}
We first describe detailed algorithms of ``Max match'' and ``FairX''. First, we show the detailed algorithm of Max match in \cref{alg:max_match}, which is a UCB-based algorithm maximizing the sum of matches. We can see the important difference between Max match and CAB-UCB in line 6. Max match chooses arms with the highest sum of expected matches instead of satisfaction. 

\cref{alg:Fairx} shows FairX's algorithm in detail. FairX is a UCB-based fairness algorithm proposed by \cite{pmlr-v139-wang21b}. 
This method ensures that each arm receives a share of exposure that is proportional to its expected match, aiming to mitigate the over-selection of specific arms. 
\cite{pmlr-v139-wang21b} proposes UCB-based and TS-based algorithms in the stochastic linear bandit setting. For fair comparisons, we use the UCB-based algorithm and extend it to CAB. 
To approximate the argmax over $\text{CR}_t$, we draw $50$ candidate parameters from the confidence region as follows: if $\bm{V}_t=\bm{L}_t\bm{L}_t^\top$, we first sample $\bm{x}$ uniformly from the unit Euclidean ball and then set $\bm{\theta}=\bar{\bm{\theta}}_t+\bm{L}_t^{-1}\prn{\sqrt{\gamma}\bm{x}}$, so that $\bm{\theta}\in\text{CR}_t$. 
For each sampled candidate, we evaluate $\sum_{i \in [N]}\sum_{a \in [K]} \frac{\mu({\bm{\phi}}_t(i,a)^\top{\bm{\theta}})}{\sum_{a^\prime \in [K]} \mu({\bm{\phi}}_t(i,a^\prime)^\top{\bm{\theta}})}\cdot \mu({\bm{\phi}}_t(i,a)^\top{\bm{\theta}})$ and keep the best one.

We then describe the implementation of ``CAB-TS (${\epsilon}$)'' in \cref{alg:ts}. 
The optimization problem in \cref{alg:ts} does not satisfy monotonicity since $h_t(\pi; \tilde{{\bm{\epsilon}}}_t)$ might be a negative value. 
For convenience, we set $w_a(S \cup j) - w_a(S)$ to 0 to retain monotonicity if $w_a(S \cup j) - w_a(S)<0$.
This implementation is a practical proxy for the perturbed allocation step and is not an exact implementation of the oracle assumed in the CAB-TS regret analysis.

Finally, we describe the hyperparameters of these algorithms. In the main text, we use $\lambda_0=d$, $c_1 = \sqrt{d}$ for CAB-UCB and Max match, $a = \sqrt{dN}$ for CAB-TS, and $\gamma = 0.1$ for FairX.

\begin{algorithm}[t]
\caption{Max match}
\label{alg:max_match}
\begin{algorithmic}[1]
    \Require {The total rounds $T$, the number of users $N$, and tuning parameter $\lambda_0$ and $\alpha$.}
    % \For {$t=1,\dots,\tau$}
    %     \State {For each $i$, independently assign an arm $a$ chosen uniformly at random from $[K]$.}
    % \EndFor
    \State {$\mathcal{D}_0=\varnothing$.}
    \State {${\bm{V}_1}\gets\lambda_0 {\bm{I}} $.}
    \For {$t=1,\dots,T$} 
        \State {$\bar{{\bm{\theta}}}_t\gets\argmin_{{\bm{\theta}}\in\R^d}\tilde{\mathcal{L}}(\mathcal{D}_t;{\bm{\theta}},\kappa_\mu\lambda_0)$.}
        \State {Choose $\pi_t=\argmax_{\pi}\prn[\bigg]{\sum_{a \in [K]}\sum_{i\in\pi^{-1}(a)}\mu({\bm{\phi}}_t(i,a)^\top{\bar{\bm{\theta}}_t})+g_t(\pi)}$, where $g_t$ is defined in \eqref{eq:definition_g_t}.}
        \State {Observe $y_t(i)$ for any $i\in\brk{N}$.}
        \State {${\bm{x}_t}(i)\gets {\bm{\phi}}_t(i,\pi_t(i))$ for any $i\in[N]$ and $\mathcal{D}_{t+1}\gets \mathcal{D}_{t}\cup\set{\prn{{\bm{x}_t}(i),y_t(i)}}_{i\in[N]}$.}
        \State {$V_{t+1} \gets \lambda_0 {\bm{I}} + \sum_{s =1}^{t}\sum_{i=1}^N {\bm{x}_s} (i){\bm{x}_s} (i)^\top$.}
    \EndFor
\end{algorithmic}
\end{algorithm}

\begin{algorithm}[t]
\caption{FairX}
\label{alg:Fairx}
\begin{algorithmic}[1]
    \Require {The total rounds $T$, the number of users $N$, and tuning parameter $\lambda_0$ and $\gamma$.}
    \State {$\mathcal{D}_0=\varnothing$.}
    \State {$\bm{V}_1\gets\lambda_0 {\bm{I}} $.}
    \For {$t=1,\dots,T$} 
        \State {$\bar{{\bm{\theta}}}_t\gets\argmin_{{\bm{\theta}}\in\R^d}\tilde{\mathcal{L}}(\mathcal{D}_t;{\bm{\theta}},\kappa_\mu\lambda_0)$.}
        \State {Set $\text{CR}_t = ({\bm{\theta}} : \|{\bm{\theta}} - \bar{{\bm{\theta}}}_t\|_{{\bm{V}_t}} \leq \sqrt{\gamma})$.}
        \State {Choose ${\bm{\theta}}_t = \argmax_{{\bm{\theta}} \in \text{CR}_t}\prn[\bigg]{\sum_{i \in [N]}\sum_{a \in [K]} \frac{\mu({\bm{\phi}}_t(i,a)^\top{\bm{\theta}})}{\sum_{a' \in [K]} \mu({\bm{\phi}}_t(i,a')^\top{\bm{\theta}})}\cdot \mu({\bm{\phi}}_t(i,a)^\top{\bm{\theta}})}$.}
        \State {Construct policy $P_t(i,a) = \frac{\mu({\bm{\phi}}_t(i,a)^\top{\bm{\theta}_t})}{\sum_{a \in [K]} \mu({\bm{\phi}}_t(i,a)^\top{\bm{\theta}_t})}$.}
        \State {Sample $\pi_t \sim P_t$.}
        \State {Observe $y_t(i)$ for any $i\in\brk{N}$.}
        \State {${\bm{x}_t}(i)\gets {\bm{\phi}}_t(i,\pi_t(i))$ for any $i\in[N]$ and $\mathcal{D}_{t+1}\gets \mathcal{D}_{t}\cup\set{\prn{{\bm{x}_t}(i),y_t(i)}}_{i\in[N]}$.}
        \State {$V_{t+1} \gets \lambda_0 {\bm{I}} + \sum_{s =1}^{t}\sum_{i=1}^N {\bm{x}_s} (i){\bm{x}_s} (i)^\top$.}
    \EndFor
\end{algorithmic}
\end{algorithm}

\subsection{Details of the setup}
\label{app:details_setup_experiments}
In this subsection, we provide the detailed setup of the synthetic experiments. We first define the 5-dimensional feature vector ${\bm{\phi}}(i,a)$ as follows.
\begin{align}
    {\bm{\phi}}(i,a) = \lambda \cdot {\bm{\phi}}_{pop}(i,a) + (1 - \lambda)\cdot {\bm{\phi}}_{base}(i,a),
\end{align}
where ${\bm{\phi}}_{pop}$ and ${\bm{\phi}}_{base}$ are sampled from the standard normal distribution. 
Moreover, ${\bm{\phi}}_{pop}$ is componentwise strictly increasing with respect to the arm index, meaning that for every user $i$, component $m \in [5]$, and all $a \in [K-1]$, $[{\bm{\phi}}_{pop}(i,a)]_m < [{\bm{\phi}}_{pop}(i,a+1)]_m$. 
The parameter $\lambda$ controls the strength of arm popularity. We use $\mu(x) = 1 / (1 + \exp{(-x)})$ and $r(x) = \min\{x,\beta\}$ as the feedback mean and satisfaction functions, respectively, so matches beyond $\beta$ have no additional effect on satisfaction. 
The unknown true parameter ${\bm{\theta}}^*$ is sampled from the standard uniform distribution.

We compare CAB-UCB, CAB-TS ($\epsilon$), the heuristic variant CAB-TS ($\theta$), and the one-pass OMD against Random, Max match, and FairX. We fix $N=50$ and $K=10$ in all settings. The default comparisons in \cref{fig:satisfaction_timestep,fig:matches_timestep} use $T=10000$, $\lambda=0.5$, and $\beta=5.0$ over $10$ runs. The $\lambda$- and $\beta$-sweeps in \cref{fig:lambda_sweep,fig:beta_sweep} vary the corresponding parameter while fixing the others. The histogram analyses in \cref{fig:selection_hist,fig:expected_matches_hist} use $T=5000$, $\lambda=1.0$, and $\beta=5.0$ with $5$ runs. For the appendix-only results in \cref{fig:n_action,fig:gamma,fig:runtime_onepass}, we use the same data-generation procedure and set $T=5000$, $\lambda=0.5$, and $\beta=5.0$ unless otherwise specified.
Throughout the experiments, the allocation routines are practical proxies for the offline oracle steps used in the theory.
% Accordingly, CAB-TS ($\epsilon$) should not be interpreted as an exact implementation of the oracle assumed in the CAB-TS regret analysis.
Accordingly, CAB-TS ($\theta$) should be interpreted only as an empirical heuristic variant, not as an algorithm covered by the theoretical guarantees, and the one-pass variant should likewise be interpreted as an empirical proxy for the theory-side oracle-based procedure.

\paragraph{Implementation of the one-pass variant}
For the one-pass OMD implementation, we set $\lambda_{\mathrm{op}}=5.0$, $\eta=1.0$, and $\delta=0.05$ in all experiments.
The projection set is $\Theta=\set{\bm{\theta}\in\mathbb{R}^d:\nrm{\bm{\theta}}_2\leq\sqrt{d}}$, which contains the true parameter generated from $[0,1]^d$.
We initialize $\bm{\theta}_1=\bm{0}$ and $\bm{Q}_1=\lambda_{\mathrm{op}}\bm{I}$, and project each iterate onto the parameter set $\Theta$. At round $t$, the implementation computes a confidence radius of the same form as in \cref{app:onepass} using the current number of past observations, with $L_\mu=1/4$ for the logistic link, and then forms the optimistic match estimate
\[
    \tilde{\mu}_t(i,a)
    =
    \mu\prn*{
        \bm{\phi}_t(i,a)^\top \bm{\theta}_t
        +
        \hat{\beta}_t \nrm{\bm{\phi}_t(i,a)}_{\bm{Q}_t^{-1}}
    }.
\]
Instead of solving the nested maximization in \eqref{eq:onepass_ofu_rule} exactly, we use the same sequential randomized welfare routine as in the other experimental methods. This replacement is heuristic and lies outside the scope of the theory in \cref{app:onepass}. If $s_a$ denotes the current accumulated optimistic matches for arm $a$, then arm $a$ receives weight $\prn*{r\prn{s_a+\tilde{\mu}_t(i,a)}-r(s_a)}^{K-1}$; these weights are normalized to probabilities, and one arm is sampled. After observing the Bernoulli matches $y_t(i)$, the update step computes
\[
    \bm{\Delta}_t
    =
    \prn*{
        \sum_{i=1}^N \nabla_{\bm{\theta}}^2\ell_{t,i}(\bm{\theta}_t)
        +
        \bm{Q}_t/\eta
    }^{-1}
    \prn*{
        \sum_{i=1}^N \nabla_{\bm{\theta}}\ell_{t,i}(\bm{\theta}_t)
    },
    \qquad
    \bm{\theta}_{t+1}
    =
    \Pi_{\Theta}\prn*{\bm{\theta}_t-\bm{\Delta}_t},
\]
that is, it first takes the unconstrained minimizer of the quadratic surrogate and then projects it onto $\Theta$. Finally, it updates
\[
    \bm{Q}_{t+1}
    =
    \bm{Q}_t
    +
    \sum_{i=1}^{N}
    \dot{\mu}\prn*{\bm{\phi}_t(i,\pi_t(i))^\top \bm{\theta}_{t+1}}
    \bm{\phi}_t(i,\pi_t(i))\bm{\phi}_t(i,\pi_t(i))^\top.
\]

We evaluate all methods in terms of cumulative arm satisfaction (our objective) and cumulative matches. These metrics are calculated as follows.
\begin{align*}
    &\text{Cumulative arm satisfaction} = \sum_{t=1}^T f_t(\pi_t;{\bm{\theta}}^\ast) \notag \\
    &\text{Cumulative matches} = \sum_{t=1}^T \sum_{a \in [K]}\sum_{i \in \pi_t^{-1}(a)}y_t(i), \notag
\end{align*}
For \cref{fig:lambda_sweep,fig:beta_sweep}, we normalize cumulative satisfaction by the cumulative satisfaction of a reference allocation computed using the true parameter and the same allocation oracle. Solid lines denote means over runs, and shaded regions indicate 95\% confidence intervals where applicable.
All experiments were performed on a MacBook Air (13-inch, Apple M3, 16 GB RAM).

\begin{figure*}[t]
    \captionsetup[subfigure]{skip=1.5pt}
    \captionsetup{skip=2.2pt}
    \centering
    \includegraphics[width=1\linewidth]{arxiv/fig/arxiv/legend.png} \\
    \vspace{-1pt}
    \begin{subfigure}[t]{0.32\linewidth}
        \centering
        \includegraphics[width=\linewidth,trim=0 8pt 0 0,clip]{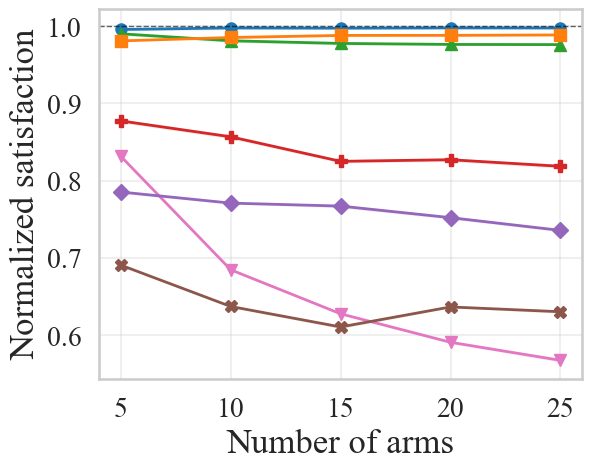}
        \caption{Normalized cumulative satisfaction under varying $K$.}
        \label{fig:n_action}
    \end{subfigure}
    \hfill
    \begin{subfigure}[t]{0.32\linewidth}
        \centering
        \includegraphics[width=\linewidth,trim=0 8pt 0 0,clip]{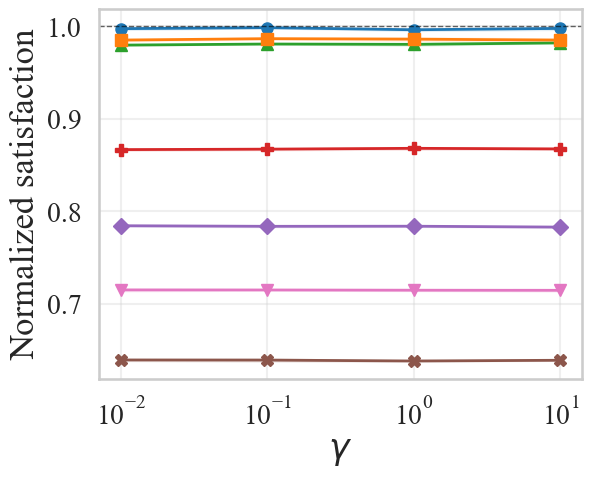}
        \caption{Normalized cumulative satisfaction under varying $\gamma$.}
        \label{fig:gamma}
    \end{subfigure}
    \hfill
    \begin{subfigure}[t]{0.32\linewidth}
        \centering
        \includegraphics[width=\linewidth,trim=0 8pt 0 0,clip]{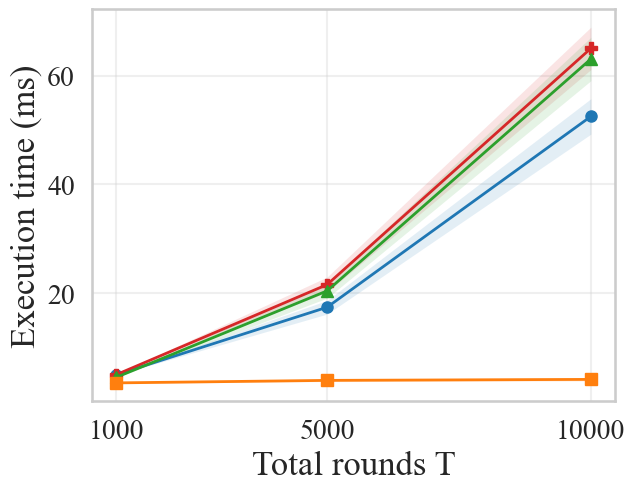}
        \caption{Average execution time per round under varying $T$.}
        \label{fig:runtime_onepass}
    \end{subfigure}
    \caption{Additional synthetic results. Panels (a) and (b) report normalized cumulative satisfaction, where each value is normalized by the cumulative satisfaction of a reference allocation computed with the true parameter and the same allocation oracle. Panel (a) varies the number of arms, panel (b) varies FairX's confidence parameter $\gamma$, and panel (c) compares the average execution time per round as the horizon length changes. Unless otherwise specified, we use $N=50$, $T=5000$, $\lambda=0.5$, and $\beta=5.0$, while fixing the remaining parameters to their default values.}
    \label{fig:appendix_additional_results}
\end{figure*}

\subsection{Additional results}

% In Figure \ref{fig:n_action} to \ref{fig:gamma}, we report additional results on synthetic experiments to demonstrate that CAB-UCB and CAB-TS perform better than the baselines in various situations. 
% We vary the number of arms and $\gamma$ in FairX. We compare CAB-UCB and CAB-TS to the baselines in terms of cumulative arm satisfaction (our objective) and matches (typical objective) normalized by those of the optimal algorithms, which use the true ${\bm{\theta}}$.
% Unless otherwise specified, we set $N=50$, $K=10$, $T=5000$, $\lambda=0.5$, and $\beta=5.0$. 
% % Note that the shaded regions in the plots represent the 95\% confidence intervals.
% \\\\
% \textbf{How does CAB-UCB perform with varying the number of arms?} \quad
% In Figure~\ref{fig:n_action}, we compare CAB-UCB with the baselines by varying the number of arms. The results show that CAB-UCB improves cumulative satisfaction across various numbers of arms. We also observe that CAB-UCB is superior to the other CAB variants, which is consistent with the result of the main text. 
% \\\\
% \textbf{How does CAB-UCB perform with varying $\gamma$ in FairX?} \quad
% Figure~\ref{fig:gamma} compares the performance of the methods with varying $\gamma$ in FairX. Figure~\ref{fig:gamma} shows that CAB-UCB is superior to all baselines with respect to satisfaction. This result suggests that FairX fails to maximize arm satisfaction effectively, regardless of the value of $\gamma$, since FairX does not directly maximize arm satisfaction. 

We report three additional synthetic experiments in \cref{fig:appendix_additional_results}, covering robustness to the number of arms, sensitivity to FairX's confidence parameter $\gamma$, and computational cost as the horizon length increases.
In \cref{fig:n_action,fig:gamma}, each value is the cumulative satisfaction normalized by a reference allocation computed with the true parameter and the same allocation oracle. Unless otherwise specified, we fix $N=50$, $K=10$, $T=5000$, $\lambda=0.5$, and $\beta=5.0$.

\paragraph{Varying the number of arms.}
\Cref{fig:n_action} shows that CAB-UCB remains essentially indistinguishable from the reference value across all tested values of $K$, while CAB-TS ($\epsilon$), the heuristic variant CAB-TS ($\theta$), and the one-pass variant also stay close to the reference value. 
In contrast, Random, Max match, and FairX remain substantially below the CAB methods, with the gap becoming particularly large for Random and Max match as $K$ increases. These results indicate that the proposed CAB methods preserve arm-side satisfaction even when the action space becomes larger.

\paragraph{Varying FairX's parameter $\gamma$.}
\Cref{fig:gamma} shows that changing $\gamma$ over several orders of magnitude has little effect on any method's normalized satisfaction. CAB-UCB consistently attains the highest value, and the other CAB variants remain close behind, whereas FairX stays well below the proposed methods for all tested values of $\gamma$. 
This suggests that FairX\char39s limitation in this problem is not primarily a matter of tuning $\gamma$. 
Rather, it stems from optimizing a fairness criterion based on expected matches instead of the arm-satisfaction objective.

\paragraph{Runtime comparison.}
\Cref{fig:runtime_onepass} compares the average execution time per round while varying the horizon length $T$. The one-pass variant is the fastest method and remains nearly constant as $T$ grows, whereas CAB-UCB, CAB-TS ($\epsilon$), and the heuristic variant CAB-TS ($\theta$) become noticeably slower for larger $T$. This pattern is consistent with the design of the one-pass variant, which avoids solving a regularized MLE from the full history at every round, and shows the expected computational trade-off between accuracy and efficiency.

\end{document}